%% file: arxiv.tex
\PassOptionsToPackage{x11names, dvipsnames, table}{xcolor}
\documentclass[11pt,pagebackref=true]{article}
\usepackage{fullpage}

\usepackage[numbers]{natbib}

\usepackage[utf8]{inputenc}
\usepackage[T1]{fontenc}
\usepackage{hyperref}
\usepackage{url}
\usepackage{booktabs}
\usepackage{amsfonts}
\usepackage{nicefrac}
\usepackage{microtype}
\usepackage[export]{adjustbox}
\usepackage[flushleft]{threeparttable}

\usepackage{parskip}
\usepackage{hyperref} 
\usepackage{subcaption}
\usepackage{todonotes}
\usepackage{booktabs}
\usepackage{dirtytalk}
\usepackage{makecell}

\definecolor{darkscarlet}{rgb}{0.34, 0.01, 0.1}
\definecolor{yaleblue}{rgb}{0.06, 0.3, 0.57}
\definecolor{darkpowderblue}{rgb}{0.0, 0.2, 0.6}
\definecolor{midnightblue}{HTML}{0059b3}
\definecolor{noonblue}{HTML}{e5eef7}
\definecolor{chromered}{HTML}{f14233}
\definecolor{olivedrab}{HTML}{6b8e23}
\definecolor{bgcolor}{rgb}{0.8,1,1}

\usepackage[hyperpageref]{backref}
\hypersetup{ 
    colorlinks=true,
    citecolor=darkscarlet,
    linkcolor=darkpowderblue,
    filecolor=magenta,      
    urlcolor=yaleblue,
    menucolor=gray
}
\renewcommand*{\backref}[1]{}
\renewcommand*{\backrefalt}[4]{%
   \ifcase #1 %
     \footnotesize{(Not cited.)}%
   \or
     \footnotesize{(Cited on page~#2)}%
   \else
     \footnotesize{(Cited on page~#2)}%
\fi }

\input{macros}
\definecolor{midnightblue}{rgb}{0.1, 0.1, 0.44}

\renewcommand{\algname}[1]{{\small \color{midnightblue!70!black} \sf #1}}

\title{Gluon: Making Muon \& Scion Great Again! \\(Bridging Theory and Practice of LMO-based Optimizers for LLMs)}

\author{
Artem Riabinin \qquad Egor Shulgin \qquad Kaja Gruntkowska \qquad Peter Richt{\'a}rik \\
\phantom{x}
    \\
    King Abdullah University of Science and Technology (KAUST) \\
    Thuwal, Saudi Arabia
}
\date{}

\begin{document}

\maketitle

\begin{abstract}
    Recent developments in deep learning optimization have brought about radically new algorithms based on the Linear Minimization Oracle (LMO) framework, such as \algname{Muon} \citep{jordan2024muon} and \algname{Scion} \citep{pethick2025training}. After over a decade of \algname{Adam}'s dominance, these LMO-based methods are emerging as viable replacements, offering several practical advantages such as improved memory efficiency, better hyperparameter transferability, and most importantly, superior empirical performance on large-scale tasks,  including LLM training. However, a significant gap remains between their practical use and our current theoretical understanding: prior analyses (1) overlook the layer-wise LMO application of these optimizers in practice, and (2) rely on an unrealistic smoothness assumption, leading to impractically small stepsizes. To address both, we propose a new LMO-based method called \algname{Gluon}, capturing prior theoretically analyzed methods as special cases, and introduce a new refined generalized smoothness model that captures the layer-wise geometry of neural networks, matches the layer-wise practical implementation of \algname{Muon} and \algname{Scion}, and leads to convergence guarantees with strong practical predictive power. Unlike prior results, our theoretical stepsizes closely match the fine-tuned values reported by \citet{pethick2025training}. Our experiments with \texttt{NanoGPT} and \texttt{CNN} confirm that our assumption holds along the optimization trajectory, ultimately closing the gap between theory and practice.
\end{abstract}

\section{Introduction}

The success of deep learning models across a wide range of challenging domains is inseparable from the optimization algorithms used to train them. As neural networks have grown deeper and datasets larger, optimization has quietly become one of the most consequential components of modern machine learning (ML). Nowhere is this more evident than in the training of large language models (LLMs), which routinely consume thousands of GPU-hours.
\algname{Adam} \citep{kingma2015adam} (and lately  \algname{AdamW} \citep{loshchilov2019decoupled})---being effective, relatively reliable, and widely adopted---has for over a decade served as the default choice for this task. While this reliance has powered much of deep learning’s progress, it has also exposed the shortcomings of adaptive moment estimation as a one-size-fits-all solution--namely, sensitivity to learning rate schedules, heavy tuning requirements \citep{wilson2017marginal}, and poor generalization when not carefully calibrated~\citep{zou2021understanding}. However, a shift may now be underway. Recent optimizers, such as \algname{Muon} \citep{jordan2024muon} and \algname{Scion} \citep{pethick2025training}, represent a significant departure from \algname{Adam}-type methods: they forgo the adaptive moment estimation in favor of a geometry-aware approach inspired by Frank-Wolfe algorithms \citep{frank1965algorithm, pokutta2024frank}. These optimizers are not only simpler to implement and easier to tune, but also appear empirically stronger, outperforming \algname{AdamW} in LLM training \citep{liu2025muon, pethick2025training}.

Yet, despite their potential, these new methods are still in their infancy, and our understanding of their theoretical foundations and practical utility in LLM training remains incomplete. Prior convergence guarantees in realistic nonconvex regimes are still far from satisfactory. Indeed, as we argue in \Cref{sec:impl_vs_an}, the (very few) existing theoretical analyses \emph{fail to capture the true algorithms used in practice}, focusing instead on simplified variants that diverge from actual implementations.
We identify two key mismatches---\emph{neglect of layer-wise structure} (\Cref{sec:layerwise}) and flawed stepsize choices stemming from an \emph{inaccurate smoothness model} (\Cref{sec:orig_steps})---and close this gap with a \emph{solution to both}. We elaborate on these advances in the remainder of the paper.

\begin{algorithm}[t]
\caption{\algname{Gluon:} Stochastic Adaptive Layer-Wise LMO-based Optimizer with Momentum}
\label{algo:stochastic}
\begin{algorithmic}[1]
\State \textbf{Input:} Initial model parameters $X^0 = [X_1^0, \dots, X_p^0] \in \cS$, momentum $M^0=[M_1^0, \dots, M_p^0]\in \cS$, momentum decay factors $\beta^k \in [0,1)$ for all iterations $k\geq 0$
\For{$k = 0, 1, 2, \dots, K-1$}
    \State Sample $\xi^k \sim \mathcal{D}$
    \For{$i = 1, 2, \dots, p$}
        \State Compute stochastic gradient $\nabla_i f_{\xi^k}(X^k)$ for layer $i$
        \State \label{step:momentum}Update momentum \ $M_i^k=\beta^k M^{k-1}_i+(1-\beta^k)\nabla_i f_{\xi^k}(X^k)$ for layer $i$
        \State Choose adaptive stepsize/radius $t_i^k>0$ for layer $i$
        \State Update parameters for layer $i$ via LMO over $\mathcal{B}_i^k \eqdef \{ X_i \in \cS_i : \|X_i - X_i^k\|_{(i)} \leq t_i^k \}$:
        \begin{equation}
            \label{eq:update_rule_stoch}
            X_i^{k+1} = \mathrm{LMO}_{\mathcal{B}_i^k} \left( M_i^k \right)
            := \underset{X_i \in \mathcal{B}_i^k}{\arg\min} \ \langle M_i^k, X_i \rangle_{(i)}
        \end{equation}
    \EndFor
    \State Update full parameter vector \
    $
    X^{k+1} = [X_1^{k+1}, \ldots, X_p^{k+1}]
    $
\EndFor
\end{algorithmic}
\end{algorithm}

Our goal is to solve the general optimization problem
\begin{eqnarray}\label{eq:problem}
    \min_{X \in\cS} \brac{f(X) := \ExpSub{\xi \sim \mathcal{D}}{f_{\xi}(X)}},
\end{eqnarray}
where $\cS$ is a finite-dimensional vector space and $f_{\xi}: \cS \mapsto \mathbb{R}$ are potentially non-convex and non-smooth but continuously differentiable functions. Here, $f_\xi(X)$ represents the loss of model parameterized by $X$ associated with training data point $\xi$ sampled from probability distribution~$\mathcal{D}$. To make the problem meaningful, we  assume that $f^{\inf} \eqdef \inf_{X \in \cS} f(X) > - \infty$. In this work we are particularly interested in the scenario when the parameter vector $X\in \cS$ is obtained by collecting the matrices $X_i \in \cS_i \eqdef \R^{m_i \times n_i}$ of trainable parameters across all layers $i = 1,\ldots, p$ of a deep model. For simplicity, we therefore write $X = [X_1, \ldots, X_p]$. This means that, formally, $\cS$ is the $d$-dimensional product space
\begin{align*}
    \textstyle \cS \eqdef \bigotimes_{i = 1}^{p} \cS_i \equiv \cS_1 \otimes \cdots \otimes \cS_p,
\end{align*}
where $d\eqdef \sum_{i=1}^p m_i n_i $. With each space $\cS_i$ we associate the trace inner product  $\langle X_i, Y_i \rangle_{(i)} \eqdef \operatorname{tr}(X_i^{\top} Y_i)$ for $X_i,Y_i \in \cS_i$, and an arbitrary norm $\|\cdot\|_{(i)}$, not necessarily induced by the inner product.

\section{Theory vs. practice of Muon and Scion}\label{sec:impl_vs_an}

In this work, we focus on an algorithm based on iteratively calling linear minimization oracles (LMOs) across all layers, formalized in \Cref{algo:stochastic}, for which we coin the name \algname{Gluon}. In particular, for each layer $i$, independently across all layers, \algname{Gluon} iteratively updates the parameters  via
\begin{align*}
    X_i^{k+1} = \lmo{\mathcal{B}_i^k}{M_i^k}
    \eqdef \underset{X_i \in \mathcal{B}_i^k}{\arg\min} \ \langle M_i^k, X_i \rangle_{(i)}, \; \text{where} \; \cB_i^k \eqdef \{X_i \in  \cS_i : \|X_i - X_i^k\|_{(i)} \leq t_i^k \},
\end{align*}
where $t_i^k>0$ is an adaptively chosen stepsize/radius/learning rate.\footnote{In this context, the radii defining the norm balls in the LMOs effectively act as stepsizes--see \Cref{sec:step_note}. Accordingly, we use the terms \emph{radius}, \emph{stepsize}, and \emph{learning rate} interchangeably throughout.}
Note that the momentum  $M^k = [M_1^k, \dots, M_p^k] \in \cS$ accumulates the contributions from the stochastic gradients $\nabla f_{\xi^k}(X^k) = [\nabla_1 f_{\xi^k}(X^k),\ldots,\nabla_p f_{\xi^k}(X^k)] \in \cS$ (see Step \ref{step:momentum} of \Cref{algo:stochastic}). 

The \algname{Gluon} framework generalizes a range of methods, including \algname{Muon} and \algname{Scion}, which are recovered as special cases under specific norm choices (see \Cref{sec:det_examples} and \Cref{sec:lmo_examples}).
Beyond their ability to outperform \algname{AdamW} on large-scale benchmarks, these optimizers offer a number of attractive properties: improved memory efficiency, greater robustness to hyperparameter settings, and the ability to transfer those settings across model sizes \citep{pethick2025training, shah2025practical}. Moreover, in contrast to \algname{Adam}, they were theoretically analyzed shortly after release and are guaranteed to converge under standard assumptions of Lipschitz smoothness\footnote{A function $f:\cS\mapsto\R$ is $L$-smooth if $\norm{\nabla f(x) - \nabla f(y)}_{\star} \leq L \norm{x-y}$ for all $x,y\in \cS$, where $\cS$ is a finite-dimensional vector space equipped with a norm $\|\cdot\|$ and $\|\cdot\|_{\star}$ is the dual norm associated with $\|\cdot\|$.} and bounded variance of stochastic gradients \citep{kovalev2025muon, li2025noteconvergencemuon, pethick2025training}. 

\algname{Gluon} presents the method that is deployed in practice \citep{jordan2024moddednanogpt, code_scion} and has proven highly effective. That said, we argue that existing analyses \citep{kovalev2025muon, li2025noteconvergencemuon, pethick2025training} do \emph{not} accurately reflect this implementation, diverging from it in two key ways. As such, they \emph{fail to explain why the algorithm performs so well}. Let us detail why.

\subsection{Layer-wise structure}\label{sec:layerwise}

First, we briefly walk through the theoretical understanding offered by previous studies. \algname{Muon} is an optimizer specifically designed for hidden layers, leaving the first and last layers to be handled by some other optimizer, e.g., \algname{Adam(W)}. Its original introduction by \citet{jordan2024muon} was purely empirical, with no attempt at theoretical analysis. The first convergence result came from \citet{li2025noteconvergencemuon}, who analyzed the smooth nonconvex setting but focused solely on problem~\eqref{eq:problem} with $p = 1$, effectively limiting the scope to the single-layer case. The \algname{Scion}\footnote{\citet{pethick2025training} introduce two variants of the \algnamefootnote{Scion} optimizer: one for constrained optimization, called simply ``\algnamefootnote{Scion}'', and another for unconstrained problems, referred to as ``unconstrained \algnamefootnote{Scion}''. In this work, ``\algnamefootnote{Scion}'' refers to either variant, and ``\algnamefootnote{unScion}'' is used when referring to the unconstrained version.} optimizer (a special case of \algname{Gluon}) proposed by \citet{pethick2025training} improves upon \algname{Muon} by applying the LMO-based rule to all layers, ultimately achieving better empirical performance. Both this work and that of \citet{kovalev2025muon} analyze (a variant of) the general update rule
\begin{align} \label{eq:orig_muon_upd}
    \begin{split}
    M^k &= \beta^k M^{k-1} + (1-\beta^k) \nabla f_{\xi^k}(X^k), \\
    X^{k+1} &= \lmo{\cB^k}{M^k},
    \end{split}
\end{align}
where $\beta^k \in [0,1)$ is momentum, $\nabla f_{\xi^k}(X^k)$ is the stochastic gradient sampled at iteration $k$, and $\cB^k \eqdef \{X \in \cS: \|X - X^k\| \leq t^k \}$ is a norm ball centered at $X^k$ with stepsize $t^k>0$. This setup closely resembles the structure of \algname{Gluon}, but is \emph{not} exactly the same.
Indeed, \algname{Gluon} updates the parameters \emph{layer-wise}, not jointly over the full vector~$X$.
This distinction is critical since for practical, extremely high-dimensional models, calculating a single global LMO for the entire parameter vector is prohibitively expensive, while breaking the problem into ``smaller'', per‑layer {\rm LMOs} restores computational feasibility.

Motivated by this disconnect, we formulate our analysis in the matrix product space $\cS$, explicitly honoring the layer-wise structure. This enables us to study the actual per-layer updates \eqref{eq:update_rule_stoch}, with assumptions and hyperparameters adapted to each layer.

\subsection{A theory with predictive power}\label{sec:orig_steps}

\begin{figure}[t]
    \centering
    \begin{subfigure}{0.33\textwidth}
        \includegraphics[width=\textwidth]{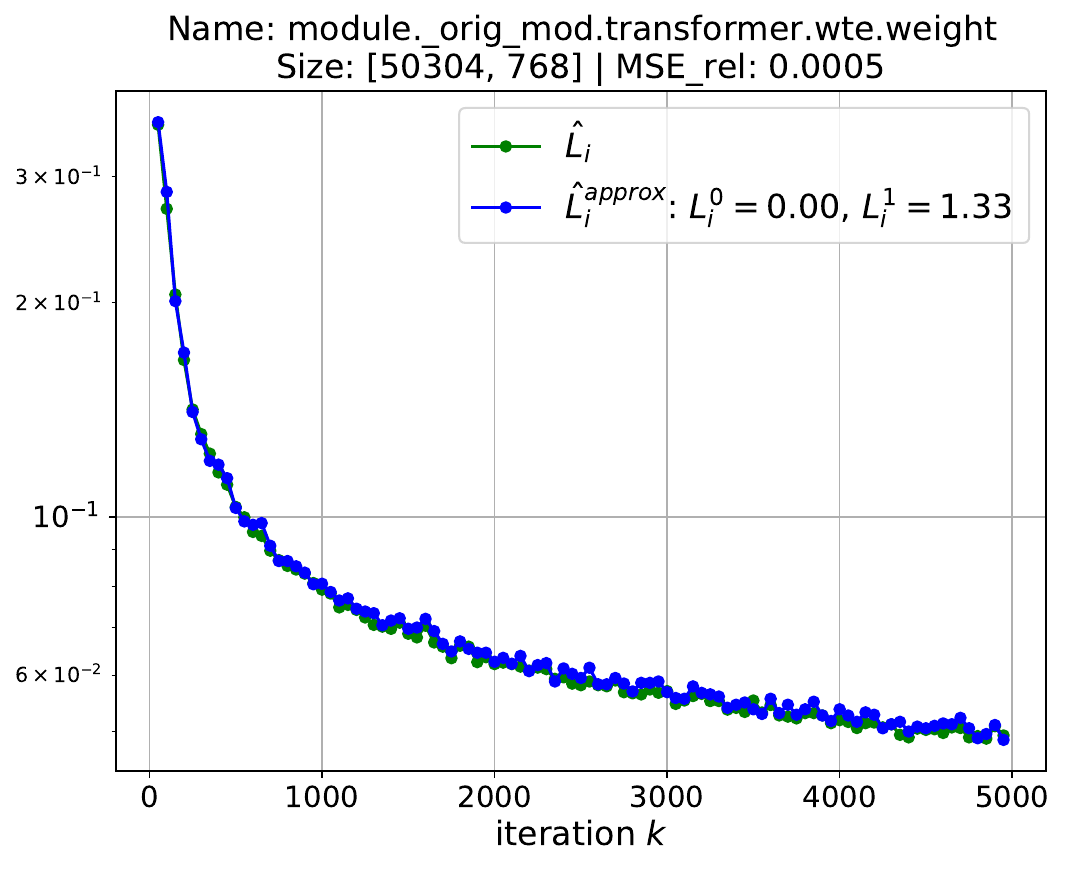}
        \caption{Token embedding matrix from the first/last layer.}
        \label{fig:embedding_ex}
    \end{subfigure}
    \hfill
    \begin{subfigure}{0.315\textwidth}
        \includegraphics[width=\textwidth]{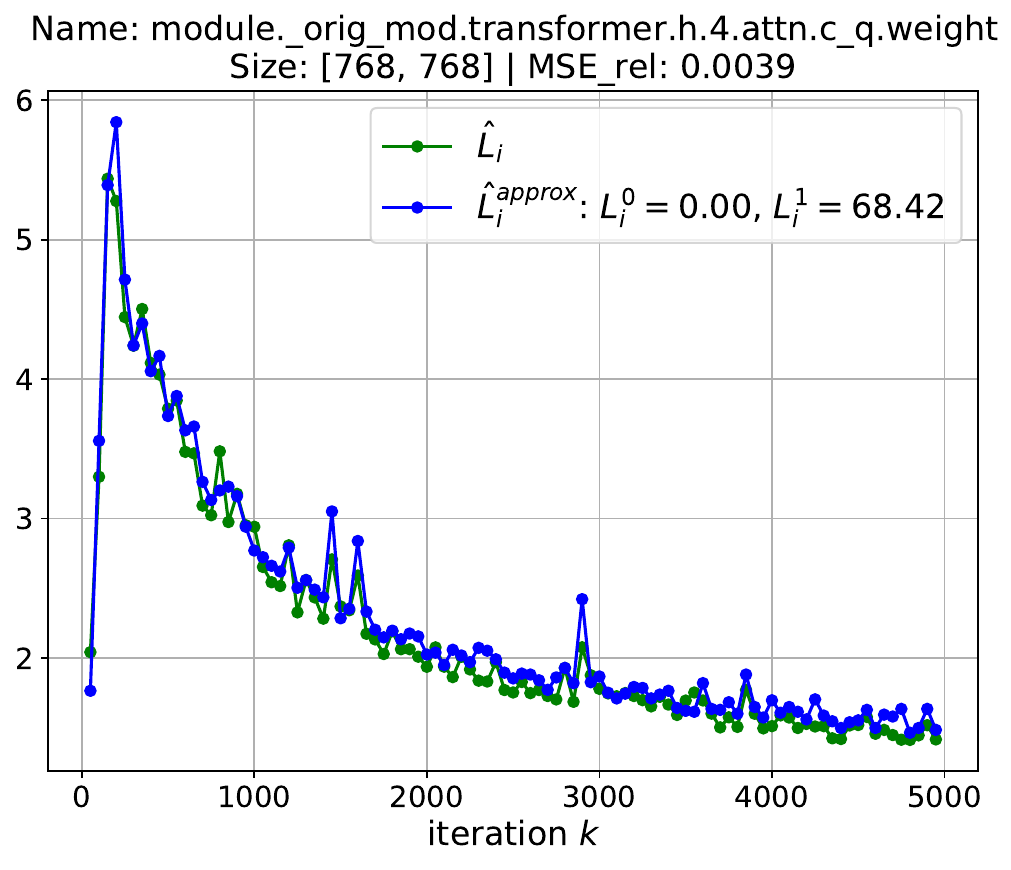}
        \caption{Self-attention query matrix from the 4th transformer block.}
        \label{fig:q_ex}
    \end{subfigure}
    \hfill
    \begin{subfigure}{0.335\textwidth}
        \includegraphics[width=\textwidth]{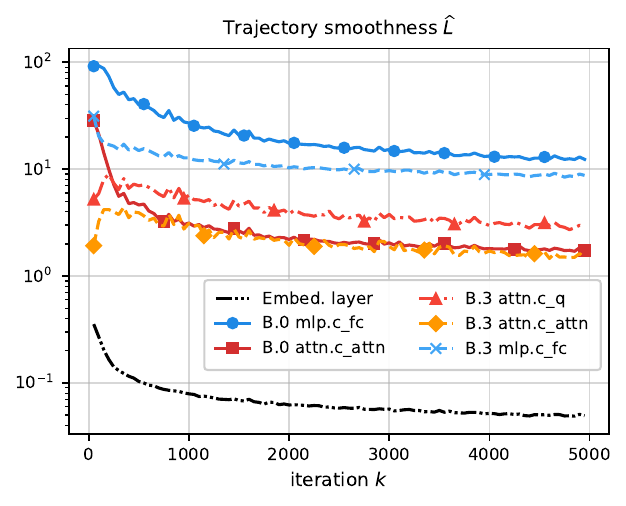}
        \caption{Trajectory smoothness across different blocks (B.$i$) and layers.}
        \label{fig:layer_smooth}
    \end{subfigure}
    \caption{Training \texttt{NanoGPT} on \texttt{FineWeb} validates our layer-wise $(L^0, L^1)$-smoothness model.}
    \label{fig:intro}
\end{figure}

All prior works claiming to guarantee convergence of \Cref{algo:stochastic} come with several serious analytical shortcomings--and these directly translate into practical deficiencies. Concretely, all existing analyses of \algname{Muon}/\algname{Scion} are built on the classical $L$-smoothness assumption, imposing a uniform smoothness constant across all layers. This is problematic, as \emph{different layers have different geometries}, and thus \emph{should be treated differently}.

But the issue runs much deeper. These algorithms are built for deep learning, where the objective functions are already well known \emph{not} to be smooth \citep{crawshaw2022robustness, zhang2020why}. This mismatch has consequences: prior convergence analyses prescribe \emph{tiny constant stepsizes} (see \Cref{table1}), uniform across all parameter groups, which bear little resemblance to the tuned learning rates that yield state-of-the-art empirical performance in practice. Consequently, they completely fail to explain why these methods perform so well empirically.
In other words, the theory falls short at the one thing it should do best: guiding practical choices, leaving practitioners reliant on costly manual tuning.

Our result in \Cref{theorem:1} shows this mismatch is \emph{not} inevitable. To better reflect the behavior of deep models, we introduce a more expressive regularity condition: the \emph{layer-wise $(L_0, L_1)$-smoothness}\footnote{While we state \Cref{ass:generalized-smoothness} in this general form, it is worth noting that the proofs do not rely on its full strength. In all cases, we only require the assumption to hold for pairs $X$, $Y$ such that $\norm{X-Y} < c$ for some constant $c\geq0$ (where $\norm{\cdot}$ is any norm on $\cS$). Specifically, the assumption is only invoked with $X=X^k$, $Y=X^{k+1}$, and since the stepsizes we use are bounded, the distances between consecutive iterates remain bounded as well. For clarity and consistency across results--since the relevant constants vary by theorem--we state the assumption in its stronger, global form, even though the local version suffices for all proofs.}--an extension of the generalized smoothness model of \citet{zhang2020why}, applied at the layer level.

\begin{assumption}[Layer-wise $(L^0, L^1)$-smoothness]\label{ass:generalized-smoothness}
    The function $f: \cS \mapsto \mathbb{R}$ is layer-wise $(L_0, L^1)$-smooth with constants $L^0 \eqdef (L^0_1, \ldots, L^0_p)\in \R^p_+$ and $L^1 \eqdef (L^1_1, \ldots, L^1_p)\in \R^p_+$. That is, the inequality
    \begin{equation}
       \| \nabla _i f(X) - \nabla _i f(Y) \|_{(i) \star} \leq \left( L^0_i + L^1_i \| \nabla _i f(X) \|_{(i) \star} \right) \|X_i - Y_i \|_{(i)} \label{eq:09yhfd98y9fd}
    \end{equation}
    holds for all $i=1,\dots,p$ and all $X = [X_1, \ldots, X_p]\in \cS$, $Y = [Y_1, \ldots, Y_p] \in \cS$, where $\|\cdot\|_{(i) \star}$ is the dual norm associated with~$\|\cdot\|_{(i)}$ (i.e., $\|X_i\|_{(i) \star} \eqdef \sup_{\norm{Z_i}_{(i)} \leq 1} \inp{X_i}{Z_i}_{(i)}$ for any $X_i\in \cS_i$).
\end{assumption}

Assumption~\ref{ass:generalized-smoothness} can be viewed as a generalization of the anisotropic \say{vector} \((L^0, L^1) \)--smoothness introduced by \citet{liu2024adagradanisotropicsmoothness} (now framed in terms of arbitrary norms), which itself is a generalization of the \((L^0, L^1) \)--smoothness model of \citet{zhang2020why}. As such, our analysis of \algname{Gluon} goes beyond all existing results, which have only considered the classical $L$-smooth setting.
Crucially, however, this is \emph{not} generalization for its own sake--we argue that this is in fact \emph{the right} model for the problem setting at hand. Why? There are (at least) two reasons.

First, unlike classical $L$-smoothness, our formulation \emph{aligns very closely with empirical observations}.
In Figures \ref{fig:embedding_ex} and \ref{fig:q_ex}, we validate Assumption~\ref{ass:generalized-smoothness} in the context of training \texttt{NanoGPT} on the \texttt{FineWeb} dataset. We plot estimated \emph{trajectory smoothness} $\hat{L}_{i}[k]$ (defined in \eqref{eq:dir_smooth}) alongside the approximation $\hat{L}_{i}^{\text{approx}}[k] \eqdef L_i^0 + L_i^1 \|\nabla_i f_{\xi^{k+1}} (X^{k+1})\|_{(i) \star}$, where $L^0_i, L^1_i$ are layer-specific parameters estimated from the training run.
The figures show these quantities for parameters from the embedding layer and one of the transformer blocks. The close correspondence between $\hat{L}_i[k]$ and $\hat{L}_{i}^{\text{approx}}[k]$ provides strong evidence that Assumption~\ref{ass:generalized-smoothness} holds approximately along the training trajectory.
In \Cref{sec:experiments}, we further corroborate this finding, showing that our assumption is satisfied \emph{across the entire model architecture} for both the \texttt{NanoGPT} language modeling task and a \texttt{CNN} trained on \texttt{CIFAR-10}. In all cases, we find that $L^0_i\approx 0$ for all $i$, again highlighting the limitations of classical smoothness. Moreover, as shown in \Cref{fig:layer_smooth}, trajectory smoothness varies substantially across blocks and layers, underscoring the need for per-layer treatment.
Together, these results suggest that layer-wise $(L^0, L^1)$-smoothness offers a \emph{significantly more realistic model of the loss landscape in modern deep learning}.

Secondly, \Cref{ass:generalized-smoothness} not only better captures the geometry of the models, but also \emph{directly informs the design of adaptive and practically effective stepsizes}. In \Cref{theorem:1}, we derive learning rates that reflect the local geometry of each parameter group, guided by our layer-wise smoothness model.
As demonstrated in \Cref{sec:nanogpt-fineweb}, our theoretically grounded stepsizes turn out to be \emph{almost the same as the ones obtained by \citet{pethick2025training} via hyperparameter tuning}--a striking validation of our approach, which further highlights the need for layer-wise reasoning. This proves that \emph{theoretical stepsizes can have predictive power} and replace trial-and-error tuning in practice.

\section{Contributions}

We present a comprehensive theoretical and empirical study of a broad class of layer-wise LMO-based optimization algorithms. Our key contributions can be summarized as follows:

$\diamond$ \textbf{A new generalized smoothness framework for deep networks.}
We introduce \emph{layer-wise $(L^0, L^1)$-smoothness} (\Cref{ass:generalized-smoothness}), a novel non-Euclidean generalized smoothness condition that reflects the anisotropic, layer-wise structure of modern deep networks. This framework extends standard $(L^0, L^1)$-smoothness assumption \citep{zhang2020why} to arbitrary norms while capturing per-layer variation, offering a \emph{realistic foundation for analyzing deep learning optimizers}.
    
$\diamond$ \textbf{First principled analysis of layer-wise methods.}
Building on our new assumption, we develop the first faithful convergence analysis for a class of LMO-based algorithms we term \algname{Gluon} (Algorithms~\ref{algo:stochastic} and~\ref{algo:deterministic}). We recover known algorithms, including state-of-the-art \algname{Muon}-type optimizers, as special cases (\Cref{sec:det_examples} and \Cref{sec:lmo_examples}), and pinpoint why earlier theoretical works \emph{fail} to explain the empirical success of these methods (\Cref{sec:impl_vs_an}). In contrast to prior analyses that oversimplify the update rules used in practice, our framework directly aligns with real-world implementations, bridging a critical gap between theory and application.

$\diamond$ \textbf{Sharper and more general convergence theory.}
We develop a convergence theory that extends prior work in both scope and sharpness. In the deterministic case (\Cref{algo:deterministic}), we establish convergence for general non-convex objectives under our \Cref{ass:generalized-smoothness} (\Cref{theorem:1}), and under the block-wise PŁ condition (\Cref{theorem:1-pl}). Unlike earlier analyses, our theory yields \emph{adaptive, layer-wise stepsizes} that align remarkably well with those selected via tuning in large-scale experiments \citep{pethick2025training} (\Cref{sec:nanogpt-fineweb}).
We next analyze the practical stochastic variant with time-varying stepsizes and momentum (\Cref{algo:stochastic}), proving convergence under non-Euclidean bounded variance assumption (\Cref{theorem:1}).
In both deterministic and stochastic regimes, our guarantees are stronger and more general than all prior work (\Cref{table1}).
While previous theories fail to explain the empirical success of \algname{Muon}-type methods, we are the first to demonstrate their \emph{provable advantage over \algname{SGD}}, offering \emph{tighter convergence rates} under \emph{more general assumptions} (\Cref{sec:convergence_stochastic}). Moreover, we provide the first theoretical explanation of the benefits of layer-wise learning rates, clearly establishing the advantages of structured, anisotropic optimization in deep learning.
    
$\diamond$ \textbf{Empirical evidence.}
We validate our theoretical insights through extensive experiments (\Cref{sec:experiments} and \Cref{appendix:more_exp}) in both language modeling (\texttt{NanoGPT} on \texttt{FineWeb}) and image classification (\texttt{CNN} on \texttt{CIFAR-10}). The results confirm that our \Cref{ass:generalized-smoothness} holds approximately throughout training and demonstrate the practical utility of our theoretically prescribed stepsizes from \Cref{theorem:1}.

\begin{table}
\caption{\small 
Comparison of  convergence guarantees for \algname{Gluon} (Algorithms \ref{algo:stochastic} and \ref{algo:deterministic}) to achieve $\min_{k=0,\dots,K-1} \sum_{i=1}^p \mathbb{E}[\| \nabla_i f(X^k) \|_{(i)\star}] \leq \varepsilon$, where the \( \mathcal{O}(\cdot) \) notation hides logarithmic factors.
Notation: $K$ = total number of iterations, $(L^0, L^1)$ = the result holds under layer-wise $(L^0, L^1)$-smoothness, $t_i^k$ = radius/stepsize, $1-\beta^k$ = momentum.
}
\label{table1}
\scriptsize
\centering 
\begin{adjustbox}{width=\columnwidth,center}
\begin{threeparttable}
\begin{tabular}[H]{cccccc}
    \toprule
    \textbf{Result} & \textbf{Stochastic?} & \textbf{$(L^0, L^1)$} & \textbf{Rate} & \textbf{Stepsizes $t_i^k$} & \textbf{$1-\beta^k$} \\
    \midrule
    \makecell{\citep[Theorem 1]{kovalev2025muon}} & \xmark & \xmark & $\mathcal{O}\parens{\frac{1}{K^{1/2}}}$ & $\textnormal{const} \propto {\frac{1}{K^{1/2}}}$\textsuperscript{{\color{blue}(b)}} & --- \\
    \makecell{\citep[Theorem 2]{kovalev2025muon}} & \cmark & \xmark & $\mathcal{O}\parens{\frac{1}{K^{1/4}}}$ & $\textnormal{const} \propto {\frac{1}{K^{3/4}}}$\textsuperscript{{\color{blue}(b)}} & $\textnormal{const} \propto {\frac{1}{K^{1/2}}}$ \\
    \makecell{\citep[Theorem 2.1]{li2025noteconvergencemuon}\textsuperscript{{\color{blue}(a)}}} & \cmark & \xmark & $\mathcal{O}\parens{\frac{1}{K^{1/4}}}$ & $\textnormal{const} \propto {\frac{1}{K^{3/4}}}$\textsuperscript{{\color{blue}(b)}} & $\textnormal{const} \propto {\frac{1}{K^{1/2}}}$ \\
    \makecell{\citep[Lemma 5.4]{pethick2025training}} & \cmark & \xmark & $\mathcal{O}\parens{\frac{1}{K^{1/4}}}$ & $\textnormal{const} \propto {\frac{1}{K^{3/4}}}$\textsuperscript{{\color{blue}(b)}} & $\propto {\frac{1}{k^{1/2}}}$ \\
    \cellcolor{bgcolor}  {\bf NEW:}  \Cref{theorem:1} & \cellcolor{bgcolor} \xmark & \cellcolor{bgcolor} \cmark & \cellcolor{bgcolor} $\mathcal{O}\parens{\frac{1}{K^{1/2}}}$ & \cellcolor{bgcolor}  Adaptive & \cellcolor{bgcolor} ---\\
    \cellcolor{bgcolor} {\bf NEW:} \Cref{theorem:3}  & \cellcolor{bgcolor} \cmark & \cellcolor{bgcolor} \cmark & \cellcolor{bgcolor} $\mathcal{O}\parens{\frac{1}{K^{1/4}}}$ & \cellcolor{bgcolor} $\propto {\frac{1}{k^{3/4}}}$ & \cellcolor{bgcolor} $\propto {\frac{1}{k^{1/2}}}$ \\
    \bottomrule
    \end{tabular}
    \begin{tablenotes}
        \tiny
        \item [{\color{blue}(a)}] Applies only to the \algnametiny{Muon}/\algnametiny{Scion} update in \eqref{eq:muon} with $p=1$.
        \item [{\color{blue}(b)}] These stepsizes are impractically tiny since they have an inverse dependence on the total number of iterations $K$.
    \end{tablenotes}
\end{threeparttable}
\end{adjustbox}
\end{table}

\section{Main theory and results}\label{sec:deterministic}

To gain a better intuition into the structure of the updates, we begin with a deterministic formulation of \algname{Gluon}, formalized in Algorithm~\ref{algo:deterministic}. At each iteration, the method independently minimizes a linear approximation of~$f$ around each parameter group~$X^k_i$ within a ball of radius~$t^k_i > 0$, ultimately allowing for layer-specific algorithmic design choices.

\subsection{Examples of optimizers satisfying our framework}\label{sec:det_examples}

Deterministic \algname{Gluon} describes a general class of methods, parameterized by the choice of norms $\|\cdot\|_{(i)}$ in the LMO. To illustrate the flexibility of this framework, we highlight several notable special cases (see \Cref{sec:lmo_examples} for more details).
First, observe that the update rule \eqref{eq:update_rule} can be written as
\begin{equation}
    X_i^{k+1} = X_i^{k} + t_i^k \mathrm{LMO}_{\{ X_i \in \cS_i : \|X_i\|_{(i)} \leq 1 \}} \left( \nabla _i f(X^k) \right) = X_i^{k} + t_i^k \underset{\left\|X_i\right\|_{(i)} \leq 1}{\operatorname{argmin}} \langle\nabla _i f(X^k), X_i \rangle_{(i)}. \label{eq:b9f8dgfd}
\end{equation}
For any $X_i \in \cS_i = \R^{m_i \times n_i}$, define  $\|X_i\|_{\alpha \rightarrow \beta} \eqdef \sup _{\|z\|_{\alpha}=1} \|X_i z\|_{\beta}$, where  $\|\cdot\|_{\alpha}$ and $\|\cdot\|_{\beta}$ are some (possibly non-Euclidean) norms on $\R^{n_i}$ and $\R^{m_i}$, respectively.
Note that \eqref{eq:b9f8dgfd} naturally recovers several known updates for specific choices of the layer norms, e.g., layer-wise normalized \algname{GD} \citep{wei2018blocknormalized} for Euclidean norms \( \|\cdot\|_{(i)}=\|\cdot\|_{2} \), and layer-wise \algname{signGD} \citep{balles2020geometrysigngradientdescent} for max-norms $\|\cdot\|_{(i)}=\|\cdot\|_{\infty}$.

Two special cases are particularly relevant to our analysis:

$\diamond$ \textbf{Muon} \citep{jordan2024muon} when \( \|\cdot\|_{(i)} = \|\cdot\|_{2 \rightarrow 2} \) for all hidden layers.

$\diamond$ \textbf{unScion for LLM training} \citep{pethick2025training}
when \(\|\cdot\|_{(i)} = \sqrt{\nicefrac{n_i}{m_i}} \|\cdot\|_{2 \to 2}\) for \( i = 1, \dots, p-1 \), corresponding to weight matrices of transformer blocks, and \(\|\cdot\|_{(p)} = n_p \|\cdot\|_{1 \to \infty}\) for the last group~\( X_p \), representing the embedding and output layers (the two coincide under the weight sharing regime\footnote{Weight sharing refers to the practice of using the same parameters (weights) for different parts of a model, rather than allowing each part to have its own unique parameters.} considered here).
In this case, update \eqref{eq:b9f8dgfd} becomes
\begin{equation}\label{eq:scion_llm_main}
\begin{aligned}
    X_i^{k+1} &=  X_i^k - t^k_i \sqrt{\frac{m_i}{n_i}} U_i^k \left(V_i^k\right)^{\top}, \quad i = 1, \dots, p-1, \\ 
    X_p^{k+1} &= X_p^k - \frac{t^k_p}{n_p} \text{sign}\left(\nabla _p f(X^k)\right), 
\end{aligned}
\end{equation}
where the matrices $U^k_i, V^k_i$ are obtained from the (reduced) SVD of $\nabla _i f(X^k) = U^k_i \Sigma^k_i \left(V^k_i\right)^{\top}$.

\subsection{Convergence results}
\label{sec:convergence_deterministic}

Having demonstrated the framework’s flexibility through concrete examples, we now state a general convergence result for deterministic \algname{Gluon}.

\begin{restatable}{theorem}{THMDET}\label{theorem:1}
    Let \Cref{ass:generalized-smoothness} hold and fix $\varepsilon>0$. Let $X^0,\ldots,X^{K-1}$ be the iterates of deterministic \algname{Gluon} (Algorithm \ref{algo:deterministic}) run with stepsizes $t^k_i = \frac{\|\nabla _i f(X^k)\|_{(i) \star}}{L^0_i + L^1_i\|\nabla _i f(X^k)\|_{(i) \star}}$.
    Then, to guarantee that
    \begin{align}\label{eq:det_precision}
        \min\limits_{k=0,\ldots,K-1} \sum\limits_{i=1}^p \left[\frac{\nicefrac{1}{L_i^1}}{\frac{1}{p} \sum_{j=1}^p \nicefrac{1}{L_j^1}} \norm{\nabla _i f(X^k)}_{(i) \star}\right] \leq \varepsilon,
    \end{align}
    it suffices to run the algorithm for
    \begin{align}\label{eq:det_rate}
        K = \left\lceil \frac{2 \Delta^0 \parens{\sum_{i=1}^p \nicefrac{L^0_i}{(L_i^1)^2}}}{\varepsilon^2 \parens{\frac{1}{p} \sum_{j=1}^p \nicefrac{1}{L_j^1}}^2} 
        + \frac{2 \Delta^0}{\varepsilon \parens{\frac{1}{p} \sum_{j=1}^p \nicefrac{1}{L_j^1}}} \right\rceil
    \end{align}
    iterations, where $\Delta^0\eqdef f(X^0) - f^{\inf}$.
\end{restatable}

Several important observations follow.

\paragraph{Convergence rate.} In \Cref{sec:det_proof}, we prove an additional result (\Cref{theorem:1_full}) that modifies the first term in \eqref{eq:det_rate} to $\nicefrac{2 \Delta^0 \sum_{i=1}^p L^0_i}{\epsilon^2}$, potentially leading to improvements in certain settings (depending on the relationship between the sequences $\{L^0_i\}$ and $\{L^1_i\}$--see \Cref{rem:det}). However, this introduces a dependence on $L^1_{\max}\eqdef\max_{i=1,\dots,p} L^1_i$ in the second term. 
Empirically, we find that $L^0_i \approx 0$ across all layers (see \Cref{sec:experiments}), making the first term vanish in both bounds. In this case, the rate~\eqref{eq:det_rate} is clearly superior, replacing the worst-case constant $L^1_{\max}$ with the more favorable harmonic mean.

When $p=1$, our rates match the best-known complexity for finding a stationary point of $(L^0, L^1)$-smooth functions, $\mathcal{O}\left(\nicefrac{L^0 \Delta^0}{\epsilon^2} + \nicefrac{L^1 \Delta^0}{\epsilon}\right)$, as established by \citet{vankov2025optimizing} for the Gradient Method.
While no prior work has analyzed deterministic \algname{Gluon} under general $(L^0, L^1)$-smoothness, there exist analyses under classical $L$-smoothness, treating the parameters as a single vector. The analysis by \citet{kovalev2025muon} guarantees convergence in $K = \left\lceil \nicefrac{6 L \Delta^0}{\epsilon^2} \right\rceil$ iterations. The same bound appears in \citet{li2025noteconvergencemuon} and \citet{pethick2025training} (by setting $\sigma^2 = 0$). Since for $p=1$, $L$-smoothness implies \Cref{ass:generalized-smoothness} with $L^1 = 0$ (\Cref{lemma:L0vsL}), our rates match these prior results up to a constant factor. Thus, even in the smooth setting, our bounds are as tight as those derived specifically for it.

However, the real strength of our guarantees lies in their broader applicability. Our analysis is much more general than prior studies, as it extends beyond standard smoothness--allowing $L_i^1 > 0$ introduces additional terms that drive the accelerated convergence enabled by $(L^0, L^1)$-smoothness. This richer model is \emph{essential for explaining the empirical speedup} of methods like \algname{Muon}, and much more accurately reflects the geometry of neural network loss surfaces. Indeed, as we demonstrate in \Cref{sec:experiments}, the assumption typically holds with $L^0_i\approx 0$ and $L_i^1 > 0$.

\paragraph{Practical radii $t^k_i$.} Unlike previous analyses \citep{kovalev2025muon, li2025noteconvergencemuon, pethick2025training}, which prescribe impractically small constant radii proportional to \( \epsilon \), our framework allows \( t^k_i \) to be \emph{adaptive} to the loss landscape. Therefore, \( t^k_i \) can be larger early in training when \( \|\nabla _i f(X^{k})\|_{(i) \star} \) is large and gradually shrink as the gradient norm decreases. In the special case when \( L^0_i \approx 0 \) (as observed empirically), \( t^k_i \approx \nicefrac{1}{L^1_i} \), which is substantially larger than the radii dictated by earlier analyses.
Crucially, as shown in \Cref{sec:nanogpt-fineweb}, our adaptive stepsizes closely match those that yield state-of-the-art empirical performance identified by \citet{pethick2025training} through hyperparameter tuning. This alignment demonstrates that \emph{principled, theory-driven stepsize selection could effectively replace costly manual tuning}.

\subsection{Stochastic case}
\label{sec:convergence_stochastic}

In practice, computing full gradients is often infeasible due to the scale of modern ML problems. We therefore turn to the practical \algname{Gluon} (\Cref{algo:stochastic}), a stochastic variant of \Cref{algo:deterministic} that operates with noisy gradient estimates available through a stochastic gradient oracle $\nabla f_{\xi}$, $\xi \sim \mathcal{D}$.

\begin{assumption}\label{ass:bounded_var}
    The stochastic gradient estimator $\nabla f_{\xi}: \cS \mapsto \cS$ is unbiased and has bounded variance. That is, $\mathbb{E}_{\xi \sim \cD}[\nabla f_{\xi}(X)]=\nabla f(X)$ for all $X \in \cS$ and there exists $\sigma \geq 0$ such that
    \begin{align*}
        \mathbb{E}_{\xi \sim \cD} \bigl[\|\nabla_i f_{\xi}(X)-\nabla_i f(X)\|^2_{(i) \star}\bigr] \leq \sigma^2, \quad \forall X \in \cS, \, i = 1, \dots, p.
    \end{align*}
\end{assumption}
Note that the choice of norm in \Cref{ass:bounded_var} is not restrictive: in finite-dimensional spaces, all norms are equivalent, so variance bounds remain valid up to a constant factor when compared to those based on the standard Euclidean norm.
The following result establishes the convergence properties.

\begin{theorem}\label{theorem:3}
    Let Assumptions \ref{ass:generalized-smoothness} and \ref{ass:bounded_var} hold and fix $\varepsilon>0$. Let $X^0,\ldots,X^{K-1}$ be the iterates of \algname{Gluon} (Algorithm \ref{algo:stochastic}) run with \( \beta^k = 1 - (k+1)^{-1/2} \), \( t_i^k = t_i (k+1)^{-3/4} \) for some \( t_i>0\), and $M^0_i=\nabla_i f_{\xi^0}(X^0)$. Then
    \begin{align}\label{eq:stoch_rate}
        \min\limits_{k=0,\ldots,K-1} \sum\limits_{i=1}^p \frac{1}{12L^1_i} \Exp{\| \nabla _i f(X^k) \|_{(i) \star}} \lesssim \frac{\Delta^0}{K^{1/4}} + \frac{1}{K^{1/4}} \sum\limits_{i=1}^p \left[ \frac{\sigma}{L^1_i} + \frac{L^0_i}{(L^1_i)^2}\right],
    \end{align}
    where $\Delta^0:=f(X^0) - f^{\inf}$ and the notation $\lesssim$ hides numerical constants and logarithmic factors.
\end{theorem}

For $p=1$, our rate in \eqref{eq:stoch_rate} recovers the complexity for finding a stationary point of $(L^0, L^1)$-smooth functions established by \citet{hubler2024parameter} for normalized \algname{SGD} with momentum.
When $p\geq 1$, compared to existing guarantees for \algname{Gluon}, our \Cref{theorem:3} operates under the significantly more general \Cref{ass:generalized-smoothness} and uniquely supports training with larger, non-constant stepsizes $t_i^k \propto k^{-3/4}$. In contrast, prior analyses prescribe constant, vanishingly small stepsizes $t_i^k \equiv t_i \propto K^{-3/4}$, tied to the \emph{total} number of iterations $K$ (see \Cref{table1}).

\section{Experiments}
\label{sec:experiments}

Below, we highlight selected experimental results for the \algname{unScion} optimizer, a special case of \algname{Gluon} (see \Cref{sec:lmo_examples}). Additional details and further experiments are provided in \Cref{appendix:more_exp}.\footnote{Code for all experiments is available \href{https://github.com/artem-riabinin/Experiments-estimating-smoothness-for-NanoGPT-and-CNN}{here}.}

\subsection{Training NanoGPT on FineWeb}
\label{sec:nanogpt-fineweb}

In the first set of experiments, we aim to verify layer-wise \((L^0, L^1)\)-smoothness (\Cref{ass:generalized-smoothness}). To this end, we train the \texttt{NanoGPT} model with 124M parameters on the \texttt{FineWeb} dataset, leveraging two open-source GitHub repositories~\citep{jordan2024moddednanogpt, code_scion}. We use the \algname{unScion} optimizer, i.e., \algname{Gluon} with the norm choices as in \eqref{eq:scion_llm_main}.
We adopt the hyperparameters from~\citet[Table~7]{pethick2025training}, mapping their values \( \gamma = 0.00036 \), \( \rho_2 = 50 \), and \( \rho_3 = 3000 \) into our notation as follows: \( t_i^k \equiv \gamma \rho_2 = 0.018 \) for \( i = 1, \dots, p-1 \) (corresponding to the transformer block layers), and \( t_p^k \equiv \gamma \rho_3 = 1.08 \) (token embeddings and output projections, due to weight sharing).
We set the number of warmdown iterations to $0$ to keep the learning rates constant throughout training. The model is trained for 5{,}000 iterations in accordance with the Chinchilla scaling laws to ensure compute-optimal training.

In Figures \ref{fig:3}, \ref{fig:1}, \ref{fig:2}, we plot the estimated \emph{trajectory smoothness} as a function of the iteration index \(k\)
\begin{align}\label{eq:dir_smooth}
    \hat{L}_i[k] \eqdef \frac{\|\nabla_i f_{\xi^{k+1}} (X^{k+1}) - \nabla_i f_{\xi^k} (X^{k}) \|_{(i)\star}}{\|X_i^{k+1} - X_i^{k}\|_{(i)}}
\end{align}
for parameter groups from the embedding layer and 4th and 8th transformer blocks (with similar trends observed across all blocks). 
We compare this to the approximation
$$\hat{L}_i^{\text{approx}}[k] \eqdef L_i^0 + L_i^1 \|\nabla_i f_{\xi^{k+1}}(X^{k+1})\|_{(i) \star}, $$
where $L^0_i, L^1_i\geq0$ are fitted to minimize the Euclidean error between $\hat{L}_i[k]$ and $\hat{L}_i^{\text{approx}}[k]$, with hinge-like penalty on underestimation (see \Cref{sec:nanogpt_l0l1}).
The close alignment between these curves implies that Assumption~\ref{ass:generalized-smoothness} is approximately satisfied along the training trajectories.
Based on the estimated values of $L^0_i$ and $L^1_i$, assuming that \Cref{ass:generalized-smoothness} holds and ignoring gradient stochasticity, \Cref{theorem:1} suggests the stepsizes
\begin{align}\label{eq:emp_step}
    \begin{split}
        & t^k_i = \frac{\|\nabla_i f_{\xi^k}(X^k)\|_{(i) \star}}{L^0_i + L^1_i\|\nabla _i f_{\xi^k}(X^k)\|_{(i) \star}} \approx \frac{1}{L^1_i} \approx \frac{1}{70} \approx 0.014, \quad i = 1, \dots, p-1, \\
        & t^k_p = \frac{\|\nabla_p f_{\xi^k}(X^k)\|_{(p) \star}}{L^0_p + L^1_p\|\nabla _p f_{\xi^k}(X^k)\|_{(p) \star}} \approx \frac{1}{L^1_p} \approx \frac{1}{1.3} \approx 0.77.
    \end{split}
\end{align}
Remarkably, these values align closely with the manually tuned values reported earlier, again underscoring the predictive power of our theoretical prescriptions (see \Cref{sec:deterministic}).

\begin{figure}[t]
    \centering
    
    \begin{subfigure}{0.32\textwidth}
        \includegraphics[width=\textwidth]{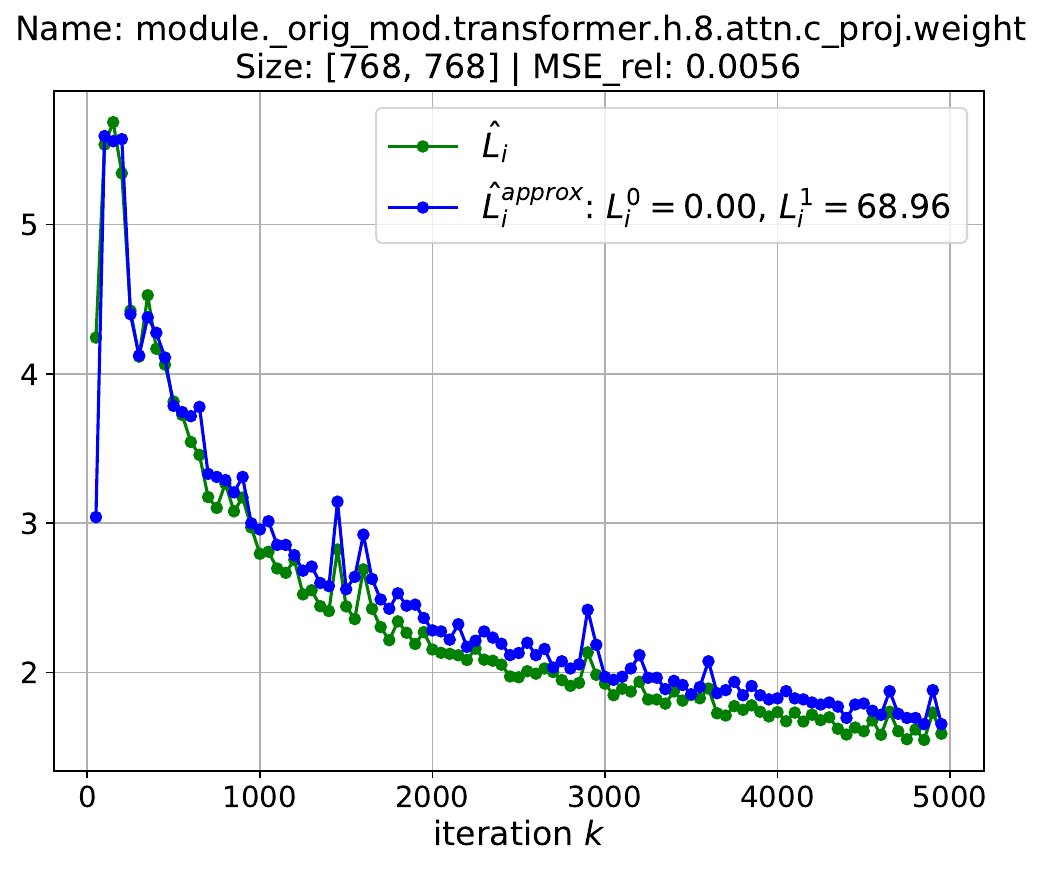}
    \end{subfigure}
    \hfill
    \begin{subfigure}{0.32\textwidth}
        \includegraphics[width=\textwidth]{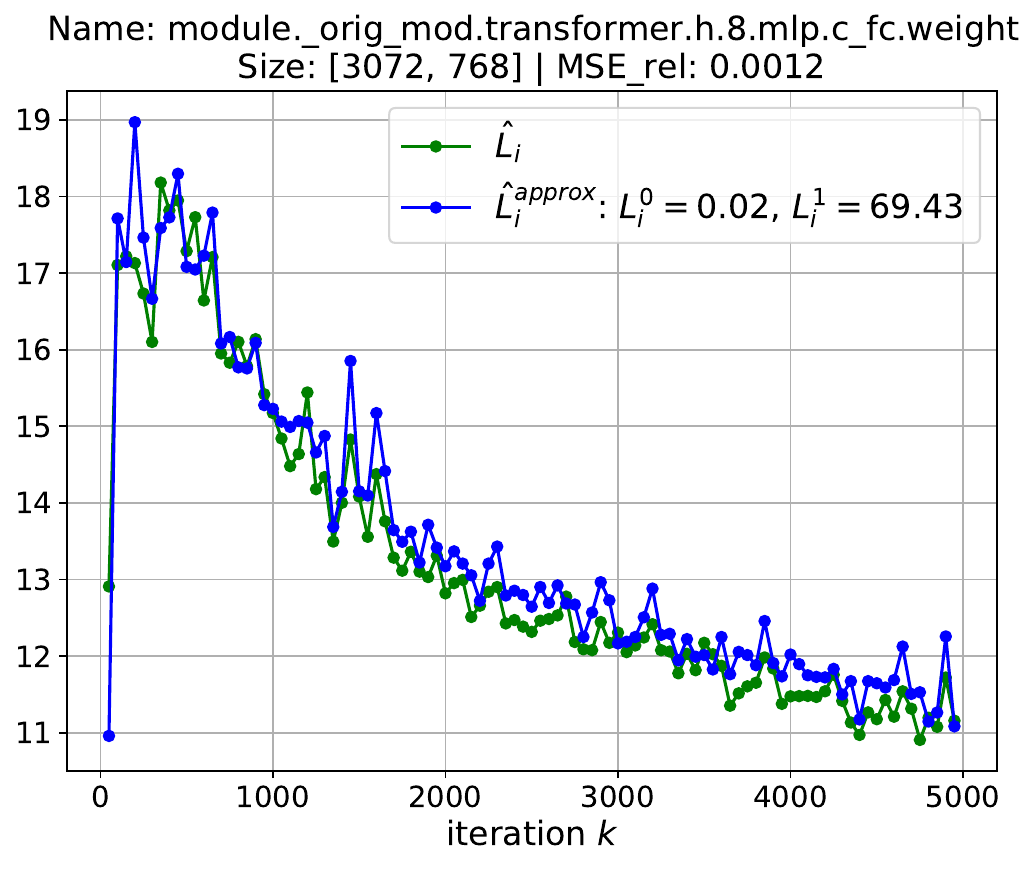}
    \end{subfigure}
    \hfill
    \begin{subfigure}{0.32\textwidth}
        \includegraphics[width=\textwidth]{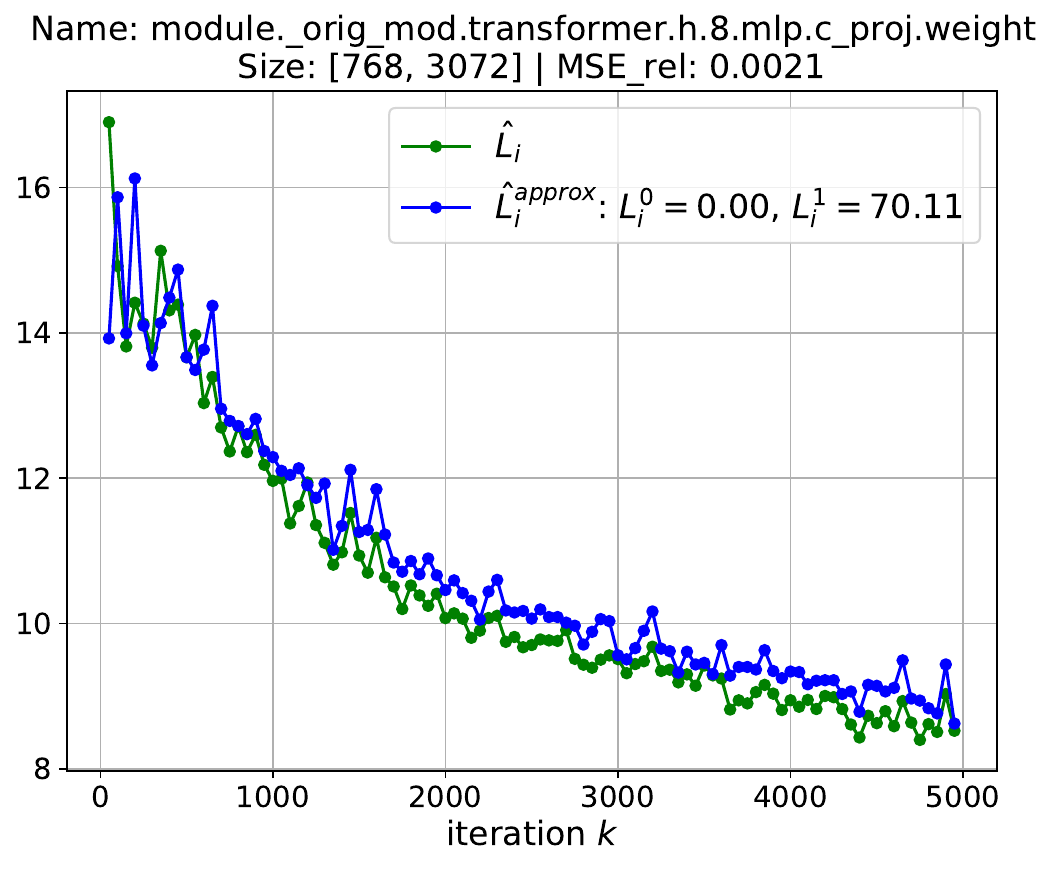}
    \end{subfigure}
    
    \caption{Validation of \Cref{ass:generalized-smoothness} for the 8th transformer block in \texttt{NanoGPT-124M} along training trajectories of \algname{unScion}.}
    \label{fig:3}
\end{figure}

\paragraph{Effect of scaling factors.}

We next evaluate the impact of the learning rate scaling factors $\rho_2$ and $\rho_3$ on the performance of the \algname{unScion} optimizer. For consistency, all other hyperparameters are fixed as described earlier. As a baseline, we include results obtained with the \algname{AdamW} optimizer, using the hyperparameter settings from Section~\ref{sec:Adam_exp}.
\Cref{fig:effect_scaling_factors} presents (a) validation curves for both optimizers, with varying $\rho_3$ in \algname{unScion}, and (b) the final validation loss for \algname{unScion} across different combinations of $\rho_2$ and $\rho_3$.
The best performance is achieved with $\rho_2 = 50$ and $\rho_3 = 3000$, i.e., $t_i^k = 0.018$ for $i = 1, \dots, p-1$ and $t^k_p = 1.08$, consistent with our theoretical prediction~\eqref{eq:emp_step}. This supports the use of non-uniform scaling across layers, with larger stepsizes for the embedding layer.

\begin{figure}[t]
    \centering
    \begin{subfigure}{0.6\textwidth}
        \includegraphics[width=\textwidth]{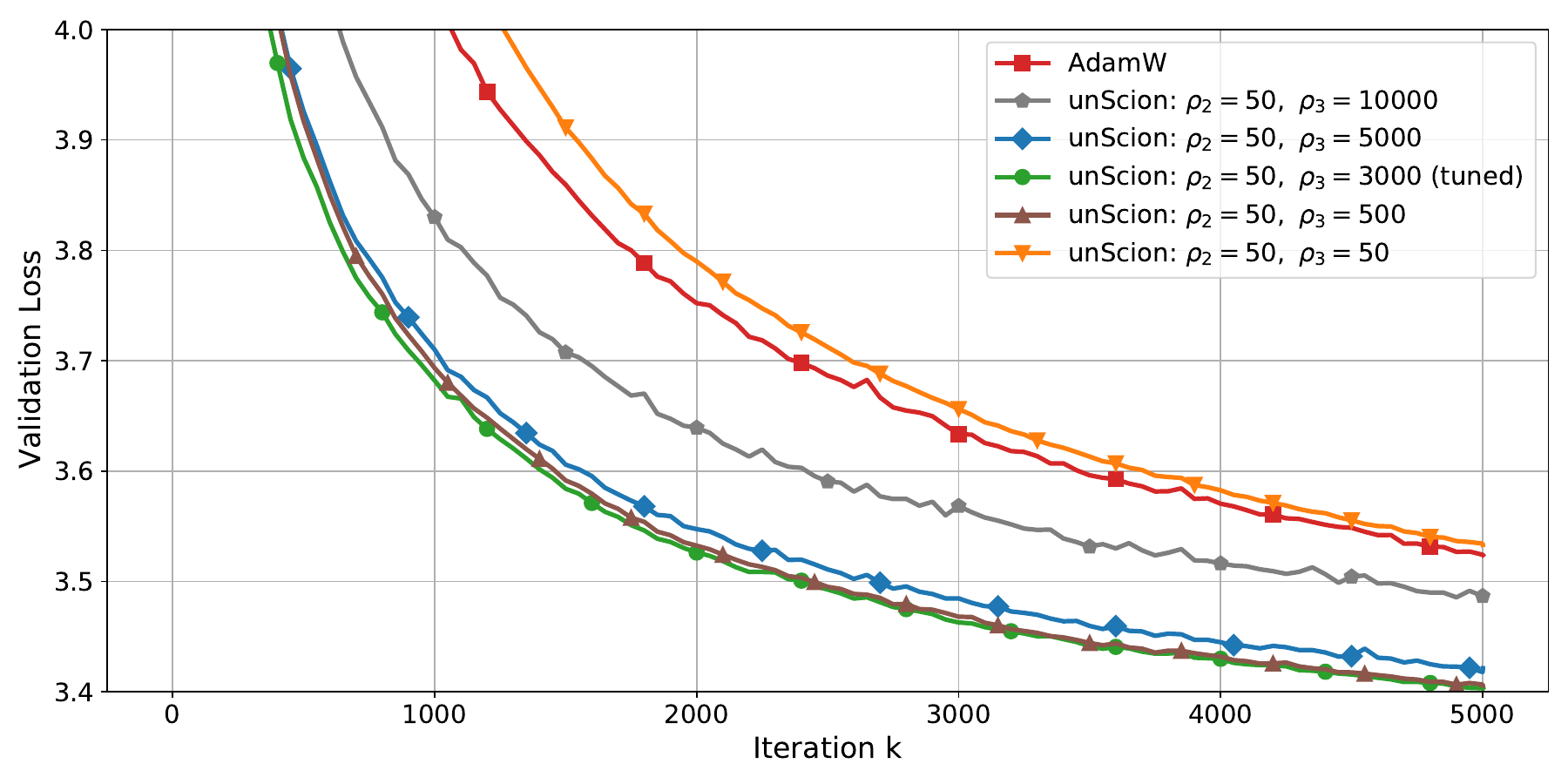}
        \caption{}
    \end{subfigure}
    \hfill
    \begin{subfigure}{0.36      \textwidth}
        \includegraphics[width=\textwidth]{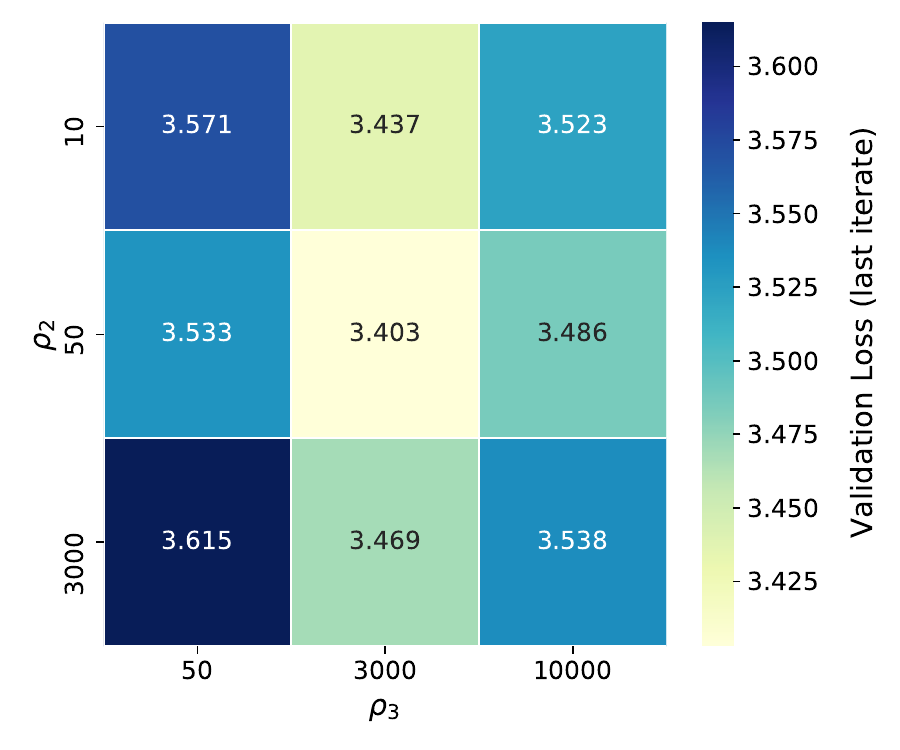}
        \caption{}
    \end{subfigure}
    \caption{(a) Validation curves for \algname{AdamW} and \algname{unScion} with varying $\rho_3$ values; (b) Heatmap of validation loss from the last iteration of \algname{unScion} across different combinations of $\rho_2$ and $\rho_3$.}
    \label{fig:effect_scaling_factors}   
\end{figure}

\paragraph{Additional ablation studies.}
In Appendix \ref{appendix:gensmooth_euclid}, we present an ablation study demonstrating that specialized norms provide a better approximation of trajectory smoothness compared to the standard Euclidean norm.
Appendix~\ref{sec:Adam_exp} demonstrates that the layer-wise $(L^0, L^1)$-smoothness model also closely approximates trajectory smoothness during \algname{AdamW} training. Notably, we observe a similar gap between transformer and embedding layers as with \algname{Scion}, suggesting that smoothness statistics from \algname{AdamW} training can guide per-layer learning rate tuning in \algname{Scion}.

\subsection{Training CNN on CIFAR-10}\label{sec:cnn}

In this experiment, we further validate layer-wise $(L^0, L^1)$-smoothness by training a \texttt{CNN} model on the \texttt{CIFAR-10} dataset, following implementations from two open-source GitHub repositories~\citep{jordan2024cifar10, code_scion}. The model is trained using the \algname{unScion} optimizer~\eqref{eq:scion_cnn} with full-batch gradients \( \nabla_i f \), no momentum and no learning rate decay (results for the stochastic case are reported in \Cref{appendix:cifar}). Other hyperparameters are as in \citet[Table10]{pethick2025training}, except that we train for more epochs.

Similar to the \texttt{NanoGPT} experiments discussed in \Cref{sec:nanogpt-fineweb}, we plot the estimated (non-stochastic) trajectory smoothness
$\hat{L}_i[k] \eqdef \|\nabla_i f (X^{k+1}) - \nabla_i f (X^{k}) \|_{(i)\star} / \|X_i^{k+1} - X_i^{k}\|_{(i)}$
alongside its approximation
$
\hat{L}_i^{\text{approx}}[k] \eqdef L_i^0 + L_i^1 \|\nabla_i f (X^{k+1})\|_{(i)\star}
$
for selected parameter groups. In this experiment, we consider a simplified variant of \Cref{ass:generalized-smoothness}, setting $L^0_i = 0$, and estimate $L^1_i\geq0$ using the same procedure as in \Cref{sec:nanogpt-fineweb}.

Figure~\ref{fig:5} presents the results, demonstrating that \Cref{ass:generalized-smoothness} is approximately satisfied along the training trajectory.
When this condition holds with $L^0_i=0$, \Cref{theorem:1} guarantees convergence under the stepsize choice $t^k_i \equiv t_i =\nicefrac{1}{L^1_i}$.
In this setting, the estimated $L^1_i$ values (shown in Figure~\ref{fig:5}) are $L^1_i \approx 3$ for all parameter groups except for the classification head weights $X_p$, where $L^1_p \approx 0.03$.
This roughly two-orders-of-magnitude difference justifies the much larger radius $t_p^k$ used for the head weights in the tuned configuration reported in \citet[Table 10]{pethick2025training}.

\begin{figure}[H]
    \centering
    \begin{subfigure}{0.32\textwidth}
        \includegraphics[width=\textwidth]{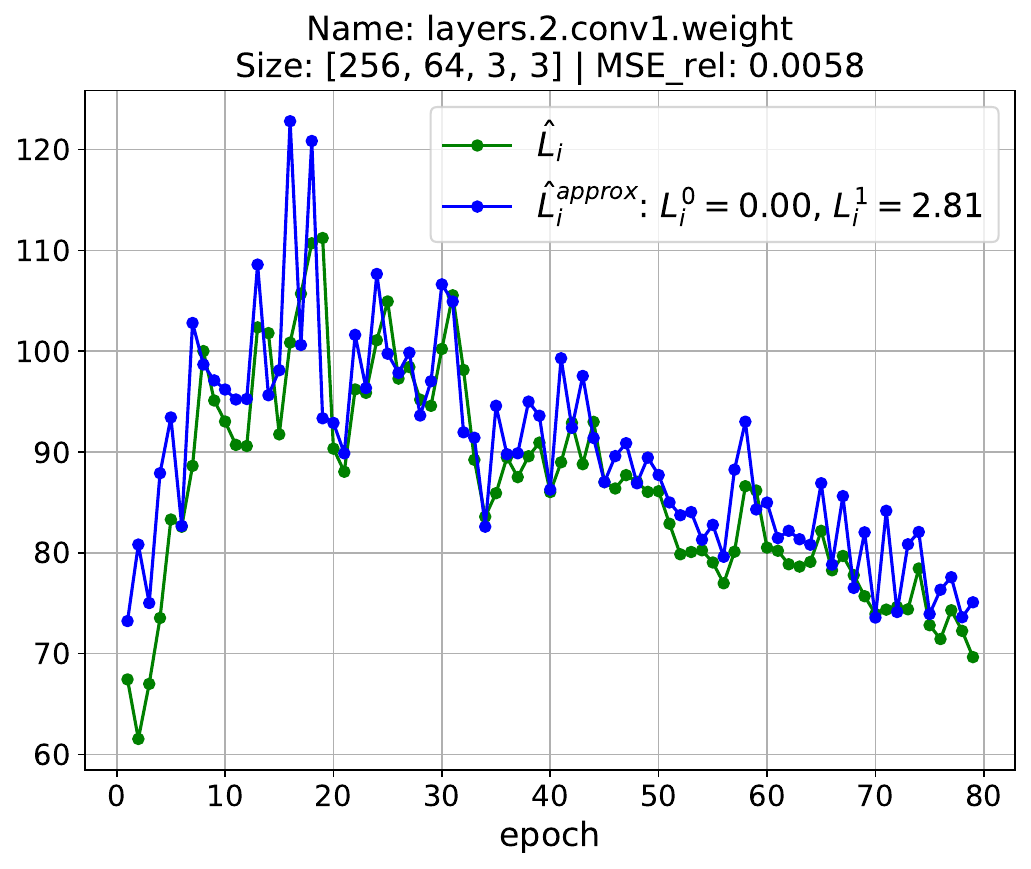}
    \end{subfigure}
    \hfill
    \begin{subfigure}{0.32\textwidth}
        \includegraphics[width=\textwidth]{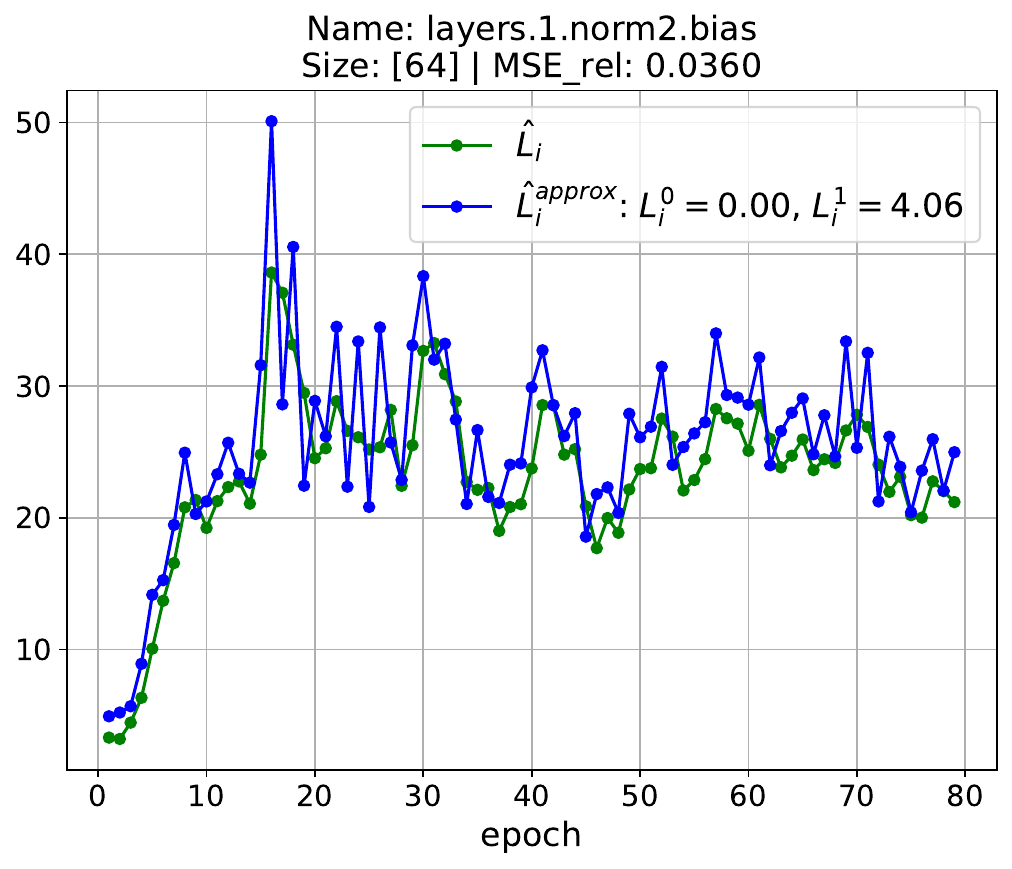}
    \end{subfigure}
    \hfill
    \begin{subfigure}{0.335\textwidth}
        \includegraphics[width=\textwidth]{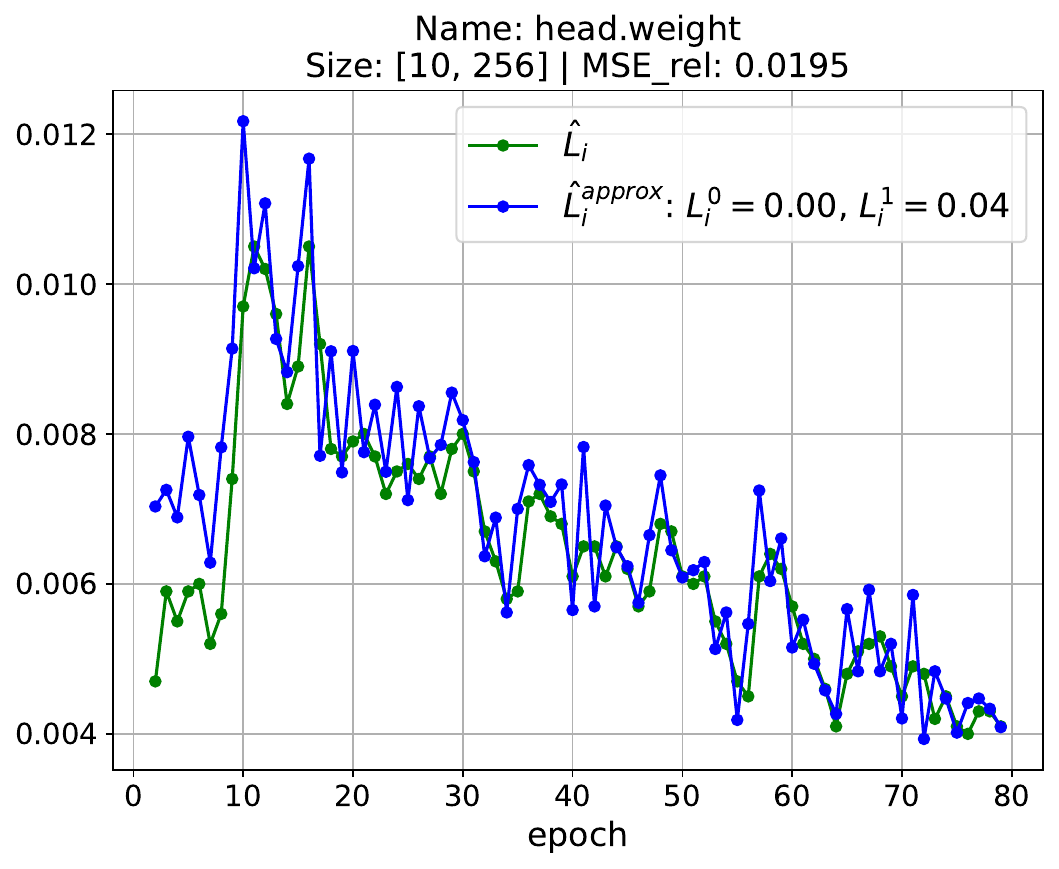}
    \end{subfigure}
    
    \caption{Validation of \Cref{ass:generalized-smoothness} for different groups of parameters in \texttt{CNN} along training trajectories of \algname{unScion} with full-batch gradients.
    }
    \label{fig:5}
\end{figure}

\section{Conclusion and future work}\label{sec:conclusion} 

In this work, we propose \algname{Gluon}, an LMO-based optimization method that recovers state-of-the-art optimizers such as \algname{Muon} and \algname{Scion} as special cases. We develop a principled analytical framework for layer-wise optimization based on a novel \emph{layer-wise $(L^0, L^1)$-smoothness} assumption, which captures the anisotropic structure of modern deep networks. This assumption enables sharper and more general convergence guarantees and, unlike prior analyses, yields theoretical stepsizes that closely match those found via finetuning. Our framework thus provides \emph{the first rigorous and practically predictive analysis of modern layer-wise optimizers}. Experiments confirm that the assumption holds approximately throughout training, reinforcing its practical relevance. Together, these results offer a refined foundation for structured optimization in deep learning.

While this work resolves two key theoretical gaps (Sections \ref{sec:layerwise} and \ref{sec:orig_steps}), it also highlights important directions for future research.
Our analysis assumes exact LMO computations, whereas practical implementations use approximations (\Cref{sec:exp_details}). Additionally, our stochastic guarantees (\Cref{theorem:3}) rely on the widely adopted bounded variance assumption, which may not hold in certain scenarios, e.g., under subsampling \citep{khaled2020better}.
Finally, our support for adaptive stepsizes is currently restricted to the deterministic setting. While they also perform well empirically in the stochastic regime  (\Cref{sec:nanogpt-fineweb}), a complete theoretical justification remains an open challenge.

In summary, although we make substantial progress by closing the two most critical gaps--establishing a realistic generalized smoothness model and aligning analysis with actual implementations--no single work can exhaust the subject. The field remains open, with many fruitful directions left to pursue.

\section*{Acknowledgements}
    The research reported in this publication was supported by funding from King Abdullah University of Science and Technology (KAUST): i) KAUST Baseline Research Scheme, ii) CRG Grant ORFS-CRG12-2024-6460, iii) Center of Excellence for Generative AI, under award number 5940, and iv) SDAIA-KAUST Center of Excellence in Artificial Intelligence and Data Science.

\bibliographystyle{plainnat}
\bibliography{ref}


\newpage
\appendix

\tableofcontents

\newpage

\section{Auxiliary lemmas}

\begin{lemma}\label{lemma:generalized-smoothness}
Let $f: \cS \mapsto \mathbb{R}$ satisfy Assumption~\ref{ass:generalized-smoothness}. Then, for any $X, Y \in \cS$, we have
\begin{align*}
\left|f(Y)-f(X)-\left\langle \nabla f(X), Y-X\right\rangle\right| \leq \sum_{i=1}^p \frac{L^0_i + L^1_i \| \nabla _i f(X)\|_{(i) \star}}{2}\|Y_i-X_i\|_{(i)}^2.
\end{align*}
\end{lemma}

\begin{proof}
For all $X, Y \in \cS$ we have
\begin{align*}
f(Y) & =f(X)+\int_0^1\left\langle \nabla f(X+\tau(Y-X)), Y-X\right\rangle d \tau \\
&=f(X)+\left\langle \nabla f(X), Y-X\right\rangle+\int_0^1\left\langle \nabla f(X+\tau(Y-X))-\nabla f(X), Y-X\right\rangle d \tau.
\end{align*}

Therefore, using the Cauchy-Schwarz inequality and Assumption~\ref{ass:generalized-smoothness}, we obtain
\begin{eqnarray*}
&&\hspace{-1cm}\left|f(Y)-f(X)-\left\langle \nabla f(X), Y-X\right\rangle\right| \\
&\leq&\left|\int_0^1 \sum_{i=1}^p \left\langle \nabla _i f(X+\tau(Y-X))-\nabla _i f(X), Y_i-X_i\right\rangle_{(i)} d \tau \right| \\
&\leq&\int_0^1 \sum_{i=1}^p \left| \left\langle \nabla _i f(X+\tau(Y-X))-\nabla _i f(X), Y_i-X_i\right\rangle_{(i)} \right| d \tau \\ 
& \leq& \int_0^1 \sum_{i=1}^p \left\|\nabla _i f(X+\tau(Y-X))-\nabla _i f(X)\right\|_{(i) \star} \|Y_i-X_i\|_{(i)} d \tau \\
& \leq& \int_0^1 \sum_{i=1}^p \tau \left( L^0_i + L^1_i \| \nabla _i f(X) \|_{(i) \star} \right) \|Y_i-X_i\|_{(i)}^2 d \tau\\
&=&\sum_{i=1}^p \frac{L^0_i + L^1_i \| \nabla _i f(X)\|_{(i) \star}}{2}\|Y_i-X_i\|_{(i)}^2.
\end{eqnarray*}
\end{proof}

\begin{lemma}\label{lemma:L0vsL}
    Suppose that $f$ is $L$-smooth with respect to the norm defined in \eqref{eq:max_norm}, i.e.,
    \begin{align*}
        \norm{\nabla f(X) - \nabla f(Y)}_{\max\star} \leq L \norm{X-Y}_{\max},
    \end{align*}
    where $X = [X_1, \dots, X_p]$ and $Y = [Y_1, \dots, Y_p]$ with $X_i, Y_i \in \cS_i$.
    Then \Cref{ass:generalized-smoothness} holds with $L^0_i \leq L$ and $L^1_i = 0$ for all $i=1,\ldots,p$.
\end{lemma}
\begin{proof}
    $L$-smoothness and the definition of the norm give
    \begin{align*}
        \sum_{i=1}^p \|\nabla_i f(X) - \nabla_i f(Y)\|_{(i) \star} \leq L \max\brac{\norm{X_1 - Y_1}_{(1)}, \ldots, \norm{X_p - Y_p}_{(p)}}
    \end{align*}
    for all $X, Y \in \cS$. In particular, choosing $X = [X_1, \dots, X_p]$ and $Y = [X_1, \dots, X_{j-1}, Y_j, X_{j+1}, \ldots X_p]$, we have
    \begin{align*}
        \|\nabla_j f(X) - \nabla_j f(Y)\|_{(j) \star} \leq \sum_{i=1}^p \|\nabla_i f(X) - \nabla_i f(Y)\|_{(i) \star} \leq L \norm{X_j - Y_j}_{(j)}
    \end{align*}
    for any $j\in\{1,\ldots,p\}$, proving the claim.
\end{proof}

\begin{lemma}\label{lemma:ineq}
    Suppose that $x_1, \ldots, x_p, y_1, \ldots, y_p \in \R$, $\max_{i\in[p]} |x_i| > 0$ and $z_1, \ldots, z_p > 0$. Then
    \begin{align*}
        \sum_{i=1}^p \frac{y_i^2}{z_i} \geq \frac{\parens{\sum_{i=1}^p x_i y_i}^2}{\sum_{i=1}^p z_i x_i^2}.
    \end{align*}
\end{lemma}
\begin{proof}
    Cauchy-Schwarz inequality gives
    \begin{align*}
        \parens{\sum_{i=1}^p x_i y_i}^2
        = \parens{\sum_{i=1}^p \frac{y_i}{\sqrt{z_i}} \sqrt{z_i} x_i}^2
        \leq \parens{\sum_{i=1}^p \frac{y_i^2}{z_i}} \parens{\sum_{i=1}^p z_i x_i^2}.
    \end{align*}
    Rearranging, we obtain the result.
\end{proof}

\begin{lemma}[Technical Lemma 10 by \citet{hubler2024parameter}]\label{lemma:tech1}
    Let $q \in(0,1)$, $p \geq 0$, and $p \geq q$. Further, let $a, b \in \mathbb{N}_{\geq 2}$ with $a \leq b$. Then
    \begin{align*}
    \sum_{k=a-1}^{b-1} (1+k)^{-p} \prod_{\tau=a-1}^k\left(1-(\tau+1)^{-q}\right) \leq (a-1)^{q-p} \exp \left(\frac{a^{1-q}-(a-1)^{1-q}}{1-q}\right).
    \end{align*}
\end{lemma}

\begin{lemma}[Technical Lemma 11 by \citet{hubler2024parameter}]\label{lemma:tech2}
    Let $t>0$ and for $k \in \mathbb{N}_{\geq 0}$, set $\beta^k =1-(k+1)^{-1 / 2}$, $t^k =t (k+1)^{-3 / 4}$, $t>0$. Then, for all $K \in \mathbb{N}_{\geq 1}$ the following inequalities hold:
    \begin{enumerate}[label=(\roman*)]
        \item $\sum_{k=0}^{K-1} t^k \sqrt{\sum_{\tau=0}^k (1-\beta^{\tau})^2 \prod_{\kappa=\tau+1}^k (\beta^\kappa)^2} \leq t\left(\frac{7}{2}+\sqrt{2 e^2} \log (K)\right)$,
        \item $\sum_{k=0}^{K-1} t^k \sum_{\tau=1}^k t^\tau \prod_{\kappa=\tau}^k \beta^\kappa \leq 7 t^2\left(3 +\log (K)\right)$.
    \end{enumerate}
\end{lemma}
\begin{proof}
    This is a direct consequence of Lemma~11 by~\citet{hubler2024parameter}. To obtain \textit{(ii)}, it suffices to take the limit as \( L^1 \to 0 \) in statement \textit{(ii)} of part~(b).
\end{proof}

\newpage

\section{Remarks on the theoretical results}

\subsection{Note on radii and stepsizes}\label{sec:step_note}

It is known (see, e.g., \citet[Theorem D.1]{gruntkowska2025ball}, who establish this for $\cS = \R^d$ under Euclidean norms; the extension to general normed vector spaces is entirely analogous) that if $g$ is a convex function, then the solution to the problem
\begin{align*}
    \argmin_{X\in\mathcal{B}^k} g(X)
\end{align*}
is unique and lies on the boundary of the ball $\mathcal{B}^k \eqdef \{ X \in \cS : \|X - X^k\| \leq t^k \}$ (unless $\nabla g(X^k) = 0$, i.e., $X^k$ is a stationary point of $g$).

This applies directly to the LMO subproblem solved at each iteration of \algname{Gluon} in \eqref{eq:update_rule_stoch}, since the objective $\langle M_i^k, X_i \rangle_{(i)}$ is a linear function of $X_i$, and hence convex. In other words, each LMO step moves the iterate from the center of the ball $X_i^k$ to a new point $X_i^{k+1}$ located on the boundary of $\mathcal{B}_i^k$, effectively traversing a distance of $t_i^k$ at each step. For this reason, we use the terms \emph{radius}, \emph{stepsize}, and \emph{learning rate} interchangeably.

\subsection{Note on prior analyses}

As presented, prior convergence results do not directly apply to the algorithms used in practice. However, there is a workaround. Specifically, some of the existing convergence guarantees \citep{kovalev2025muon, pethick2025training} expressed in terms of the flat vector $x$ are transferable to the structured parameters $X=[X_1, \ldots, X_l]\in\cS$ by employing the max-norm \cite{bernstein2024modular, larges2024calable}, defined as
\begin{align}\label{eq:max_norm}
    \norm{X}_{\max} \eqdef \max\brac{\norm{X_1}_{(1)}, \ldots, \norm{X_p}_{(p)}},
\end{align}
with corresponding dual norm $\|Y\|_{\max \star} = \sup_{\|X\|_{\text{max}} \leq 1} \langle X, Y \rangle = \sum_{i=1}^p \|Y_i\|_{(i) \star}$. Nevertheless, these works do not make this connection explicit, and an additional layer of analysis is required to ensure the guarantees meaningfully extend to the structured practical setting. Even if such a translation was attempted, the global treatment introduces serious practical limitations.
For example, real-world training pipelines tune parameters on a per-layer basis, reflecting the heterogeneous structure of deep networks. Max-norm-based guarantees overlook this variability and offer no mechanism for per-layer control in hyperparameter selection.

\newpage

\section{Deterministic case}

We begin by considering the deterministic counterpart of \algname{Gluon}, as formalized in \Cref{algo:deterministic}. We first review several existing algorithms that fall within this framework (\Cref{sec:lmo_examples}), followed by a proof of \Cref{theorem:1} (\Cref{sec:det_proof}). Finally, we present an additional convergence guarantee under the layer-wise Polyak-Łojasiewicz (PŁ) condition (\Cref{sec:pl}).

\begin{algorithm}[t]
    \caption{Deterministic Adaptive Layer-Wise LMO-based Optimizer}
    \label{algo:deterministic}
    \begin{algorithmic}[1]
    \State \textbf{Input:} Initial model parameters $X^0 = [X_1^0, \dots, X_p^0] \in \cS$
    \For{$k = 0, 1, \dots, K-1$}
        \For{$i = 1, 2, \dots, p$}
            \State Choose adaptive stepsize/radius $t_i^k>0$ for layer $i$
            \State Update parameters for layer $i$ via LMO over $\mathcal{B}_i^k \eqdef \{ X_i \in \cS_i : \|X_i - X_i^k\|_{(i)} \leq t_i^k \}$:
            \begin{equation}
                \label{eq:update_rule}
                X_i^{k+1} = \mathrm{LMO}_{\mathcal{B}_i^k} \left( \nabla_i f(X^k) \right)
                \eqdef \underset{X_i \in \mathcal{B}_i^k}{\arg\min} \ \langle \nabla_i f(X^k), X_i \rangle_{(i)}
            \end{equation}
        \EndFor
        \State Update full vector: \
        $
        X^{k+1} = [X_1^{k+1}, \dots, X_p^{k+1}]
        $
    \EndFor
    \end{algorithmic}
\end{algorithm}

\subsection{Special cases of the LMO framework}\label{sec:lmo_examples}

As outlined in \Cref{sec:det_examples}, deterministic \algname{Gluon} encompasses a general class of algorithms, parameterized by the choice of norms $\|\cdot\|_{(i)}$ in the LMO. We now provide a more detailed discussion of the most notable special cases.

\paragraph{Layer-wise normalized GD \citep{wei2018blocknormalized}.}  Let \( \|\cdot\|_{(i)}= \|\cdot\|_{2 \rightarrow 2} \) for each parameter group and assume that $n_i = 1$ for all $i=1,\ldots,p$. In this case, the spectral norm reduces to the standard Euclidean norm \( \|\cdot\|_2 \), yielding the update rule
\begin{equation*}
    X_i^{k+1} = X_i^k - t^k_i \frac{\nabla _i f(X^k)}{\|\nabla _i f(X^k)\|_2}, \quad i = 1, \dots, p, 
\end{equation*}
which corresponds to layer-wise normalized \algname{GD}. With a suitable choice of $t_i^k$ (see Theorem~\ref{theorem:1}), the method can also recover the Gradient Method for $(L^0, L^1)$-smooth functions \citep{vankov2025optimizing}.

\paragraph{Layer-wise signGD \citep{balles2020geometrysigngradientdescent}.} Suppose that $\|\cdot\|_{(i)}=\|\cdot\|_{1 \rightarrow \infty}$ for each parameter group, with \( n_i = 1 \) for all $i=1,\ldots,p$. Then, \( \|\cdot\|_{1 \rightarrow \infty} \) reduces to \( \|\cdot\|_{\infty} \), and the update becomes
\begin{equation*}
    X_i^{k+1} = X_i^k - t^k_i \text{sign}\parens{\nabla _i f(X^k)}, \quad i = 1, \dots, p,
\end{equation*}
where the sign function is applied element-wise. This is equivalent to layer-wise \algname{signGD}. 

\paragraph{Muon \citep{jordan2024muon}.} Here, the spectral norm \( \|\cdot\|_{2 \rightarrow 2} \) is used for all parameter groups, without restrictions on $n_i$. In this case, it can be shown that \eqref{eq:update_rule} is equivalent to
\begin{equation}\label{eq:muon}
    X_i^{k+1} = X_i^k - t^k_i U_i^k \left(V_i^k\right)^{\top}, \quad i = 1, \dots, p, 
\end{equation}
where $\nabla _i f(X^k) = U^k_i \Sigma^k_i \left(V^k_i\right)^{\top}$ is the singular value decomposition \citep{bernstein2024old}. This is exactly the per-layer deterministic version of the \algname{Muon} optimizer. In practical LLM training, a more general variant of \eqref{eq:muon} incorporating stochasticity and momentum is applied to the intermediate layers, while the input and output layers are optimized using other methods.

\paragraph{Unconstrained Scion \citep{pethick2025training}.} We can also recover two variants of \algname{unScion} introduced by~\citet{pethick2025training}: one for training LLMs on next-token prediction, and another for training CNNs for image classification.

\begin{itemize} 
    \item \textbf{Training LLMs.} Define the norms \( \|\cdot\|_{(i)} \) as follows: for \( i = 1, \dots, p-1 \), corresponding to weight matrices of transformer blocks, set \(\|\cdot\|_{(i)} = \sqrt{\nicefrac{n_i}{m_i}} \|\cdot\|_{2 \to 2},\)
    and for the last group~\( X_p \), representing the embedding and output layers (which coincide under the weight sharing regime considered here), let \(\|\cdot\|_{(p)} = n_p \|\cdot\|_{1 \to \infty}\).
    In this case, \eqref{eq:update_rule} becomes
    \begin{equation}\label{eq:scion_llm}
    \begin{aligned}
        X_i^{k+1} &= X_i^k - t^k_i \sqrt{\frac{m_i}{n_i}} U_i^k \left(V_i^k\right)^{\top}, \quad i = 1, \dots, p-1, \\ 
        X_p^{k+1} &= X_p^k - \frac{t^k_p}{n_p} \text{sign}\left(\nabla _p f(X^k)\right), 
    \end{aligned}
    \end{equation}
    where $\nabla _i f(X^k) = U^k_i \Sigma^k_i \left(V^k_i\right)^{\top}$ is the singular value decomposition. This is equivalent to deterministic layer-wise \algname{unScion} optimizer without momentum.
    A more general variant, incorporating stochasticity and momentum and applied to all layers, was shown by \citet{pethick2025training} to outperform \algname{Muon} on LLM training tasks.
    
    \item \textbf{Training CNNs.} The main difference in the CNN setting  is the presence of not only 2D weight matrices, but also 1D bias vectors and 4D convolutional kernels parameters. Biases are 1D tensors of shape $\mathbb{R}^{C_i^{out}}$, for which we use scaled Euclidean norms. Convolutional parameters (conv) are 4D tensors with shapes \(\mathbb{R}^{C_i^{out} \times C_i^{in} \times k \times k}\), where \(C_i^{out}\) and \(C_i^{in}\) denote the number of output and input channels, and \(k\) is the kernel size. To compute norms, we reshape each 4D tensor to a 2D matrix of shape \(\mathbb{R}^{C_i^{out} \times C_i^{in}k^2}\), and then apply a scaled \( \|\cdot\|_{2 \rightarrow 2} \) norm. This yields the norm choices \(\|\cdot\|_{(i)} = \sqrt{\nicefrac{1}{C_i^{out}}}\|\cdot\|_2\) for biases, \( \|\cdot\|_{(i)} = k^2 \sqrt{\nicefrac{C_i^{in}}{C_i^{out}}} \|\cdot\|_{2 \to 2} \) for conv, and \(\|\cdot\|_{(p)} = n_p \|\cdot\|_{1 \to \infty}\) for the last group \( X_p \), associated with classification head weights. Then, it can be shown that (\ref{eq:update_rule}) is equivalent to
    \begin{equation}\label{eq:scion_cnn}
    \begin{aligned}
        X_i^{k+1} &= X_i^k - t^k_i \sqrt{C_i^{out}} \frac{\nabla _i f(X^k)}{\|\nabla _i f(X^k)\|_2}, \quad\quad\, \text{(for biases)}, \\
        X_i^{k+1} &= X_i^k - t^k_i \frac{1}{k^2} \sqrt{\frac{C_i^{out}}{C_i^{in}}} U_i^k \left(V_i^k\right)^{\top}, \quad \text{(for conv)}, \\
        X_p^{k+1} &= X_p^k - \frac{t^k_p}{n_p} \text{sign}\left(\nabla _p f(X^k)\right), \quad\quad\, \text{(for head)}
    \end{aligned}
    \end{equation}
    where $\nabla _i f(X^k) = U^k_i \Sigma^k_i \left(V^k_i\right)^{\top}$ is the singular value decomposition. This corresponds to the deterministic layer-wise \algname{unScion} optimizer without momentum.
\end{itemize}

\subsection{Proof of Theorem \ref{theorem:1}}\label{sec:det_proof}

We now state and prove a generalization of \Cref{theorem:1}.
\begin{theorem}\label{theorem:1_full}
    Let \Cref{ass:generalized-smoothness} hold and fix $\varepsilon>0$. Let $X^0,\ldots,X^{K-1}$ be the iterates of deterministic \algname{Gluon} (Algorithm \ref{algo:deterministic}) run with stepsizes $t^k_i = \frac{\|\nabla _i f(X^k)\|_{(i) \star}}{L^0_i + L^1_i\|\nabla _i f(X^k)\|_{(i) \star}}$.
    Then,
    \begin{enumerate}
        \item In order to reach the precision
        \begin{align*}
            \min_{k=0,\ldots,K-1} \sum_{i=1}^p \left\| \nabla _i f(X^{k}) \right\|_{(i) \star} \leq \epsilon, 
        \end{align*}
        it suffices to run the algorithm for
        \begin{align}\label{eq:det_rate1}
            K = \left\lceil \frac{2 \Delta^0 \sum_{i=1}^p L^0_i}{\epsilon^2} + \frac{2 \Delta^0 L^1_{\max}}{\epsilon} \right\rceil
        \end{align}
        iterations;
        
        \item In order to reach the precision
        \begin{align}\label{eq:det_precision2}
            \min\limits_{k=0,\ldots,K-1} \sum\limits_{i=1}^p \left[\frac{\frac{1}{L_i^1}}{\frac{1}{p} \sum_{j=1}^p \frac{1}{L_j^1}} \norm{\nabla _i f(X^k)}_{(i) \star}\right] \leq \varepsilon,
        \end{align}
        it suffices to run the algorithm for
        \begin{align}\label{eq:det_rate2}
            K = \left\lceil \frac{2 \Delta^0 \parens{\sum_{i=1}^p \frac{L^0_i}{(L_i^1)^2}}}{\varepsilon^2 \parens{\frac{1}{p} \sum_{j=1}^p \frac{1}{L_j^1}}^2} 
            + \frac{2 \Delta^0}{\varepsilon \parens{\frac{1}{p} \sum_{j=1}^p \frac{1}{L_j^1}}} \right\rceil
        \end{align}
        iterations,
    \end{enumerate}
    where $\Delta^0\eqdef f(X^0) - \inf_{X \in \cS} f(X)$ and $L^1_{\max}\eqdef\max_{i=1,\dots,p} L^1_i$.
\end{theorem}

\begin{remark}\label{rem:det}
    Let us compare bounds \eqref{eq:det_rate1} and \eqref{eq:det_rate2}. Due to the reweighting of the gradient component norms in \eqref{eq:det_precision2}, the rates are not exactly equivalent. Nevertheless, both use weights that sum to $p$, ensuring a fair comparison. Obviously, $(\nicefrac{1}{p} \sum_{j=1}^p \nicefrac{1}{L_j^1})^{-1} \leq L^1_{\max}$,
    so the second term in~\eqref{eq:det_rate2} is always no worse than its counterpart in~\eqref{eq:det_rate1}. The comparison of the first terms, however, depends on how the sequences $\{L^0_i\}$ and $\{L^1_i\}$ relate: if larger values of $L^0_i$s tend to be attached to smaller values of $L^1_i$, then the first term in~\eqref{eq:det_rate2} improves over that in~\eqref{eq:det_rate1}, while for a positive correlation the opposite is true. Indeed, in the extreme case when $L^0_1 \geq \ldots \geq L^0_p$ and $L^1_1 \leq \ldots \leq L^1_p$ (or the reverse ordering), Chebyshev's sum inequality implies that
    \begin{align*}
        \frac{\sum\limits_{i=1}^p \frac{L^0_i}{(L_i^1)^2}}{\parens{\frac{1}{p} \sum\limits_{j=1}^p \frac{1}{L_j^1}}^2}
        \leq \frac{\parens{\frac{1}{p} \sum\limits_{i=1}^p \frac{L^0_i}{L_i^1}} \parens{\frac{1}{p} \sum\limits_{i=1}^p \frac{1}{L_i^1}}}{\frac{1}{p} \parens{\frac{1}{p} \sum\limits_{j=1}^p \frac{1}{L_j^1}}^2}
        \leq \frac{\parens{\frac{1}{p} \sum\limits_{i=1}^p L^0_i} \parens{\frac{1}{p} \sum\limits_{i=1}^p \frac{1}{L_i^1}}}{\frac{1}{p} \parens{\frac{1}{p} \sum\limits_{j=1}^p \frac{1}{L_j^1}}}
        = \sum\limits_{i=1}^p L^0_i.
    \end{align*}
    Conversely, if both sequences $\{L^0_i\}$ and $\{L^1_i\}$ are sorted in the same order (either increasing or decreasing), the inequality reverses, and the first term of \eqref{eq:det_rate1} may be tighter.
    That said, empirical evidence we provide in \Cref{sec:experiments} indicates that in practice $L^0_i \approx 0$ across all layers, in which case the first terms in \eqref{eq:det_rate1} and \eqref{eq:det_rate2} effectively vanish. Then, \eqref{eq:det_rate2} is clearly superior, replacing the worst-case constant $L^1_{\max}$ by the harmonic mean.
\end{remark}

\begin{proof}
We start with the result obtained in Lemma~\ref{lemma:generalized-smoothness} with $X=X^k$ and $Y=X^{k+1}$
\begin{align*}
f(X^{k+1}) &\leq f(X^k) + \left\langle \nabla f(X^k), X^{k+1}-X^k\right\rangle + \sum_{i=1}^p\frac{L^0_i + L^1_i\|\nabla _i f(X^k)\|_{(i) \star}}{2}\|X^k_i-X^{k+1}_i\|^2_{(i)}\\
&= f(X^k) + \sum_{i=1}^p \left[ \left\langle \nabla _i f(X^k), X_i^{k+1}-X_i^k\right\rangle_{(i)} + \frac{L^0_i + L^1_i\|\nabla _i f(X^k)\|_{(i) \star}}{2}\|X^k_i-X^{k+1}_i\|^2_{(i)}\right].
\end{align*}

The update rule (\ref{eq:update_rule}) and the definition of the dual norm $\|\cdot\|_{(i) \star}$ give
\begin{align*}
\|X^k_i-X^{k+1}_i\|^2_{(i)}\leq \left(t^k_i\right)^2
\end{align*}
and
\begin{align*}
    \left\langle \nabla _i f(X^k), X_i^{k+1}-X_i^k\right\rangle_{(i)}
    &= \left\langle \nabla _i f(X^k), \mathrm{LMO}_{\mathcal{B}_i^k} \left( \nabla _i f(X^k) \right)-X_i^k\right\rangle_{(i)} \\
    &= - t_i^k \max_{\|X_i\|_{(i)} \leq 1} \left\langle \nabla _i f(X^k), X_i\right\rangle_{(i)} \\
    &= - t_i^k \| \nabla _i f(X^k) \|_{(i) \star}.
\end{align*}
Consequently, 
\begin{align*}
f(X^{k+1}) \leq f(X^k) + \sum_{i=1}^p \left[ - t_i^k \| \nabla _i f(X^k) \|_{(i) \star} + \frac{L^0_i + L^1_i\|\nabla _i f(X^k)\|_{(i) \star}}{2} \left(t^k_i\right)^2 \right].
\end{align*}
Now, choosing
\begin{align*}
t^k_i = \frac{\|\nabla _i f(X^k)\|_{(i) \star}}{L^0_i + L^1_i\|\nabla _i f(X^k)\|_{(i) \star}},
\end{align*}
which minimizes the right-hand side of the last inequality, yields the descent inequality
\begin{align}\label{eq:aousghr}
f(X^{k+1}) \leq f(X^k) - \sum_{i=1}^p \frac{\|\nabla _i f(X^k)\|^2_{(i) \star}}{2\left(L^0_i + L^1_i\|\nabla _i f(X^k)\|_{(i) \star}\right)}.
\end{align}

Summing the terms, we obtain
\begin{equation}
\begin{aligned}
\sum_{k=0}^{K-1} \sum_{i=1}^p \frac{\|\nabla _i f(X^k)\|^2_{(i) \star}}{2\left(L^0_i + L^1_i\|\nabla _i f(X^k)\|_{(i) \star}\right)} &\leq \sum_{k=0}^{K-1} \left( f(X^k) - f(X^{k+1}) \right) \\
&= f(X^0) - f(X^{K}) \\
&\leq f(X^0) - \inf_{X \in \cS} f(X) =: \Delta^0.
\end{aligned}
\label{eq:telescoping}
\end{equation}

Now, the analysis can proceed in two ways:

\begin{enumerate}
    \item Upper-bounding $L^1_i$ by $L^1_{\max}:=\max_{i=1,\dots,p} L^1_i$ in \eqref{eq:telescoping}, we obtain
    \begin{align}\label{eq:anmzdgzn}
    \sum_{k=0}^{K-1} \sum_{i=1}^p \frac{\|\nabla _i f(X^k)\|^2_{(i) \star}}{2\left(L^0_i + L^1_{\max}\|\nabla _i f(X^k)\|_{(i) \star}\right)} 
    \leq \Delta^0.
    \end{align}
    Now, applying \Cref{lemma:ineq} with $x_i = 1$, $y_i = \|\nabla _i f(X^{k})\|_{(i) \star}$ and $z_i = 2 \parens{L^0_i + L^1_{\max} \norm{\nabla _i f(X^k)}_{(i) \star}}$ gives
    \begin{align*}
    \phi \left( \sum_{i=1}^p \|\nabla _i f(X^{k})\|_{(i) \star}\right) 
    &= \frac{\left( \sum_{i=1}^p \|\nabla _i f(X^{k})\|_{(i) \star} \right)^2}{2 \left(\sum_{i=1}^p L^0_i + L^1_{\max} \sum_{i=1}^p \|\nabla _i f(X^{k})\|_{(i) \star}\right)} \\
    &\leq \sum_{i=1}^p \frac{\|\nabla _i f(X^{k})\|^2_{(i) \star}}{2 \left(L^0_i + L^1_{\max} \|\nabla _i f(X^{k})\|_{(i) \star}\right)},
    \end{align*}
    where $\phi(t)\eqdef \frac{t^2}{2(\sum_{i=1}^p L^0_i+L^1_{\max} t)}$.
    Combining the last inequality with (\ref{eq:anmzdgzn}) and using the fact that $\phi$ is increasing, we obtain
    \begin{align}
    K \phi \left(\min_{k=0,\ldots,K-1} \sum_{i=1}^p \|\nabla _i f(X^{k})\|_{(i) \star}\right) \leq \sum_{k=0}^{K-1} \phi\left(\sum_{i=1}^p \|\nabla _i f(X^k)\|_{(i) \star}\right) \leq \Delta^0,
    \label{eq:for_the_proof_th2_part1}
    \end{align}
    and hence
    \begin{align*}
    \min_{k=0,\ldots,K-1} \sum_{i=1}^p \|\nabla _i f(X^{k})\|_{(i) \star} \leq \phi^{-1}\left( \frac{\Delta^0}{K} \right),
    \end{align*}
    where $\phi^{-1}$ is the inverse function (which exists since $\phi$ is increasing).    
    Therefore, to reach the precision $\min_{k=0,\ldots,K-1} \sum_{i=1}^p \left\| \nabla _i f(X^{k}) \right\|_{(i) \star} \leq \epsilon$, it
    is sufficient to choose the number of iterations to be
    \begin{align*}
    K = \left\lceil \frac{\Delta^0}{\phi(\epsilon)} \right\rceil = \left\lceil \frac{2 \sum_{i=1}^p L^0_i \Delta^0}{\epsilon^2} + \frac{2L^1_{\max} \Delta^0}{\epsilon} \right\rceil.
    \end{align*}

    \item Alternatively, we can start from the inequality \eqref{eq:telescoping} and apply \Cref{lemma:ineq} with $x_i = \nicefrac{1}{L_i^1}$, $y_i = \norm{\nabla _i f(X^k)}_{(i) \star}$ and $z_i = 2 (L^0_i + L^1_i \norm{\nabla _i f(X^k)}_{(i) \star})$
    to obtain
    \begin{eqnarray*}
        \Delta^0 &\geq& \sum_{k=0}^{K-1} \sum_{i=1}^p \frac{\|\nabla _i f(X^k)\|^2_{(i) \star}}{2\left(L^0_i + L^1_i\|\nabla _i f(X^k)\|_{(i) \star}\right)} \\
        &\geq& \sum_{k=0}^{K-1} \frac{\parens{\sum_{i=1}^p \frac{1}{L_i^1} \norm{\nabla _i f(X^k)}_{(i) \star}}^2}{2 \parens{\sum_{i=1}^p \frac{1}{(L_i^1)^2} \parens{L^0_i + L^1_i \norm{\nabla _i f(X^k)}_{(i) \star}}}} \\
        &=& \sum_{k=0}^{K-1} \frac{\parens{\sum_{i=1}^p \frac{1}{L_i^1} \norm{\nabla _i f(X^k)}_{(i) \star}}^2}{2 \parens{\sum_{i=1}^p \frac{L^0_i}{(L_i^1)^2} + \sum_{i=1}^p \frac{1}{L_i^1} \norm{\nabla _i f(X^k)}_{(i) \star}}} \\
        &=& \sum_{t=0}^{K-1} \psi{\parens{\sum_{i=1}^p \frac{1}{L_i^1} \norm{\nabla _i f(X^k)}_{(i) \star}}},
    \end{eqnarray*}
    where $\psi(t) \eqdef \frac{t^2}{2 \parens{\sum_{i=1}^p \frac{L^0_i}{(L_i^1)^2} + t}}$. Since the function $\psi$ is increasing for $t>0$, $\psi^{-1}$ exists. It follows that
    \begin{eqnarray*}
        \Delta^0 &\geq& \sum_{k=0}^{K-1} \psi{\parens{\sum_{i=1}^p \frac{1}{L_i^1} \norm{\nabla _i f(X^k)}_{(i) \star}}} \\
        &\geq& K \psi{\parens{\min_{k=0,\ldots,K-1} \sum_{i=1}^p \frac{1}{L_i^1} \norm{\nabla _i f(X^k)}_{(i) \star}}},
    \end{eqnarray*}
    and hence
    \begin{eqnarray*}
        \min_{k=0,\ldots,K-1} \sum_{i=1}^p \frac{1}{L_i^1} \norm{\nabla _i f(X^k)}_{(i) \star}
        &\leq& \psi^{-1}\parens{\frac{\Delta^0}{K}}.
    \end{eqnarray*}
    This in turn means that to reach the precision
    \begin{eqnarray*}
        \min_{k=0,\ldots,K-1} \sum_{i=1}^p \left[\frac{\frac{1}{L_i^1}}{\frac{1}{p} \sum_{j=1}^p \frac{1}{L_j^1}} \norm{\nabla _i f(X^k)}_{(i) \star}\right] \leq \varepsilon,
    \end{eqnarray*}
    it suffices to run the algorithm for
    \begin{align*}
        K = \left\lceil \frac{\Delta^0}{\psi\parens{\varepsilon \parens{\frac{1}{p} \sum_{j=1}^p \frac{1}{L_j^1}}}} \right\rceil
        = \left\lceil \frac{2 \Delta^0 \parens{\sum_{i=1}^p \frac{L^0_i}{(L_i^1)^2}}}{\varepsilon^2 \parens{\frac{1}{p} \sum_{j=1}^p \frac{1}{L_j^1}}^2} 
        + \frac{2 \Delta^0}{\varepsilon \parens{\frac{1}{p} \sum_{j=1}^p \frac{1}{L_j^1}}} \right\rceil
    \end{align*}
    iterations.
\end{enumerate}

\end{proof}

\subsection{Convergence under the PŁ condition}\label{sec:pl}

We now establish convergence rates under the layer-wise Polyak–Łojasiewicz (PŁ) condition, introduced in Assumption~\ref{ass:pl}. This property is especially relevant for heavily over-parameterized neural networks, as it has been shown to capture the properties of their loss landscapes \citep{overparameterized-loss-landscapes}.

\begin{assumption}[Layer-wise Polyak-Łojasiewicz condition]\label{ass:pl}
    The function $f: \cS \mapsto \mathbb{R}$ satisfies the layer-wise Polyak-Łojasiewicz (PŁ) condition with a constant $\mu>0$, i.e., for any $X\in \cS$
    \begin{equation*}
       \sum_{i=1}^p \|\nabla _i f(X)\|^2_{(i) \star} \geq 2\mu\left(f(X)-f^{\star}\right),
    \end{equation*}
    where $f^{\star}:=\inf_{X \in \cS} f(X) > - \infty$.
\end{assumption}

Assumption~\ref{ass:pl} reduces to the standard PŁ condition \citep{karimi2020linearconvergencegradientproximalgradient} by vectorizing the parameters and adopting the Euclidean norm \(\|\cdot\|_{2}\).

\begin{theorem}\label{theorem:1-pl}
    Let Assumptions \ref{ass:generalized-smoothness} and \ref{ass:pl} hold and fix $\varepsilon>0$. Let $X^0,\ldots,X^{K-1}$ be the iterates of deterministic \algname{Gluon} (Algorithm \ref{algo:deterministic}) run with \( t^k_i = \frac{\|\nabla _i f(X^k)\|_{(i) \star}}{L^0_i + L^1_i\|\nabla _i f(X^k)\|_{(i) \star}}\).

    \begin{enumerate}
        \item If $L^1_i\geq0$, then to reach the precision $\min_{k=0,\ldots,K-1} f(X^{k})-f^{\star} \leq \epsilon$, it suffices to run the algorithm for
        \begin{align*}
        K = \left\lceil \frac{\sum^p_{i=1} L^0_i \Delta^0}{\mu\epsilon} + \frac{\sqrt{2} L^1_{\max} \Delta^0}{\sqrt{\mu\epsilon}} \right\rceil
        \end{align*}
        iterations,
    
        \item If $L^1_i=0$ for all $i = 1, \dots, p$, then to reach the precision $f(X^{K})-f^{\star} \leq \epsilon$, it suffices to run the algorithm for
        \begin{align*}
        K = \left\lceil \frac{L^0_{\max}}{\mu} \log{\frac{\Delta^0}{\epsilon}} \right\rceil,
        \end{align*}
    \end{enumerate}
    where $L^0_{\max}:=\max_{i=1,\dots,p}L^0_i$, $L^1_{\max}:=\max_{i=1,\dots,p}L^1_i$, $\Delta^0 := f(X^0) - f^{\star}$ and $f^{\star} := \inf_{X \in \cS} f(X)$.
\end{theorem}

\begin{proof}
We consider two scenarios: (1) the general case with arbitrary $L^1_i\geq0$ and (2) $L^1_i=0$ for all $i = 1, \dots, p$.

\paragraph{Case 1: $L^1_i\geq0$.}

We start by following the same steps as in the proof of Theorem~\ref{theorem:1}. From~\eqref{eq:for_the_proof_th2_part1}, we have

\begin{align*}
\sum_{k=0}^{K-1} \phi\left(\sum_{i=1}^p \|\nabla _i f(X^k)\|_{(i) \star}\right) \leq \Delta^0,
\end{align*}
where $\phi(t):= \frac{t^2}{2(\sum_{i=1}^p L^0_i+L^1_{\max} t)}$.
Now, using Assumption~\ref{ass:pl}, we get
\begin{align*}
\left(\sum_{i=1}^p \|\nabla_i f(X^{k})\|_{(i)\star}\right)^2 
\geq \sum_{i=1}^p \|\nabla_i f(X^{k})\|^2_{(i)\star} 
\geq 2\mu \left(f(X^{k})-f^{\star}\right).
\end{align*}

Consequently, since $\phi$ is an increasing function,
\begin{align*}
K \phi \left(\sqrt{2\mu} \sqrt{f(X^{k^{\star}})-f^{\star}} \right) &\leq \sum_{k=0}^{K-1} \phi \left( \sqrt{2\mu} \sqrt{f(X^{k})-f^{\star}} \right) \\ &\leq \sum_{k=0}^{K-1} \phi\left(\sum_{i=1}^p \|\nabla _i f(X^k)\|_{(i) \star}\right) \leq \Delta^0,
\end{align*}
where $k^{\star}:=\operatorname{argmin}_{k=0,\ldots,K-1} f(X^{k})-f^{\star}$.
Denoting the corresponding inverse function (which exists since $\phi$ is increasing) by $\phi^{-1}$, it follows that
\begin{align*}
\sqrt{2\mu} \sqrt{f(X^{k^{\star}})-f^{\star}} \leq \phi^{-1}\left( \frac{\Delta^0}{K} \right) \leq \sqrt{2\mu \epsilon}.
\end{align*}

Therefore, to reach the precision $f(X^{k^{\star}})-f^{\star} \leq \epsilon$, it
is sufficient to choose the number of iterations 
\begin{align*}
K = \left\lceil \frac{\Delta^0}{\phi\left( \sqrt{2\mu\epsilon}\right)} \right\rceil = \left\lceil \frac{\sum^p_{i=1} L^0_i \Delta^0}{\mu\epsilon} + \frac{\sqrt{2} L^1_{\max} \Delta^0}{\sqrt{\mu\epsilon}} \right\rceil.
\end{align*}

\paragraph{Case 2: $L^1_i = 0$.}

Inequality \eqref{eq:aousghr} from the proof of Theorem~\ref{theorem:1} with $L^1_i=0$ gives
\begin{align*}
f(X^{k+1}) \leq f(X^k) - \sum_{i=1}^p \frac{\|\nabla _i f(X^k)\|^2_{(i) \star}}{2L^0_i}.
\end{align*}

Using the fact that
\begin{align*}
\sum_{i=1}^p \frac{\|\nabla_i f(X^k)\|^2_{(i)\star}}{2L^0_i}  \geq \min_{j=1,\dots,p} \frac{1}{2L^0_j} \sum_{i=1}^p \|\nabla_i f(X^k)\|^2_{(i)\star} = \frac{1}{2 \max_{j=1,\dots,p} L^0_j} \sum_{i=1}^p \|\nabla f(X^k)\|^2_{(i)\star}
\end{align*} 
along with Assumption~\ref{ass:pl}, we obtain
\begin{align*}
f(X^{k+1}) \leq f(X^k) - \frac{\mu}{L^0_{\max}} \left(f(X^k)-f^{\star}\right).
\end{align*}
The remaining part of the proof follows from the simple observation
\begin{align*}
\log \left( \frac{\Delta_0}{\epsilon} \right) \leq k \frac{\mu}{L^0_{\max}} \leq k \log \left( \frac{1}{1 - \frac{\mu}{L^0_{\max}}} \right).
\end{align*}

\end{proof}

\newpage

\section{Stochastic case}

\subsection{Adaptive stepsizes}

Before proving the main result from \Cref{sec:convergence_stochastic}, we first present an attempt to formulate an adaptive stepsize strategy for the stochastic setting. This requires the following assumption:

\begin{assumption}\label{ass:grad_stochastic}
    The stochastic gradient estimator $\nabla f_{\xi}: \cS \mapsto \cS$ is unbiased and has bounded relative variance. That is, $\mathbb{E}[\nabla f_{\xi}(X)]=\nabla f(X)$ for all $X \in \cS$ and there exists $0\leq\zeta<1$ such that
    \begin{align*}
        \|\nabla_i f_{\xi}(X)-\nabla_i f(X)\|_{(i) \star} \leq \zeta \|\nabla_i f_{\xi}(X)\|_{(i) \star}, \quad \ i = 1, \dots, p
    \end{align*}
    holds almost surely for all $X \in \cS$.
\end{assumption}

This assumption is somewhat unconventional due to the presence of the stochastic gradients on the right-hand side of the inequality. It does not follow from standard conditions and does not fall within known frameworks for modeling stochasticity, such as the ABC inequality of \citet{khaled2020better}. Instead, it introduces a novel structure with parallels to the literature on contractive compression \citep{beznosikov2023biased, demidovich2023guide}.

To elaborate, recall the definition of a contractive compressor:
\begin{definition}[Contractive compressor]\label{def:contractive_compr}
	A stochastic mapping $\cC:\cS \to \cS$ is called a \emph{contractive compressor} if there exists  $\alpha \in [0,1)$ such that
	\begin{align}\label{eq:contractive_compr}
		\Exp{\|\cC(X) - X\|^2} \leq (1-\alpha)\|X\|^2 
	\end{align}
    for any $X\in\cS$.
\end{definition}
There is a conceptual similarity between \Cref{ass:grad_stochastic} and the contractive property in~\eqref{eq:contractive_compr}. \Cref{ass:grad_stochastic} can be interpreted as asserting that the true gradient $\nabla f(X)$ is effectively a contraction of the stochastic gradient $\nabla f_{\xi}(X)$, with contraction factor $1 - \zeta$. Unlike contractive compressors, there is no explicit mapping from $\nabla f_{\xi}(X)$ to $\nabla f(X)$, and the uniform bound implies the same contraction-like behavior across all stochastic gradients.

Although \Cref{ass:grad_stochastic} is admittedly strong, it allows us to establish a convergence theorem using an adaptive stepsize strategy similar to the one employed in the deterministic case in \Cref{theorem:1_full}.

\begin{theorem}\label{theorem:2}
    Let Assumptions \ref{ass:generalized-smoothness} and \ref{ass:grad_stochastic} hold and fix $\varepsilon>0$. Let $X^0,\ldots,X^{K-1}$ be the iterates of \algname{Gluon} (Algorithm \ref{algo:stochastic}) run with $\beta^k=0$ and \( t^k_i = \frac{(1-\zeta)\|\nabla _i f_{\xi^k}(X^k)\|_{(i) \star}}{L^0_i + (1+\zeta)L^1_i\|\nabla _i f_{\xi^k}(X^k)\|_{(i) \star}}\).
    Then,
    \begin{enumerate}
        \item In order to reach the precision
        \begin{align*}
            \min_{k=0,\ldots,K-1} \sum_{i=1}^p \mathbb{E} \left[\left\| \nabla _i f(X^{k}) \right\|_{(i) \star}\right] \leq \epsilon, 
        \end{align*}
        it suffices to run the algorithm for
        \begin{align*}
            K = \left\lceil \frac{2 \sum_{i=1}^p L^0_i \Delta^0}{\left(1 - \zeta \right)^2\epsilon^2} + \frac{2(1+\zeta)L^1_{\max} \Delta^0}{\left(1 - \zeta \right)^2\epsilon} \right\rceil
        \end{align*}
        iterations.
        
        \item In order to reach the precision
        \begin{align*}
            \min_{k=0,\ldots,K-1} \sum_{i=1}^p \left[\frac{\frac{1}{L_i^1}}{\frac{1}{p} \sum_{j=1}^p \frac{1}{L_j^1}} \norm{\nabla _i f(X^k)}_{(i) \star}\right] \leq \varepsilon,
        \end{align*}
        it suffices to run the algorithm for
        \begin{align*}
            K = \left\lceil \frac{2 \Delta^0 \sum_{i=1}^p \frac{L^0_i}{(L_i^1)^2}}{\varepsilon^2 (1-\zeta)^2 \parens{\frac{1}{p} \sum_{j=1}^p \frac{1}{L_j^1}}^2} 
            + \frac{2 \Delta^0 (1+\zeta)}{\varepsilon (1-\zeta)^2 \parens{\frac{1}{p} \sum_{j=1}^p \frac{1}{L_j^1}}} \right\rceil
        \end{align*}
        iterations,
    \end{enumerate}
    where $\Delta^0\eqdef f(X^0) - \inf_{X \in \cS} f(X)$ and $L^1_{\max}\eqdef\max_{i=1,\dots,p} L^1_i$.
\end{theorem}

\begin{proof}
Lemma~\ref{lemma:generalized-smoothness} with $X=X^k$ and $Y=X^{k+1}$ gives
\begin{eqnarray*}
    &&\hspace{-6mm}f(X^{k+1}) \\
    &\leq& f(X^k) + \left\langle \nabla f(X^k), X^{k+1}-X^k\right\rangle + \sum_{i=1}^p\frac{L^0_i + L^1_i\|\nabla _i f(X^k)\|_{(i) \star}}{2}\|X^k_i-X^{k+1}_i\|^2_{(i)} \\
    &=& f(X^k) + \sum_{i=1}^p \left[ \left\langle \nabla _i f(X^k), X_i^{k+1}-X_i^k\right\rangle_{(i)} + \frac{L^0_i + L^1_i\|\nabla _i f(X^k)\|_{(i) \star}}{2}\|X^k_i-X^{k+1}_i\|^2_{(i)}\right]\\
    &=& f(X^k) + \sum_{i=1}^p \Big[ \left\langle \nabla_i f_{\xi^k}(X^k), X_i^{k+1}-X_i^k\right\rangle_{(i)} + \left\langle \nabla _i f(X^k) - \nabla_i f_{\xi^k}(X^k), X_i^{k+1}-X_i^k\right\rangle_{(i)}\Big]\\
    && + \sum_{i=1}^p \frac{L^0_i + L^1_i\|\nabla _i f(X^k)\|_{(i) \star}}{2}\|X^k_i-X^{k+1}_i\|^2_{(i)},
\end{eqnarray*}
and applying the Cauchy-Schwarz inequality, we get
\begin{align*}
    f(X^{k+1}) \leq f(X^k) + \sum_{i=1}^p \Bigg[ &\left\langle \nabla_i f_{\xi^k}(X^k), X_i^{k+1}-X_i^k\right\rangle_{(i)} \\
    &+ \| \nabla _i f(X^k) - \nabla_i f_{\xi^k}(X^k) \|_{(i) \star} \| X_i^{k+1}-X_i^k \|_{(i)}\\
    &+ \frac{L^0_i + L^1_i\|\nabla _i f(X^k)\|_{(i) \star}}{2}\|X^k_i-X^{k+1}_i\|^2_{(i)}\Bigg].
\end{align*}

The update rule (\ref{eq:update_rule_stoch}) and the definition of the dual norm $\|\cdot\|_{(i) \star}$ give
\begin{align*}
\|X^k_i-X^{k+1}_i\|^2_{(i)}\leq \left(t^k_i\right)^2
\end{align*}
and
\begin{align*}
\left\langle \nabla_i f_{\xi^k}(X^k), X_i^{k+1}-X_i^k\right\rangle_{(i)} &= \left\langle \nabla_i f_{\xi^k}(X^k), \mathrm{LMO}_{\mathcal{B}_i^k} \left( \nabla_i f_{\xi^k}(X^k) \right)-X_i^k\right\rangle_{(i)} \\
&= - t_i^k \max_{\|X_i\|_{(i)} \leq 1} \left\langle \nabla_i f_{\xi^k}(X^k), X_i\right\rangle_{(i)} \\
&= - t_i^k \| \nabla_i f_{\xi^k}(X^k) \|_{(i) \star}.
\end{align*}

Consequently, using Assumption~\ref{ass:grad_stochastic}, we obtain
\begin{align*}
    f(X^{k+1}) &\leq f(X^k) + \sum_{i=1}^p \Bigg[ - t_i^k \| \nabla_i f_{\xi^k}(X^k) \|_{(i) \star} + t^k_i \| \nabla _i f(X^k) - \nabla_i f_{\xi^k}(X^k) \|_{(i) \star} \\
    &\qquad\qquad\qquad\qquad+ \frac{L^0_i + L^1_i\|\nabla _i f(X^k)\|_{(i) \star}}{2} \left(t^k_i\right)^2 \Bigg]\\
    &\leq f(X^k) + \sum_{i=1}^p \Bigg[ - (1-\zeta) t_i^k \| \nabla _i f_{\xi^k}(X^k) \|_{(i) \star} \\
    &\qquad\qquad\qquad\qquad+ \frac{L^0_i + (1+\zeta) L^1_i\|\nabla _i f_{\xi^k}(X^k)\|_{(i) \star}}{2} \left(t^k_i\right)^2 \Bigg].
\end{align*} 

Minimizing the right-hand side of the last inequality with respect to \( t^k_i \) yields
\begin{align*}
t^k_i = \frac{(1-\zeta)\|\nabla _i f_{\xi^k}(X^k)\|_{(i) \star}}{L^0_i + (1+\zeta)L^1_i\|\nabla _i f_{\xi^k}(X^k)\|_{(i) \star}}.
\end{align*}

This greedy approach for choosing \( t^k_i \) gives the descent inequality
\begin{align*}
f(X^{k+1}) \leq f(X^k) - \sum_{i=1}^p \frac{(1-\zeta)^2\|\nabla _i f_{\xi^k}(X^k)\|^2_{(i) \star}}{2\left(L^0_i + (1+\zeta)L^1_i\|\nabla _i f_{\xi^k}(X^k)\|_{(i) \star}\right)}.
\end{align*}

Taking expectations, we have
\begin{align}\label{eq:aoisbhr5}
\mathbb{E}[f(X^{k+1})] \leq \mathbb{E}[f(X^k)] - \sum_{i=1}^p \mathbb{E} \left[ \frac{(1-\zeta)^2\|\nabla _i f_{\xi^k}(X^k)\|^2_{(i) \star}}{2\left(L^0_i + (1+\zeta)L^1_i\|\nabla _i f_{\xi^k}(X^k)\|_{(i) \star}\right)}\right].
\end{align}

Now, let us define the function $\phi_i(t):= \frac{(1 - \zeta)^2 t^2}{2(L^0_i+(1+\zeta)L^1_i t)}$. Since $\phi_i(t)$ is convex, Jensen's inequality gives
\begin{align*}
    \mathbb{E}[f(X^k)] - \mathbb{E}[f(X^{k+1})]
    &\geq \sum_{i=1}^p \mathbb{E} \left[ \frac{(1-\zeta)^2\|\nabla _i f_{\xi^k}(X^k)\|^2_{(i) \star}}{2\left(L^0_i + (1+\zeta)L^1_i\|\nabla _i f_{\xi^k}(X^k)\|_{(i) \star}\right)}\right] \\
    &\geq \sum_{i=1}^p \frac{(1-\zeta)^2\left(\mathbb{E} \left[ \|\nabla _i f_{\xi^k}(X^k)\|_{(i) \star}\right]\right)^2}{2\left(L^0_i + (1+\zeta)L^1_i\mathbb{E} \left[ \|\nabla _i f_{\xi^k}(X^k)\|_{(i) \star}\right]\right)}.
\end{align*}
By Jensen's inequality and \Cref{ass:grad_stochastic}
\begin{align*}
    \mathbb{E} \left[ \left\| \nabla_i f(X^k) \right\|_{(i)\star} \right]
    &= \Exp{\norm{\ExpCond{\nabla_i f_{\xi_k}(X^k)}{X^k}}_{(i)\star}} \\
    &\leq \Exp{\ExpCond{\norm{\nabla_i f_{\xi_k}(X^k)}_{(i)\star}}{X^k}} \\
    &= \mathbb{E} \left[ \left\| \nabla_i f_{\xi_k}(X^k) \right\|_{(i)\star} \right],
\end{align*} 
and hence, using the fact that $\phi_i$ is increasing, we get
\begin{align*}
    \mathbb{E}[f(X^k)] - \mathbb{E}[f(X^{k+1})]
    &\geq \sum_{i=1}^p \frac{(1-\zeta)^2\left(\Exp{\norm{\nabla _i f(X^k)}_{(i) \star}}\right)^2}{2\left(L^0_i + (1+\zeta)L^1_i \Exp{\norm{\nabla _i f(X^k)}_{(i) \star}}\right)}.
\end{align*}
Summing the terms gives
\begin{equation}
    \begin{aligned}
    \sum_{k=0}^{K-1} \sum_{i=1}^p \frac{(1-\zeta)^2\left(\Exp{\norm{\nabla _i f(X^k)}_{(i) \star}}\right)^2}{2\left(L^0_i + (1+\zeta)L^1_i \Exp{\norm{\nabla _i f(X^k)}_{(i) \star}}\right)}
    &\leq \sum_{k=0}^{K-1} \parens{\mathbb{E}[f(X^k)] - \mathbb{E}[f(X^{k+1})]} \\
    &= \mathbb{E}[f(X^0)] - \mathbb{E}[f(X^K)] \\
    &\leq f(X^0) - \inf_{X \in \cS} f(X) =: \Delta^0,
    \end{aligned}
    \label{eq:artjnsd32}
\end{equation}
The remaining part of the proof closely follows the proof of Theorem~\ref{theorem:1_full}.
We can proceed in two ways:
\begin{enumerate}
    \item Upper-bounding $L^1_i$ by $L^1_{\max}:=\max_{i=1,\dots,p} L^1_i$ in \eqref{eq:artjnsd32}, we obtain
    \begin{align}\label{eq:sarja}
        \sum_{k=0}^{K-1} \sum_{i=1}^p \frac{(1-\zeta)^2\left(\Exp{\norm{\nabla _i f(X^k)}_{(i) \star}}\right)^2}{2\left(L^0_i + (1+\zeta) L^1_{\max} \Exp{\norm{\nabla _i f(X^k)}_{(i) \star}}\right)}
        \leq \Delta^0.
    \end{align}
    Now, \Cref{lemma:ineq} with $x_i = 1$, $y_i = (1-\zeta) \Exp{\|\nabla _i f(X^{k})\|_{(i) \star}}$ and $z_i = 2 \parens{L^0_i + (1+\zeta) L^1_{\max} \Exp{\norm{\nabla _i f(X^k)}_{(i) \star}}}$ gives
    \begin{align*}
        \phi\parens{\sum_{i=1}^p \Exp{\|\nabla _i f(X^{k})\|_{(i) \star}}}
        &= \frac{\parens{(1-\zeta) \sum_{i=1}^p \Exp{\|\nabla _i f(X^{k})\|_{(i) \star}}}^2}{2 \sum_{i=1}^p \parens{L^0_i + (1+\zeta) L^1_{\max} \Exp{\norm{\nabla _i f(X^k)}_{(i) \star}}}} \\
        &\leq \sum_{i=1}^p \frac{(1-\zeta)^2 \Exp{\|\nabla _i f(X^{k})\|_{(i) \star}}^2}{2 \parens{L^0_i + (1+\zeta) L^1_{\max} \Exp{\norm{\nabla _i f(X^k)}_{(i) \star}}}}
    \end{align*}
    where $\phi(t)\eqdef \frac{(1-\zeta)^2 t^2}{2(\sum_{i=1}^p L^0_i + (1+\zeta) L^1_{\max} t)}$.
    Combining the last inequality with \eqref{eq:sarja} and using the fact that $\phi$ is increasing, we get
    \begin{align*}
        K \phi\parens{\min_{k=0,\ldots,K-1} \sum_{i=1}^p \Exp{\|\nabla _i f(X^{k})\|_{(i) \star}}}
        \leq \sum_{k=0}^{K-1} \phi\parens{\sum_{i=1}^p \Exp{\|\nabla _i f(X^{k})\|_{(i) \star}}}
        \leq \Delta^0.
    \end{align*}
    and hence
    \begin{align*}
        \min_{k=0,\ldots,K-1} \sum_{i=1}^p \Exp{\|\nabla _i f(X^{k})\|_{(i) \star}} \leq \phi^{-1}\left( \frac{\Delta^0}{K} \right),
    \end{align*}
    where $\phi^{-1}$ denotes the inverse function (which exists since $\phi$ is increasing).    
    Therefore, to reach the precision $\min_{k=0,\ldots,K-1} \sum_{i=1}^p \Exp{\|\nabla _i f(X^{k})\|_{(i) \star}} \leq \epsilon$, it suffices to run the algorithm for
    \begin{align*}
        K = \left\lceil \frac{\Delta^0}{\phi(\epsilon)} \right\rceil
        = \left\lceil \frac{2 \Delta^0 \sum_{i=1}^p L^0_i}{(1-\zeta)^2 \epsilon^2} + \frac{2 \Delta^0 (1+\zeta) L^1_{\max}}{(1-\zeta)^2 \epsilon}\right\rceil
    \end{align*}
    iterations.

    \item Alternatively, we can start from inequality \eqref{eq:artjnsd32} and apply \Cref{lemma:ineq} with $x_i = \nicefrac{1}{L_i^1}$, $y_i = (1-\zeta) \Exp{\norm{\nabla _i f(X^k)}_{(i) \star}}$ and $z_i = 2 \parens{L^0_i + (1+\zeta) L^1_i \Exp{\norm{\nabla _i f(X^k)}_{(i) \star}}}$ to obtain
    \begin{eqnarray*}
        \Delta^0 &\geq& \sum_{k=0}^{K-1} \sum_{i=1}^p \frac{(1-\zeta)^2 \Exp{\norm{\nabla _i f(X^k)}_{(i) \star}}^2}{2 \parens{L^0_i + (1+\zeta) L^1_i \Exp{\norm{\nabla _i f(X^k)}_{(i) \star}}}} \\
        &\geq& \sum_{k=0}^{K-1} \frac{\parens{\sum_{i=1}^p \frac{1}{L_i^1} (1-\zeta) \Exp{\norm{\nabla _i f(X^k)}_{(i) \star}}}^2}{2 \sum_{i=1}^p \parens{\frac{L^0_i}{(L_i^1)^2} + (1+\zeta) \frac{1}{L_i^1} \Exp{\norm{\nabla _i f(X^k)}_{(i) \star}}}} \\
        &=& \sum_{t=0}^{K-1} \psi{\parens{\sum_{i=1}^p \frac{1}{L_i^1} \Exp{\norm{\nabla _i f(X^k)}_{(i) \star}}}},
    \end{eqnarray*}
    where $\psi(t) \eqdef \frac{(1-\zeta)^2 t^2}{2 \parens{\sum_{i=1}^p \frac{L^0_i}{(L_i^1)^2} + (1+\zeta) t}}$. Since the function $\psi$ is increasing for $t>0$, $\psi^{-1}$ exists. It follows that
    \begin{eqnarray*}
        \Delta^0 &\geq& \sum_{k=0}^{K-1} \psi{\parens{\sum_{i=1}^p \frac{1}{L_i^1} \Exp{\norm{\nabla _i f(X^k)}_{(i) \star}}}} \\
        &\geq& K \psi{\parens{\min_{k=0,\ldots,K-1} \sum_{i=1}^p \frac{1}{L_i^1} \Exp{\norm{\nabla _i f(X^k)}_{(i) \star}}}},
    \end{eqnarray*}
    and hence
    \begin{eqnarray*}
        \min_{k=0,\ldots,K-1} \sum_{i=1}^p \frac{1}{L_i^1} \Exp{\norm{\nabla _i f(X^k)}_{(i) \star}}
        \leq \psi^{-1}\parens{\frac{\Delta^0}{K}}.
    \end{eqnarray*}
    This in turn means that to reach the precision
    \begin{eqnarray*}
        \min_{k=0,\ldots,K-1} \sum_{i=1}^p \left[\frac{\frac{1}{L_i^1}}{\frac{1}{p} \sum_{j=1}^p \frac{1}{L_j^1}} \norm{\nabla _i f(X^k)}_{(i) \star}\right] \leq \varepsilon,
    \end{eqnarray*}
    it suffices to run the algorithm for
    \begin{align*}
        K &= \left\lceil \frac{\Delta^0}{\psi\parens{\varepsilon \parens{\frac{1}{p} \sum_{j=1}^p \frac{1}{L_j^1}}}} \right\rceil \\
        &= \left\lceil \frac{2 \Delta^0 \sum_{i=1}^p \frac{L^0_i}{(L_i^1)^2}}{(1-\zeta)^2 \varepsilon^2 \parens{\frac{1}{p} \sum_{j=1}^p \frac{1}{L_j^1}}^2} 
        + \frac{2 \Delta^0 (1+\zeta)}{(1-\zeta)^2 \varepsilon \parens{\frac{1}{p} \sum_{j=1}^p \frac{1}{L_j^1}}} \right\rceil
    \end{align*}
    iterations.
\end{enumerate}
\end{proof}

\subsection{Proof of Theorem \ref{theorem:3}}\label{sec:thm_stoch_proof}

We now establish the main result of \Cref{sec:convergence_stochastic}. The guarantees in \Cref{theorem:3} follow from the more general result below.

\begin{theorem}\label{theorem:3_appendix}
    Let Assumptions \ref{ass:generalized-smoothness} and \ref{ass:bounded_var} hold and fix $\varepsilon>0$. Let $X^0,\ldots,X^{K-1}$ be the iterates of \algname{Gluon} (Algorithm \ref{algo:stochastic}) run with \( \beta^k = 1 - (k+1)^{-1/2} \), \( t_i^k = t_i (k+1)^{-3/4} \) for some \( t_i>0\), and $M^0_i=\nabla_i f_{\xi^0}(X^0)$.
    \begin{enumerate}
        \item If $L^1_i=0$, then
        \begin{align*}
            &\min_{k=0,\ldots,K-1} \sum_{i=1}^p t_i \mathbb{E}\left[\| \nabla _i f(X^k) \|_{(i) \star}\right] \\
            &\hspace{1cm}\leq \frac{\Delta^0}{K^{1/4}} + \frac{1}{K^{1/4}} \sum_{i=1}^p \Bigg[ \sigma t_i \left(7+2\sqrt{2 e^2} \log (K)\right) + L^0_i t^2_i \left(\frac{87}{2} + 14\log (K)\right)\Bigg],
        \end{align*}
        \item If $L^1_i\neq0$, then for \( t_i = \frac{1}{12L^1_i} \), we have
        \begin{align*}
            &\min_{k=0,\ldots,K-1} \sum_{i=1}^p \frac{1}{12L^1_i} \mathbb{E}\left[\| \nabla _i f(X^k) \|_{(i) \star}\right] \\
            &\hspace{1cm}\leq \frac{2\Delta^0}{K^{1/4}} + \frac{1}{K^{1/4}} \sum_{i=1}^p \Bigg[ \frac{\sigma}{6L^1_i}\left(7+2\sqrt{2 e^2} \log (K)\right) + \frac{L^0_i}{144(L^1_i)^2}\left(87 +28\log (K)\right)\Bigg],
        \end{align*}
    \end{enumerate}
    where $\Delta^0:=f(X^0) - \inf_{X\in\cS} f(X)$.
\end{theorem}

\begin{proof}
We again start with the result in Lemma~\ref{lemma:generalized-smoothness} with $X=X^k$ and $Y=X^{k+1}$, obtaining
\begin{eqnarray*}
    f(X^{k+1}) &\leq& f(X^k) + \left\langle \nabla f(X^k), X^{k+1}-X^k\right\rangle + \sum_{i=1}^p \frac{L^0_i + L^1_i\|\nabla _i f(X^k)\|_{(i) \star}}{2} \|X^k_i-X^{k+1}_i\|^2_{(i)}\\
    &=& f(X^k) + \sum_{i=1}^p \left[ \left\langle \nabla _i f(X^k), X_i^{k+1}-X_i^k\right\rangle_{(i)} + \frac{L^0_i + L^1_i\|\nabla _i f(X^k)\|_{(i) \star}}{2} \|X^k_i-X^{k+1}_i\|^2_{(i)}\right]\\
    &=& f(X^k) + \sum_{i=1}^p \Bigg[ \left\langle M_i^k, X_i^{k+1}-X_i^k\right\rangle_{(i)} + \left\langle \nabla _i f(X^k) - M_i^k, X_i^{k+1}-X_i^k\right\rangle_{(i)}\Bigg]\\
    &&+ \sum_{i=1}^p \frac{L^0_i + L^1_i\|\nabla _i f(X^k)\|_{(i) \star}}{2}\|X^k_i-X^{k+1}_i\|^2_{(i)}.
\end{eqnarray*}

Applying the Cauchy-Schwarz inequality, we have
\begin{eqnarray*}
    f(X^{k+1}) &\leq& f(X^k) + \sum_{i=1}^p \Bigg[ \left\langle M_i^k, X_i^{k+1}-X_i^k\right\rangle_{(i)} + \| \nabla _i f(X^k) - M_i^k \|_{(i) \star} \| X_i^{k+1}-X_i^k \|_{(i)}\Bigg]\\
    &&+\sum_{i=1}^p \frac{L^0_i + L^1_i\|\nabla _i f(X^k)\|_{(i) \star}}{2}\|X^k_i-X^{k+1}_i\|^2_{(i)}.
\end{eqnarray*}
Now, the update rule (\ref{eq:update_rule_stoch}) and the definition of the dual norm $\|\cdot\|_{(i) \star}$ give
\begin{align*}
\|X^k_i-X^{k+1}_i\|^2_{(i)}\leq \left(t^k_i\right)^2
\end{align*}
and
\begin{align*}
\left\langle M_i^k, X_i^{k+1}-X_i^k\right\rangle = \left\langle M_i^k, \mathrm{LMO}_{\mathcal{B}_i^k} \left( M_i^k \right)-X_i^k\right\rangle = - t_i^k \max_{\|X_i\|_{(i)} \leq 1} \left\langle M_i^k, X_i\right\rangle = - t_i^k \| M_i^k \|_{(i) \star}.
\end{align*}

Consequently, 
\begin{eqnarray*}
&&\hspace{-6mm}f(X^{k+1}) \\
&\leq& f(X^k) + \sum_{i=1}^p \left[ - t_i^k \| M_i^k \|_{(i) \star} + t^k_i \| \nabla _i f(X^k) - M_i^k \|_{(i) \star}  + \frac{L^0_i + L^1_i\|\nabla _i f(X^k)\|_{(i) \star}}{2} \left(t^k_i\right)^2 \right]\\
&=& f(X^k) + \sum_{i=1}^p \Bigg[ - t_i^k \| M_i^k - \nabla _i f(X^k) + \nabla _i f(X^k) \|_{(i) \star} + t^k_i \|M_i^k - \nabla _i f(X^k)\|_{(i) \star} \Bigg]\\
&& + \sum_{i=1}^p \frac{L^0_i + L^1_i\|\nabla _i f(X^k)\|_{(i) \star}}{2} \left(t^k_i\right)^2 \\
&\leq& f(X^k) + \sum_{i=1}^p \left[ - t_i^k \| \nabla _i f(X^k) \|_{(i) \star} + 2t^k_i \|M_i^k - \nabla _i f(X^k)\|_{(i) \star}\right] \\
&& + \sum_{i=1}^p \frac{L^0_i + L^1_i\|\nabla _i f(X^k)\|_{(i) \star}}{2} \left(t^k_i\right)^2 .
\end{eqnarray*}

Taking expectations, we obtain
\begin{align*}
\mathbb{E}[f(X^{k+1})] \leq \mathbb{E}[f(X^k)] + \sum_{i=1}^p \Bigg[ - t_i^k \mathbb{E}[\| \nabla _i f(X^k) \|_{(i) \star}]& + 2t^k_i \mathbb{E}\left[\left\|M_i^k - \nabla_i f(X^k)\right\|_{(i) \star}\right]\\ &+ \frac{L^0_i + L^1_i\mathbb{E}[\|\nabla _i f(X^k)\|_{(i) \star}]}{2} \left(t^k_i\right)^2 \Bigg].
\end{align*}

Telescoping the last inequality gives
\begin{align}\label{eq:main_eq_momentum}
\sum_{i=1}^p \sum_{k=0}^{K-1} t_i^k \mathbb{E}[\| \nabla _i f(X^k) \|_{(i) \star}] \leq \Delta^0 + \sum_{i=1}^p \Bigg[ &2 \sum_{k=0}^{K-1} t^k_i \mathbb{E}\left[\left\|M_i^k - \nabla_i f(X^k)\right\|_{(i) \star}\right]\\ &+ \sum_{k=0}^{K-1} \frac{L^0_i}{2}\left(t^k_i\right)^2 
+ \sum_{k=0}^{K-1} \frac{L^1_i}{2} \mathbb{E}[\|\nabla _i f(X^k)\|_{(i) \star}]\left(t^k_i\right)^2 \Bigg],\notag
\end{align}
where $\Delta^0:=f(X^0) - \inf_{X \in \cS} f(X)$.

Now, inspired by the analysis in \citet{hubler2024parameter}, we introduce the following notation: $\mu_i^k:=M_i^k - \nabla_i f(X^k)$, $\gamma_i^k:=\nabla_i f_{\xi^k}(X^k) - \nabla_i f(X^k)$, $\alpha^k=1-\beta^k$, $\beta^{a: b}:=\prod_{k=a}^b \beta^k$ and $S_i^k:=\nabla_i f(X^{k-1}) - \nabla_i f(X^{k})$. Then, we can rewrite the algorithm's momentum update rule as
\begin{align*}
M_i^k &= \beta^k M^{k-1}_i+(1-\beta^k)\nabla_i f_{\xi^k}(X^k)\\
&=\beta^k \left(\mu_i^{k-1}+\nabla_i f(X^{k-1})\right)+(1-\beta^k)\left(\gamma_i^k+\nabla_i f(X^k)\right) \\
&= \nabla_i f\left(X^k\right)+\alpha^k \gamma_i^k+\beta^k S_i^k+\beta^k \mu_i^{k-1}.
\end{align*}

This yields
\begin{align*}
\mu_i^k &=M_i^k-\nabla_i f\left(X^k\right) \\
&=\alpha^k \gamma_i^k+\beta^k S_i^k+\beta^k \mu_i^{k-1} \\
&=\sum_{\tau=1}^{k} \beta^{(\tau+1): k} \alpha^{\tau} \gamma_i^{\tau}+\sum_{\tau=1}^k \beta^{\tau: k} S_i^\tau+\beta^{1: k} \mu^0_i\\
&=\sum_{\tau=0}^{k} \beta^{(\tau+1): k} \alpha^{\tau} \gamma_i^{\tau}+\sum_{\tau=1}^k \beta^{\tau: k} S_i^\tau,
\end{align*}
where the last line follows from the fact that $M^0_i=\nabla_i f_{\xi^0}(X^0)$ and $\beta^0=0$.
Thus,
\begin{align*}
\mathbb{E}\left[\left\|M_i^k - \nabla_i f(X^k)\right\|_{(i) \star}\right]
&=\mathbb{E}\left[\left\|\mu_i^k\right\|_{(i) \star}\right] \\
&\leq \mathbb{E}\left[\left\|\sum_{\tau=0}^{k} \beta^{(\tau+1): k} \alpha^{\tau} \gamma_i^{\tau}\right\|_{(i) \star}\right]+\sum_{\tau=1}^k \beta^{\tau: k} \mathbb{E}\left[\left\|S_i^\tau\right\|_{(i) \star}\right]\\
&= \sqrt{\sum_{\tau=0}^{k} \left(\beta^{(\tau+1): k} \alpha^{\tau}\right)^2 \mathbb{E}\left[\left\|\gamma_i^{\tau}\right\|_{(i) \star}^2\right]}+\sum_{\tau=1}^k \beta^{\tau: k} \mathbb{E}\left[\left\|S_i^\tau\right\|_{(i) \star}\right],
\end{align*}
where in the last equality we used the fact that for all \( q < l \)
\begin{align*}
    \mathbb{E} \left[ ({\gamma^l_i})^\top \gamma^q_i \right] &= \mathbb{E} \left[ \mathbb{E} \left[ ({\gamma^l_i})^\top \gamma^q_i \mid X^l_i \right] \right] = \mathbb{E} \left[ \mathbb{E} \left[ \gamma^l_i \mid X^l_i \right]^\top \gamma^q_i \right] \\
    &= \mathbb{E} \left[ \left( \mathbb{E} \left[ \nabla_i f_{\xi^l}(X^l) - \nabla_i f(X^l) \mid X^l_i \right] \right)^\top \gamma^q_i \right] = 0,
\end{align*}
Using Assumptions~\ref{ass:generalized-smoothness} and \ref{ass:bounded_var}, we get
\begin{align*}
\mathbb{E}\left[\left\|\gamma_i^{\tau}\right\|_{(i) \star}^2\right] = \mathbb{E} \Bigg[ \underbrace{\mathbb{E}\left[\left\|\gamma_i^{\tau}\right\|_{(i) \star}^2 \mid X^{\tau}_i \right]}_{\leq \sigma^2}\Bigg] \leq \sigma^2
\end{align*}
and
\begin{align*}
\left\|S_i^\tau\right\|_{(i) \star} \leq \left( L^0_i + L^1_i \| \nabla _i f(X^{\tau}) \|_{(i) \star} \right) \|X^{\tau+1}_i - X^{\tau}_i \|_{(i)} \leq \left( L^0_i + L^1_i \| \nabla _i f(X^{\tau}) \|_{(i) \star} \right) t^{\tau}_i.
\end{align*}

Therefore,
\begin{eqnarray*}
\mathbb{E}\left[\left\|M_i^k - \nabla_i f(X^k)\right\|_{(i) \star}\right] &\leq& \sigma \sqrt{\sum_{\tau=0}^{k} \left(\beta^{(\tau+1): k} \alpha^{\tau} \right)^2} + L^0_i \sum_{\tau=1}^k \beta^{\tau: k} t^{\tau}_i \\
&&+ L^1_i \sum_{\tau=1}^k \beta^{\tau: k} t^{\tau}_i \mathbb{E} \left[\| \nabla _i f(X^{\tau}) \|_{(i) \star}\right].
\end{eqnarray*}

Combining the last inequality with (\ref{eq:main_eq_momentum}) gives
\begin{align}\label{eq:main_eq_momentum_2}
\sum_{i=1}^p \sum_{k=0}^{K-1} t_i^k \mathbb{E}[\| \nabla _i f(X^k) \|_{(i) \star}] \leq \Delta^0 +\sum_{i=1}^p \Bigg[ & \underbrace{2 \sigma \sum_{k=0}^{K-1} t^k_i \sqrt{\sum_{\tau=0}^{k} \left(\beta^{(\tau+1): k} \alpha^{\tau}\right)^2}}_{=:I_1} + \underbrace{2 L^0_i \sum_{k=0}^{K-1} t^k_i \sum_{\tau=1}^k \beta^{\tau: k} t^{\tau}_i}_{=:I_2}\notag\\
&+ \underbrace{2 L^1_i \sum_{k=0}^{K-1} t^k_i \sum_{\tau=1}^k \beta^{\tau: k} t^{\tau}_i \mathbb{E} \left[\| \nabla _i f(X^{\tau}) \|_{(i) \star}\right]}_{=:I_3}\notag\\
&+ \underbrace{\frac{L^0_i}{2} \sum_{k=0}^{K-1} \left(t^k_i\right)^2}_{=:I_4} + \frac{L^1_i}{2} \sum_{k=0}^{K-1} \left(t^k_i\right)^2 \mathbb{E}\left[\|\nabla _i f(X^k)\|_{(i) \star}\right]\Bigg].
\end{align}

Let us now upper-bound each term \( I_i \), \( i = 1, 2, 3, 4\).

\( I_1 \): using Lemma~\ref{lemma:tech2}, we obtain
\begin{align*}
I_1 \leq \sigma t_i\left(7+2\sqrt{2 e^2} \log (K)\right).
\end{align*}

\( I_2 \): using Lemma~\ref{lemma:tech2}, we obtain
\begin{align*}
I_2 \leq 14 L^0_i t_i^2\left(3 +\log (K)\right).
\end{align*}

\( I_3 \): rearranging the sums and using Lemma~\ref{lemma:tech1} with $a=\tau+1$, $b=K$, $p=3/4$ and $q=1/2$, we have
\begin{align*}
I_3 &= 2 L^1_i \sum_{k=0}^{K-1} t^k_i \sum_{\tau=1}^k \beta^{\tau: k} t^{\tau}_i \mathbb{E} \left[\| \nabla _i f(X^{\tau}) \|_{(i) \star}\right] \\
&= 2 L^1_i \sum_{\tau=1}^{K-1} t^\tau_i \left(\sum_{k=\tau}^{K-1} t_i^k \beta^{\tau: k}\right) \mathbb{E} \left[\| \nabla _i f(X^{\tau}) \|_{(i) \star}\right]\\ 
&= 2 L^1_i \sum_{\tau=1}^{K-1} t^\tau_i t_i \left(\sum_{k=\tau}^{K-1} (k+1)^{-3/4} \beta^{\tau: k}\right) \mathbb{E} \left[\| \nabla _i f(X^{\tau}) \|_{(i) \star}\right]\\
&\leq 2 L^1_i \sum_{\tau=1}^{K-1} t^\tau_i t_i \tau^{-1/4} \underbrace{e^{2\left((\tau+1)^{1/2}-\tau^{1/2}\right)}}_{\leq e^{2(\sqrt{2}-1)} \ \text{for} \ \tau \geq 1} \mathbb{E} \left[\| \nabla _i f(X^{\tau}) \|_{(i) \star}\right]\\
&\leq 2 e^{2(\sqrt{2}-1)} L^1_i \sum_{\tau=1}^{K-1} t^\tau_i t_i \tau^{-1/4} \mathbb{E} \left[\| \nabla _i f(X^{\tau}) \|_{(i) \star}\right]\\
&\leq 2 e^{2(\sqrt{2}-1)} L^1_i \sum_{k=0}^{K-1} t^k_i t_i \mathbb{E} \left[\| \nabla _i f(X^k) \|_{(i) \star}\right].
\end{align*}

\( I_4 \):
\begin{align*}
I_4 &= \frac{L^0_i}{2} \sum_{k=0}^{K-1} \left(t^k_i\right)^2 \leq \frac{L^0_i}{2} \sum_{k=0}^{\infty} \left(t^k_i\right)^2 = \frac{L^0_i}{2} t_i^2 \sum_{k=0}^{\infty} (1+k)^{-3/2} \\
&\leq \frac{L^0_i}{2} t_i^2 \left( 1 + \int_1^{\infty} \frac{1}{z^{3/2}} \, dz \right) = \frac{3L^0_i}{2} t_i^2.
\end{align*}

Combining the upper-bounds for \( I_i \), \( i = 1, 2, 3, 4\) with (\ref{eq:main_eq_momentum_2}) gives
\begin{align*}
\sum_{i=1}^p \sum_{k=0}^{K-1} t_i^k \mathbb{E}[\| \nabla _i f(X^k) \|_{(i) \star}] \leq \Delta^0 +\sum_{i=1}^p \Bigg[ & \sigma t_i\left(7+2\sqrt{2 e^2} \log (K)\right) + 14 L^0_i t_i^2\left(3 +\log (K)\right)\\
&+ 2 e^{2(\sqrt{2}-1)} L^1_i \sum_{k=0}^{K-1} t^k_i t_i \mathbb{E} \left[\| \nabla _i f(X^k) \|_{(i) \star}\right]\\
&+ \frac{3L^0_i}{2} t_i^2 + \frac{L^1_i}{2} \sum_{k=0}^{K-1} \left(t^k_i\right)^2 \mathbb{E}[\|\nabla _i f(X^k)\|_{(i) \star}]\Bigg].
\end{align*}

Using the fact that \( t^k_i = t_i (1+k)^{-3/4} \leq t_i \), and denoting \( C := 2e^{2(\sqrt{2} - 1)} + \frac{1}{2} \leq 5.1\), we get
\begin{align*}
\sum_{i=1}^p \sum_{k=0}^{K-1} t_i^k \mathbb{E}[\| \nabla _i f(X^k) \|_{(i) \star}] \leq \Delta^0 +\sum_{i=1}^p \Bigg[ & \sigma t_i\left(7+2\sqrt{2 e^2} \log (K)\right) + 14 L^0_i t_i^2\left(\frac{87}{28} +\log (K)\right)\\
&+ C L^1_i t_i \sum_{k=0}^{K-1} t^k_i \mathbb{E} \left[\| \nabla _i f(X^k) \|_{(i) \star}\right]\Bigg].
\end{align*}

Now, let us consider two options: (1) $L^1_i = 0$ for all $i \in \{1, \dots, p\}$ and (2) $L^1_i \neq 0$, for all $i \in \{1, \dots, p\}$.

\paragraph{Case 1: $L^1_i = 0$, $i = 1, \dots, p$.}

In this case,
\begin{align*}
\sum_{i=1}^p \sum_{k=0}^{K-1} t_i^k \mathbb{E}[\| \nabla _i f(X^k) \|_{(i) \star}] \leq \Delta^0 +\sum_{i=1}^p \Bigg[ & \sigma t_i\left(7+2\sqrt{2 e^2} \log (K)\right) + 14 L^0_i t_i^2\left(\frac{87}{28} +\log (K)\right)\Bigg],
\end{align*}
and therefore,
\begin{eqnarray*}
&&\hspace{-2cm}\min_{k=0,\ldots,K-1} \sum_{i=1}^p t_i \mathbb{E}[\| \nabla _i f(X^k) \|_{(i) \star}] \\
&\leq& \frac{1}{K} \sum_{k=0}^{K-1} \sum_{i=1}^p t_i \mathbb{E}[\| \nabla _i f(X^k) \|_{(i) \star}]\\
&\leq& \frac{1}{K^{1/4}} \sum_{k=0}^{K-1} \sum_{i=1}^p t_i (1+k)^{-3/4} \mathbb{E}[\| \nabla _i f(X^k) \|_{(i) \star}]\\
&=& \frac{1}{K^{1/4}} \sum_{k=0}^{K-1} \sum_{i=1}^p t^k_i \mathbb{E}[\| \nabla _i f(X^k) \|_{(i) \star}]\\
&\leq& \frac{\Delta^0}{K^{1/4}} + \frac{1}{K^{1/4}} \sum_{i=1}^p \Bigg[ \sigma t_i \left(7+2\sqrt{2 e^2} \log (K)\right) + L^0_i t^2_i \left(\frac{87}{2} + 14\log (K)\right)\Bigg].
\end{eqnarray*}

\paragraph{Case 2:  $L^1_i \neq 0$, $i = 1, \dots, p$.}

Let us choose $t_i=\frac{1}{12L^1_i}$. Then
\begin{align*}
\sum_{i=1}^p \sum_{k=0}^{K-1} t_i^k \mathbb{E}[\| \nabla _i f(X^k) \|_{(i) \star}] \leq 2\Delta^0 +\sum_{i=1}^p \Bigg[ 2\sigma t_i\left(7+2\sqrt{2 e^2} \log (K)\right) + L^0_i t_i^2\left(87 +28\log (K)\right)\Bigg],
\end{align*}
and hence 
\begin{eqnarray*}
&&\hspace{-1.1cm}\min_{k=0,\ldots,K-1} \sum_{i=1}^p \frac{1}{12L^1_i} \mathbb{E}[\| \nabla _i f(X^k) \|_{(i) \star}] \\
&\leq& \frac{1}{K} \sum_{k=0}^{K-1} \sum_{i=1}^p t_i \mathbb{E}[\| \nabla _i f(X^k) \|_{(i) \star}]\\
&\leq& \frac{1}{K^{1/4}} \sum_{k=0}^{K-1} \sum_{i=1}^p t_i (1+k)^{-3/4} \mathbb{E}[\| \nabla _i f(X^k) \|_{(i) \star}]\\
&=& \frac{1}{K^{1/4}} \sum_{i=1}^p \sum_{k=0}^{K-1} t_i^k \mathbb{E}[\| \nabla _i f(X^k) \|_{(i) \star}]\\
&\leq& \frac{2\Delta^0}{K^{1/4}} + \frac{1}{K^{1/4}} \sum_{i=1}^p \Bigg[ \frac{\sigma}{6L^1_i}\left(7+2\sqrt{2 e^2} \log (K)\right) + \frac{L^0_i}{144(L^1_i)^2}\left(87 +28\log (K)\right)\Bigg].
\end{eqnarray*}

\end{proof}

\newpage

\section{Additional experimental results and details} \label{appendix:more_exp}

\subsection{Experimental details}\label{sec:exp_details}

All experiments for the \texttt{NanoGPT} model are conducted using PyTorch\footnote{PyTorch Documentation. Available at: \url{https://pytorch.org/docs/stable/index.html}} with Distributed Data Parallel (DDP)\footnote{Distributed Data Parallel (DDP) in PyTorch. Available at: \url{https://pytorch.org/docs/stable/notes/ddp.html}} across 4 NVIDIA A100 GPUs (40GB each). For the \texttt{CNN} experiments, training is performed on a single NVIDIA A100 GPU (40GB). The training and evaluation pipelines are implemented using open-source codebases \citep{jordan2024cifar10, jordan2024moddednanogpt, code_scion}, with all modifications clearly documented and properly referenced where applicable.

For LMO-based methods, we compute inexact LMOs using the Newton–Schulz iteration when an analytical solution is unavailable (e.g., for SVD-type updates), following the approach proposed by \citet{jordan2024muon}. This method provides a computationally efficient approximation of the required orthogonalization while preserving the convergence behavior of the overall algorithm.

\subsection{Fitting $L^0_i$ and $L^1_i$}\label{sec:nanogpt_l0l1}
To minimize the Euclidean error between the true value \( \hat{L}_i[k] \) and its approximation \( \hat{L}_i^{\text{approx}}[k] \), while penalizing underestimation, we incorporate a hinge-like penalty term. Specifically, we fit $L^0_i$ and $L^1_i$ by minimizing the loss function
\begin{align}\label{eq:l0l1_loss}
    \mathcal{L}_i\parens{L^0_i, L^1_i} := \sum_{k=0}^{K-1} \left( \hat{L}_i[k] - \hat{L}_i^{\text{approx}}[k] \right)^2 + \lambda \sum_{k=0}^{K-1} \max\left(0, \hat{L}_i[k] - \hat{L}_i^{\text{approx}}[k]\right)^2.
\end{align}

The first term of $\mathcal{L}_i$ captures the standard Euclidean (squared) error, while the second term introduces an additional penalty proportional to the amount of underestimation (i.e., when \( \hat{L}_i[k] > \hat{L}_i^{\text{approx}}[k] \)). The hyperparameter \( \lambda \geq 0 \) controls the strength of this penalty.

\subsection{Training NanoGPT on FineWeb.}

In this section, we present additional results and experimental details for the experiment described in the main text, which involves training a \texttt{NanoGPT} model on the \texttt{FineWeb} dataset using the \algname{unScion} optimizer.

\subsubsection{Empirical validation of \Cref{ass:generalized-smoothness}}

We begin by presenting additional results for the experiment described in \Cref{sec:nanogpt-fineweb}, aimed at empirically validating \Cref{ass:generalized-smoothness}. We plot the estimated \emph{trajectory smoothness}
\begin{align*}
    \hat{L}_i[k] \eqdef \frac{\|\nabla_i f_{\xi^{k+1}} (X^{k+1}) - \nabla_i f_{\xi^k} (X^{k}) \|_{(i)\star}}{\|X_i^{k+1} - X_i^{k}\|_{(i)}}
\end{align*}
and its approximation
$$\hat{L}_i^{\text{approx}}[k] \eqdef L_i^0 + L_i^1 \|\nabla_i f_{\xi^{k+1}}(X^{k+1})\|_{(i) \star}$$
as functions of the iteration index $k$, where $L^0_i, L^1_i\geq0$ are fitted using the procedure described in \Cref{sec:nanogpt_l0l1}.

Figures~\ref{fig:1}, \ref{fig:2}, and \ref{fig:3_appndx} show results for parameter groups from the embedding layer and from the 4th and~8th transformer blocks. Similar patterns are observed across all layers. In each case, we see a strong agreement between $\hat{L}_i[k]$ and $\hat{L}_i^{\text{approx}}[k]$, suggesting that \Cref{ass:generalized-smoothness} holds approximately along the optimization trajectory.
\begin{figure}[H]
    \centering
    \begin{subfigure}{0.32\textwidth}
        \includegraphics[width=\textwidth]{plots/nanogpt/main_exp/plot7_0.pdf}
    \end{subfigure}
    \begin{subfigure}{0.32\textwidth}
        \includegraphics[width=\textwidth]{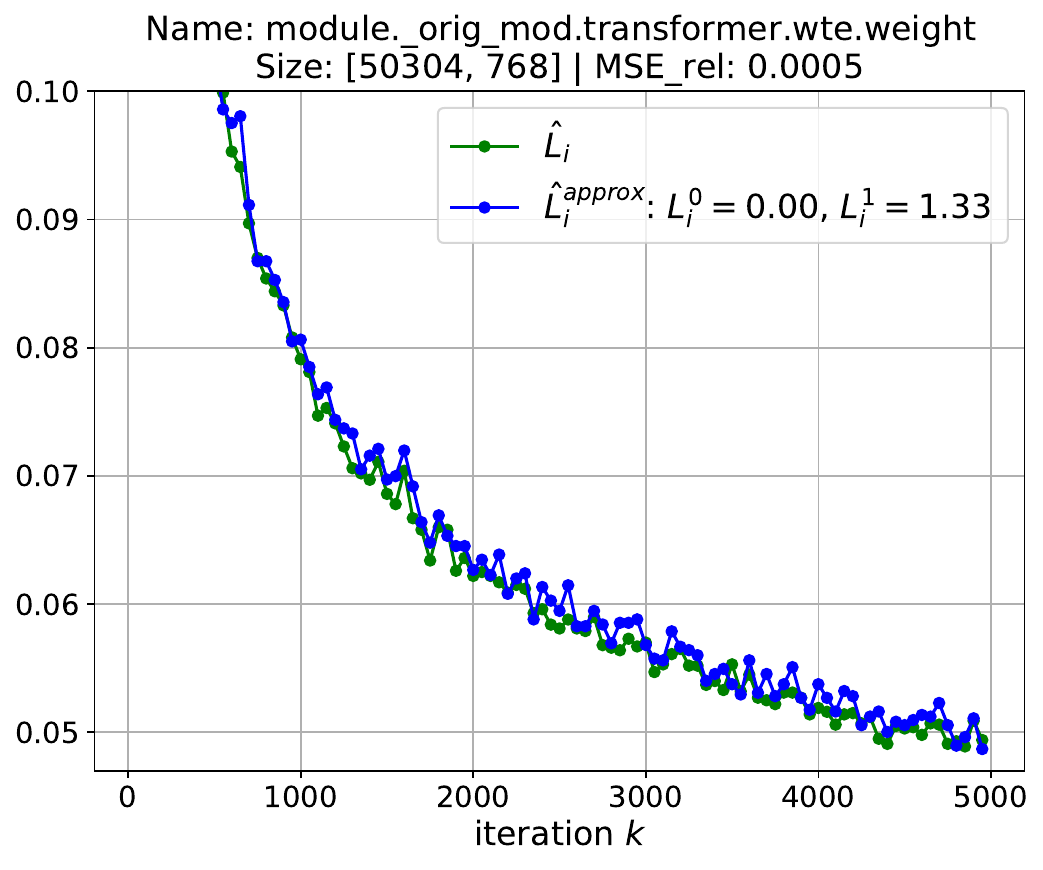}
    \end{subfigure}
    \caption{\small Validation of layer-wise $(L^0, L^1)$-smoothness for the group of parameters from the embedding layer of \texttt{NanoGPT-124M} along \algname{unScion} training trajectories. The group norm is \(\|\cdot\|_{(p)} = n_p \|\cdot\|_{1 \to \infty} \), with fitted values $L_p^0 \approx 0$, $L_p^1 \approx 1.3$. The same plot is shown twice with  different $y$-axis limits.}
    \label{fig:1}
\end{figure}
\begin{figure}[H]
    \centering
    \begin{subfigure}{0.32\textwidth}
        \includegraphics[width=\textwidth]{plots/nanogpt/main_exp/plot1.pdf}
    \end{subfigure}
    \hfill
    \begin{subfigure}{0.32\textwidth}
        \includegraphics[width=\textwidth]{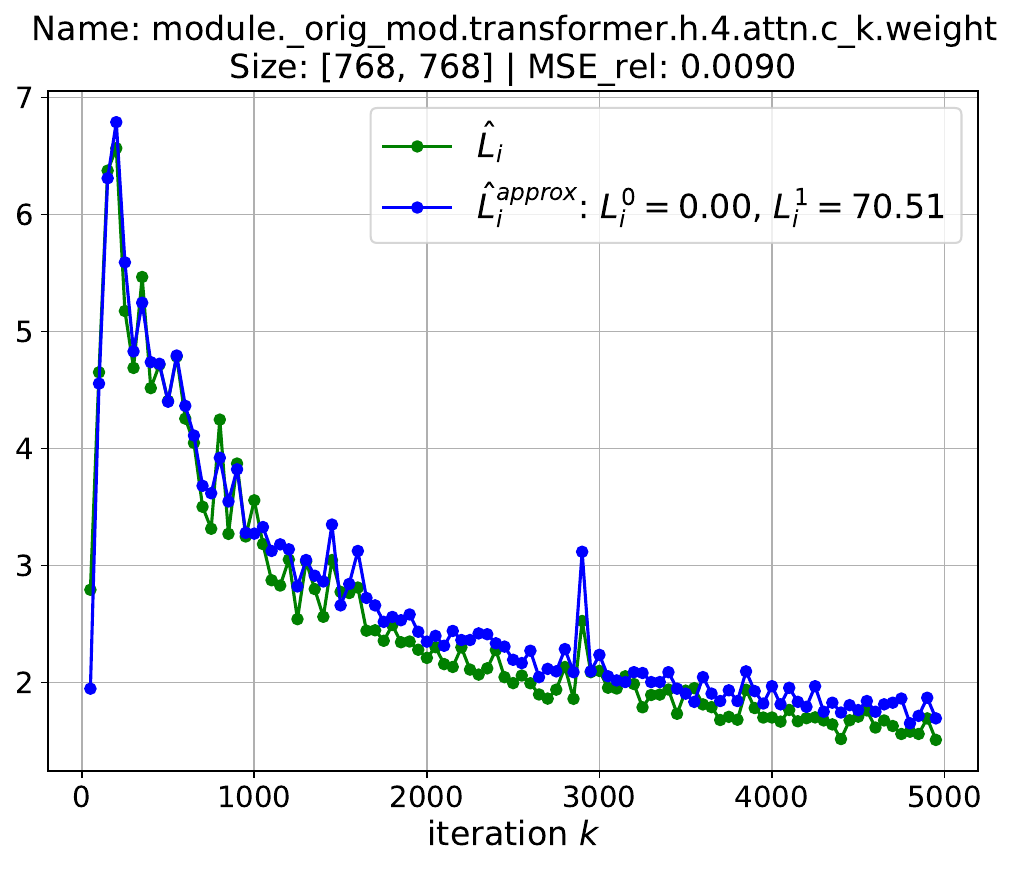}
    \end{subfigure}
    \hfill
    \begin{subfigure}{0.32\textwidth}
        \includegraphics[width=\textwidth]{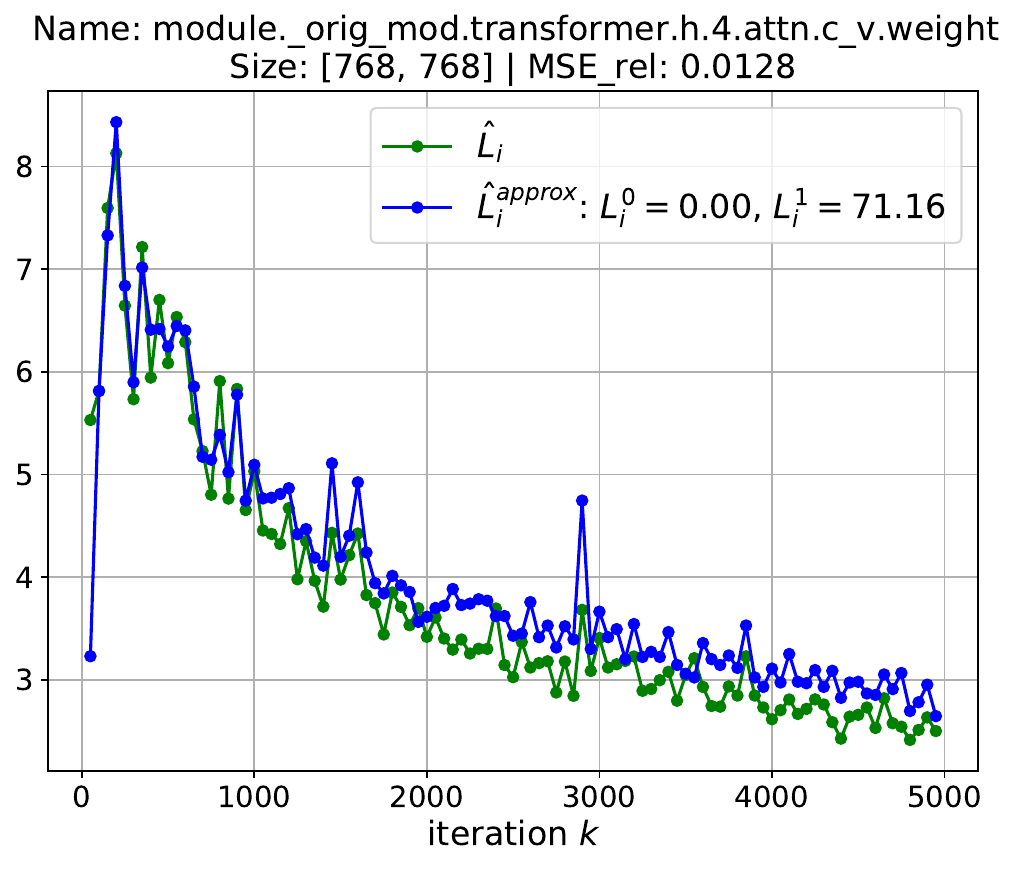}
    \end{subfigure}
    
    
    \begin{subfigure}{0.32\textwidth}
        \includegraphics[width=\textwidth]{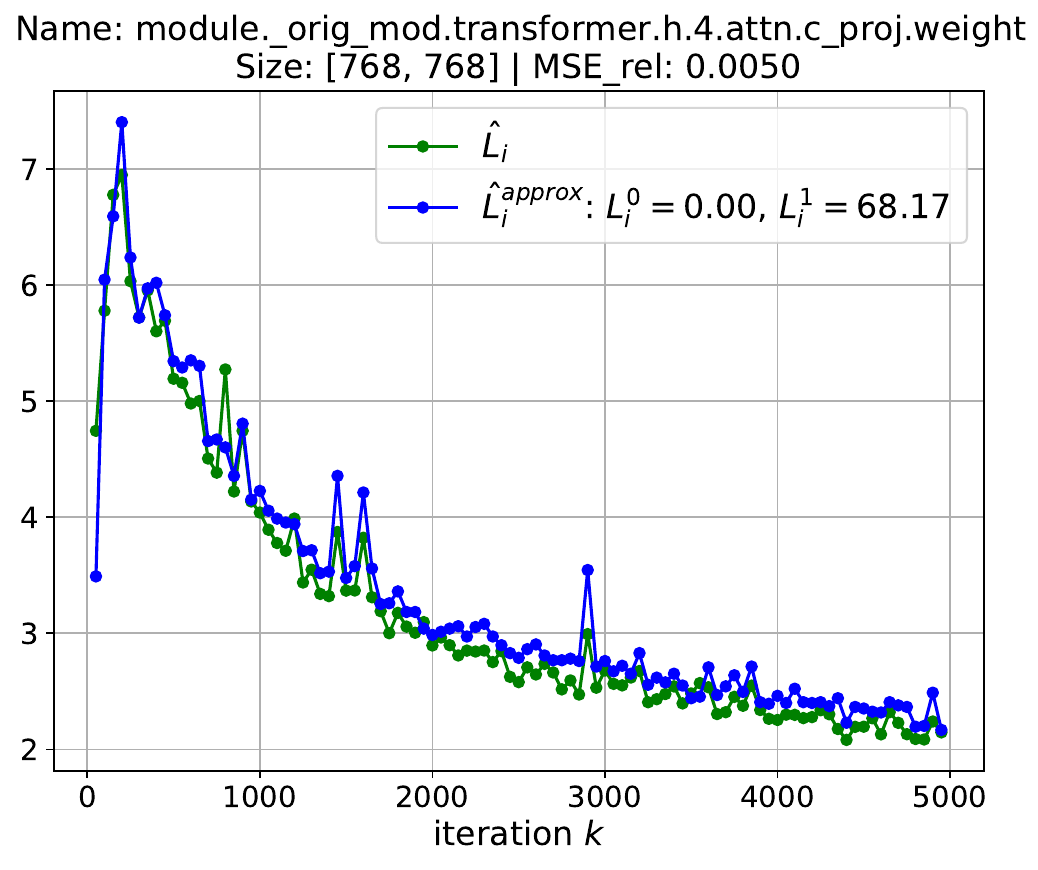}
    \end{subfigure}
    \hfill
    \begin{subfigure}{0.32\textwidth}
        \includegraphics[width=\textwidth]{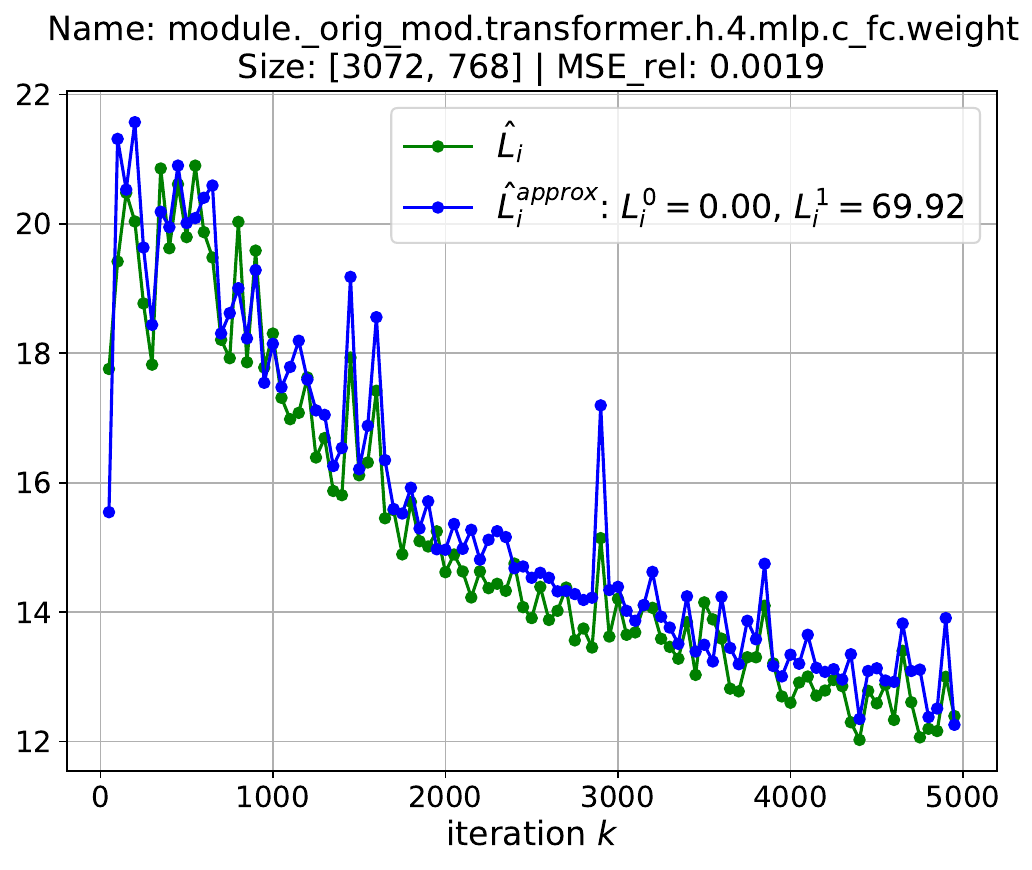}
    \end{subfigure}
    \hfill
    \begin{subfigure}{0.32\textwidth}
        \includegraphics[width=\textwidth]{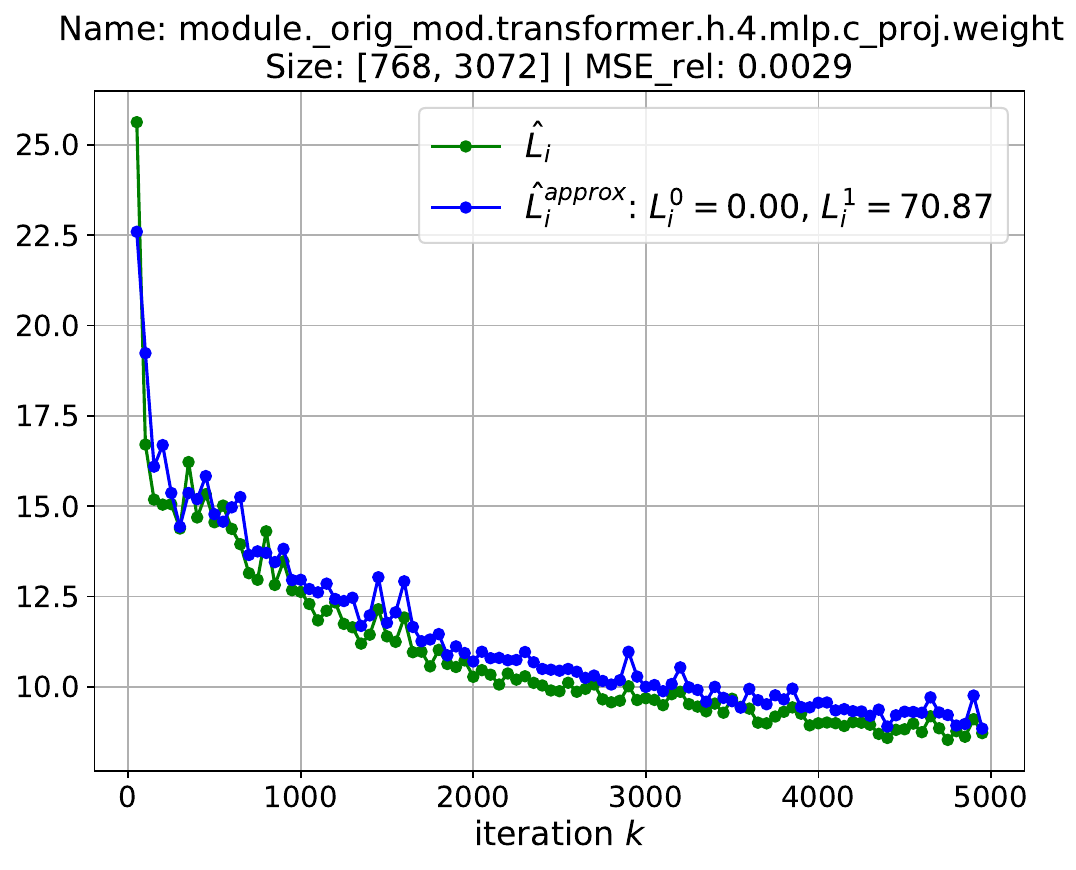}
    \end{subfigure}
    \caption{\small Validation of layer-wise $(L^0, L^1)$-smoothness for the group of parameters from the 4th transformer block of \texttt{NanoGPT-124M} along \algname{unScion} training trajectories. The group norms are \( \|\cdot\|_{(i)} = \sqrt{\nicefrac{n_i}{m_i}} \|\cdot\|_{2 \to 2} \), with fitted values $L_i^0 \approx 0$, $L_i^1 \approx 70$.}
    \label{fig:2}
\end{figure}

\begin{figure}[h]
    \centering
    \begin{subfigure}{0.32\textwidth}
        \includegraphics[width=\textwidth]{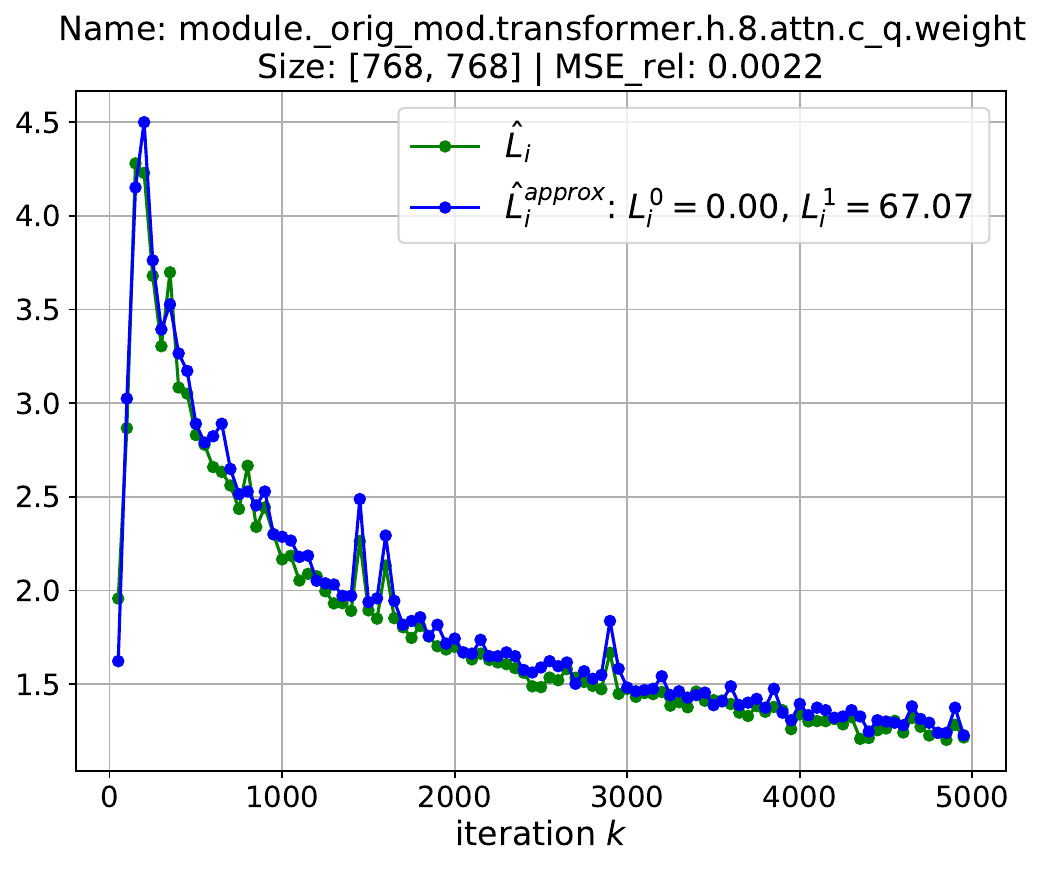}
    \end{subfigure}
    \hfill
    \begin{subfigure}{0.32\textwidth}
        \includegraphics[width=\textwidth]{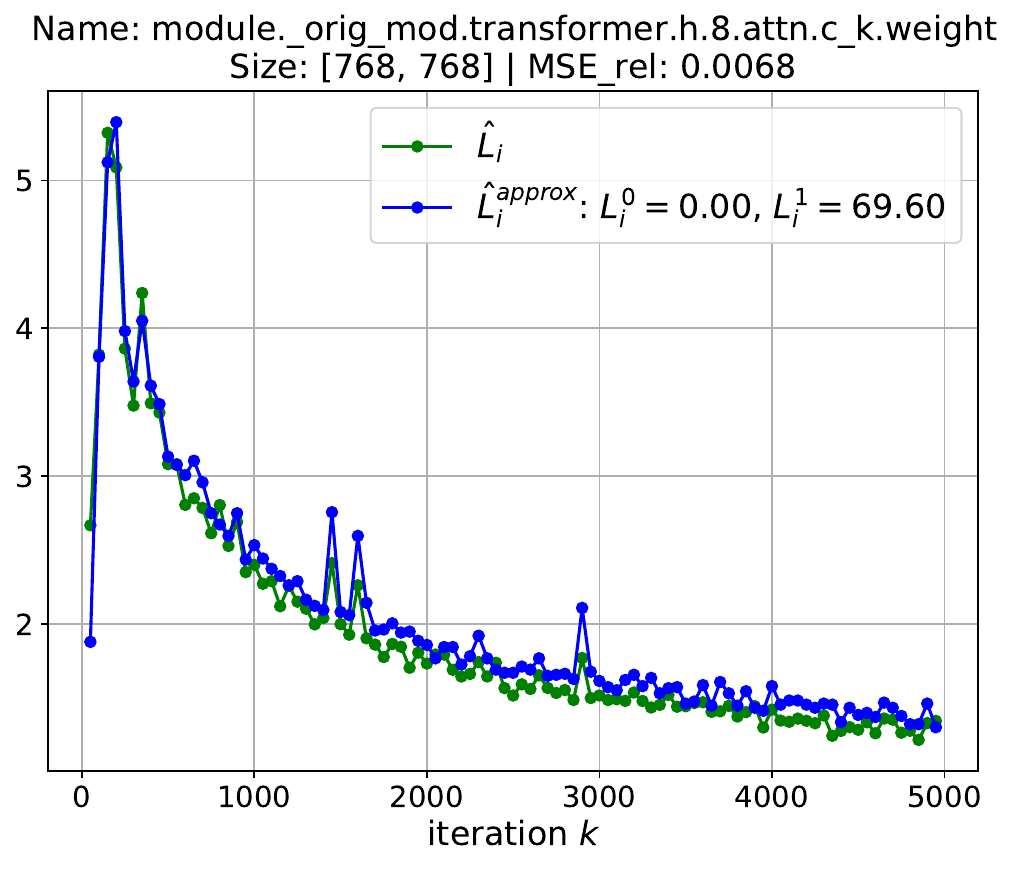}
    \end{subfigure}
    \hfill
    \begin{subfigure}{0.32\textwidth}
        \includegraphics[width=\textwidth]{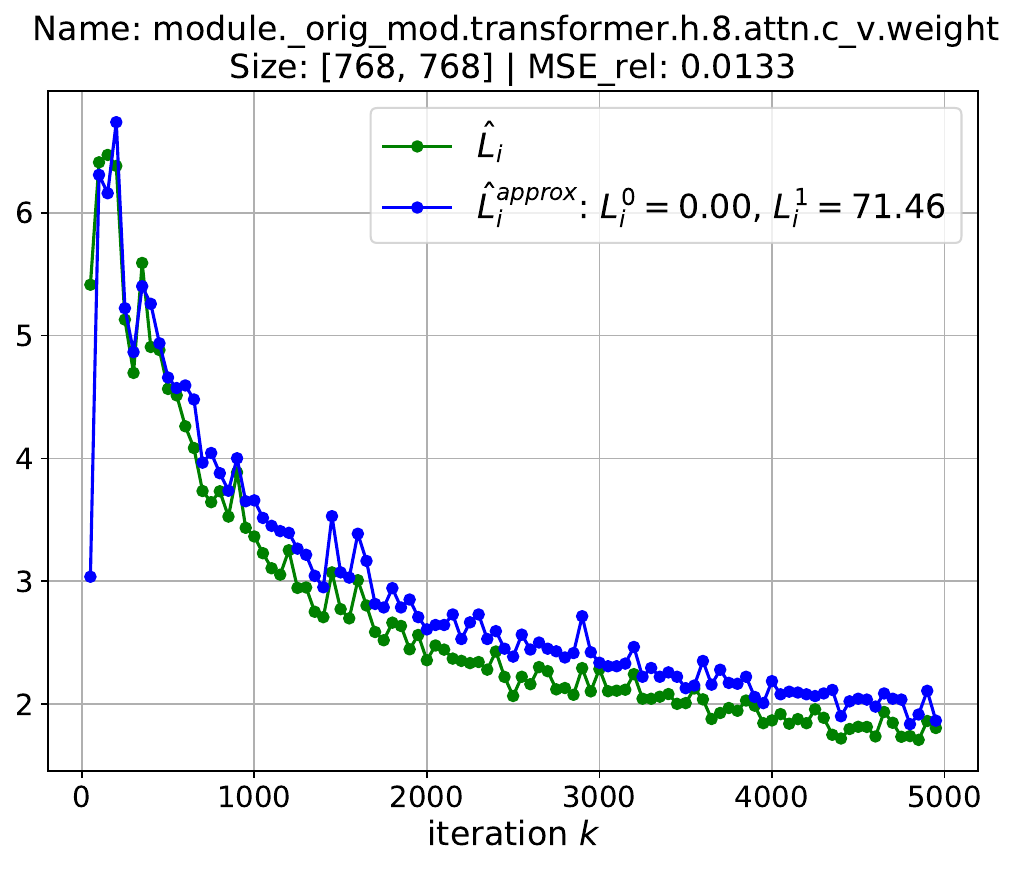}
    \end{subfigure}
    
    
    \begin{subfigure}{0.32\textwidth}
        \includegraphics[width=\textwidth]{plots/nanogpt/main_exp/plot4_1.pdf}
    \end{subfigure}
    \hfill
    \begin{subfigure}{0.32\textwidth}
        \includegraphics[width=\textwidth]{plots/nanogpt/main_exp/plot5_1.pdf}
    \end{subfigure}
    \hfill
    \begin{subfigure}{0.32\textwidth}
        \includegraphics[width=\textwidth]{plots/nanogpt/main_exp/plot6_1.pdf}
    \end{subfigure}
    
    \caption{Validation of layer-wise $(L^0, L^1)$-smoothness for the group of parameters from the 8th transformer block of \texttt{NanoGPT-124M} along \algname{unScion} training trajectories. The group norms are \( \|\cdot\|_{(i)} = \sqrt{\nicefrac{n_i}{m_i}} \|\cdot\|_{2 \to 2} \), with fitted values $L_i^0 \approx 0$, $L_i^1 \approx 70$.}
    \label{fig:3_appndx}
\end{figure}

\newpage
\subsubsection{Generalized smoothness under Euclidean vs. specialized norms} \label{appendix:gensmooth_euclid}

In this experiment, we compare how well the layer-wise $(L^0, L^1)$-smoothness assumption is satisfied under the standard Euclidean norms \( \|\cdot\|_2 \) for each parameter block, as opposed to the specialized norms described in~(\ref{eq:scion_llm}). We adopt the same training setup as in Section~\ref{sec:nanogpt-fineweb}, plotting the estimated trajectory smoothness \( \hat{L}_i \) and its approximation \( \hat{L}_i^{\text{approx}} \) along the training trajectories across several parameter groups. Unlike previous sections, here we do not penalize instances where \( \hat{L}_i > \hat{L}_i^{\text{approx}} \) in order to find the best approximation (i.e., $\lambda = 0$ in \eqref{eq:l0l1_loss}). Additionally, when using the standard Euclidean norm \( \|\cdot\|_2 \) for approximation, we exclude the first point, as it could distort the result.

We evaluate the quality of each approximation using the relative mean squared error ($\mathrm{MSE}_i^{\mathrm{rel}}$, denoted MSE\_rel in the figures), defined as
\begin{align*}
    \text{MSE}_i^{\text{rel}}:= \frac{1}{K} \sum_{i=1}^{K} \left( \frac{\hat{L}_i[k] - \hat{L}_i^{\text{approx}}[k]}{\hat{L}_i[k]} \right)^2,
\end{align*}
where a lower value indicates a better fit.

As shown in Figures~\ref{fig:specific_norms_mse} and~\ref{fig:euclideum_norms_mse}, both visually and in terms of $\mathrm{MSE}_i^{\mathrm{rel}}$, using specialized norms for each group of parameters provides a better approximation than the standard Euclidean norm~\( \|\cdot\|_2 \). Notably, the relative mean squared error $\mathrm{MSE}_i^{\mathrm{rel}}$ is consistently an order of magnitude lower under specialized norms.

\begin{figure}[h]
\vspace{-30pt}
    \centering
    \begin{subfigure}{0.32\textwidth}
        \includegraphics[width=\textwidth]{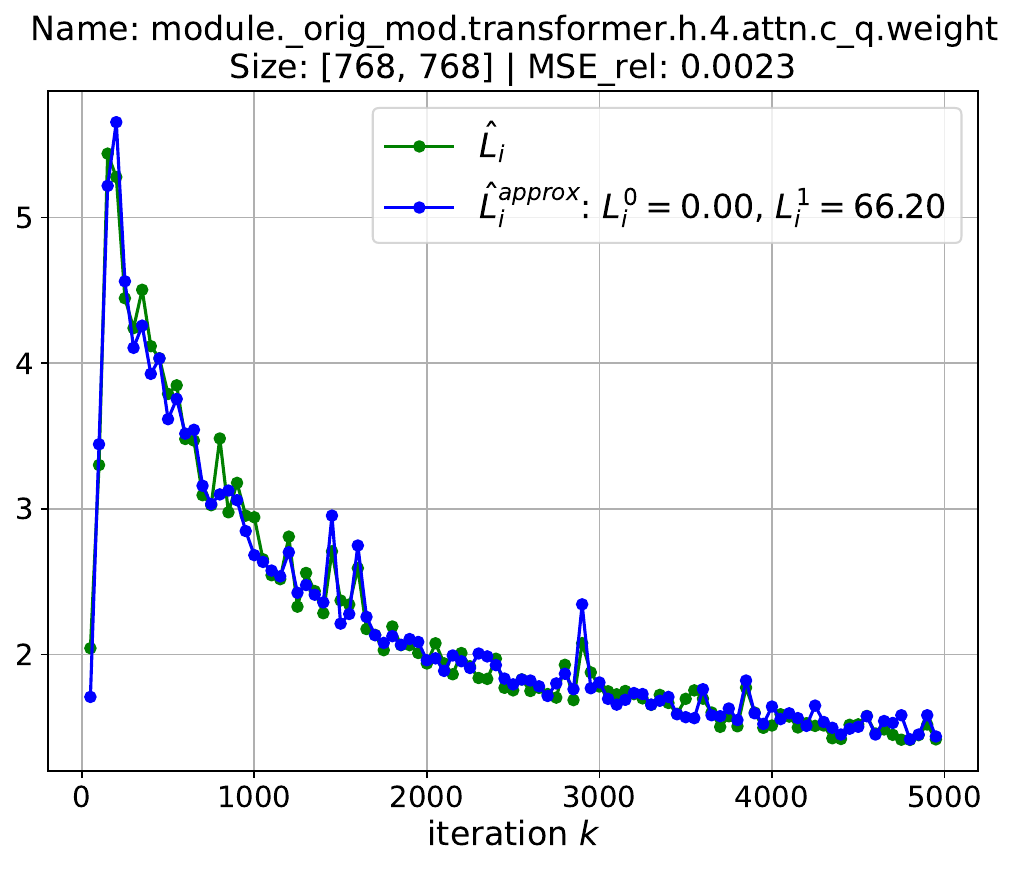}
        \caption{\small $\mathrm{MSE}_i^{\mathrm{rel}} = 0.0023$}
    \end{subfigure}
    \hfill
    \begin{subfigure}{0.32\textwidth}
        \includegraphics[width=\textwidth]{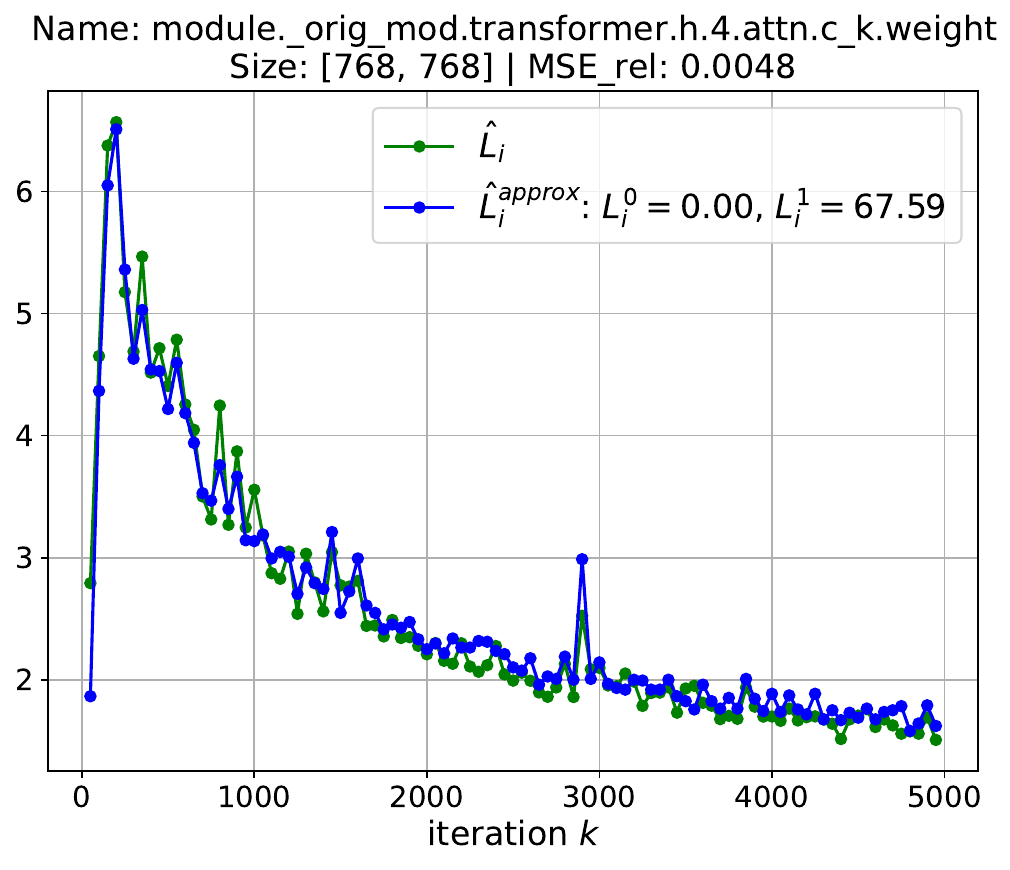}
        \caption{\small $\mathrm{MSE}_i^{\mathrm{rel}} = 0.0048$}
    \end{subfigure}
    \hfill
    \begin{subfigure}{0.32\textwidth}
        \includegraphics[width=\textwidth]{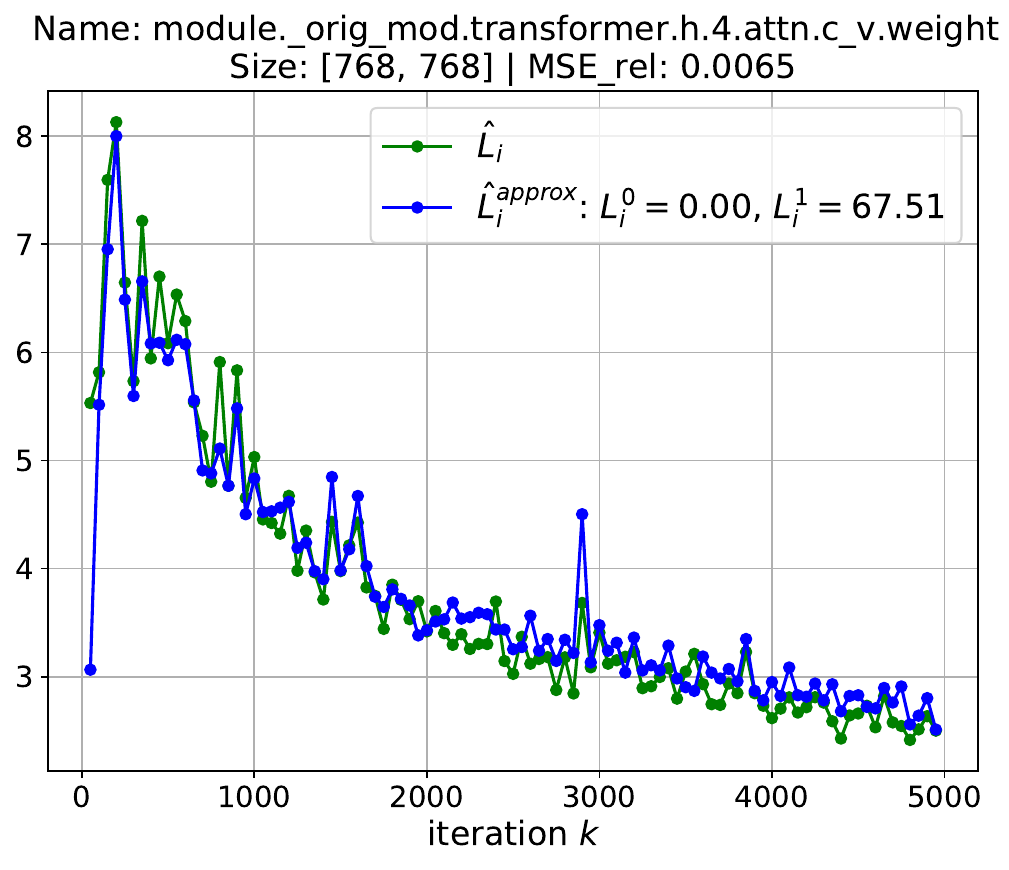}
        \caption{\small $\mathrm{MSE}_i^{\mathrm{rel}} = 0.0065$}
    \end{subfigure}
    
    
    \begin{subfigure}{0.32\textwidth}
        \includegraphics[width=\textwidth]{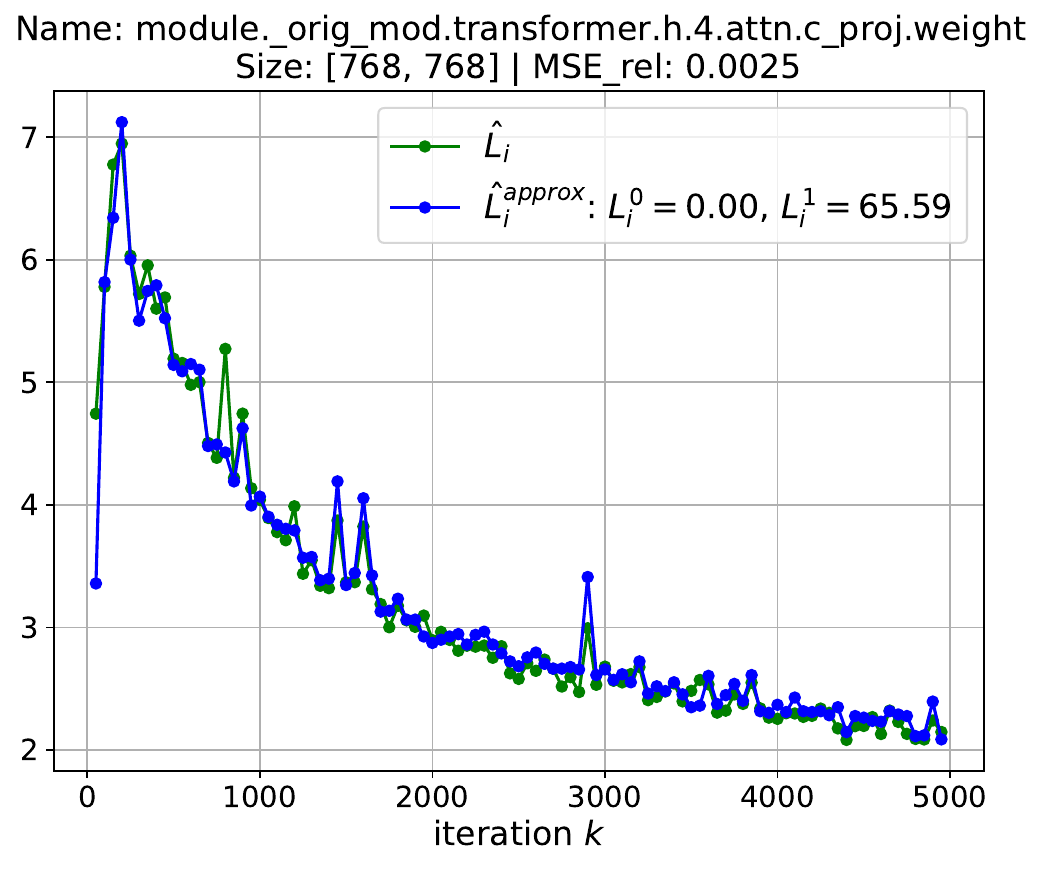}
        \caption{\small $\mathrm{MSE}_i^{\mathrm{rel}} = 0.0025$}
    \end{subfigure}
    \hfill
    \begin{subfigure}{0.32\textwidth}
        \includegraphics[width=\textwidth]{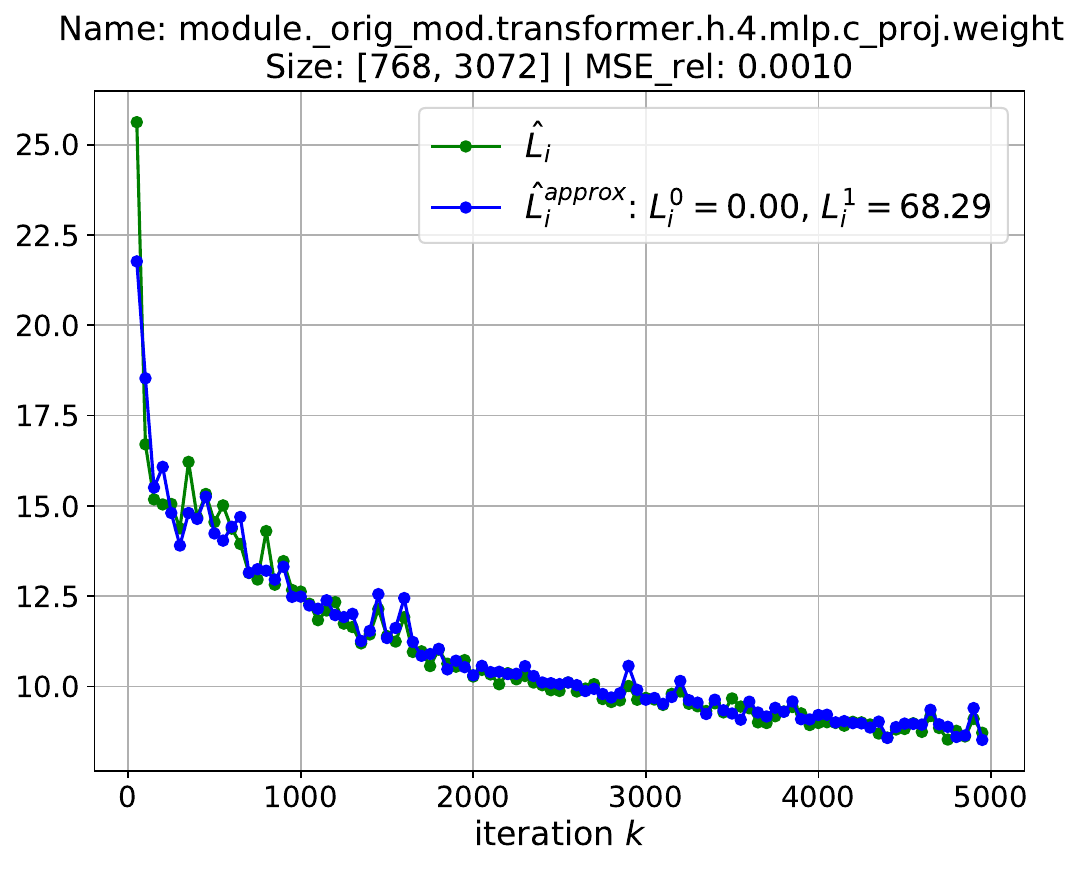}
        \caption{\small $\mathrm{MSE}_i^{\mathrm{rel}} = 0.001$}
    \end{subfigure}
    \hfill
    \begin{subfigure}{0.32\textwidth}
        \includegraphics[width=\textwidth]{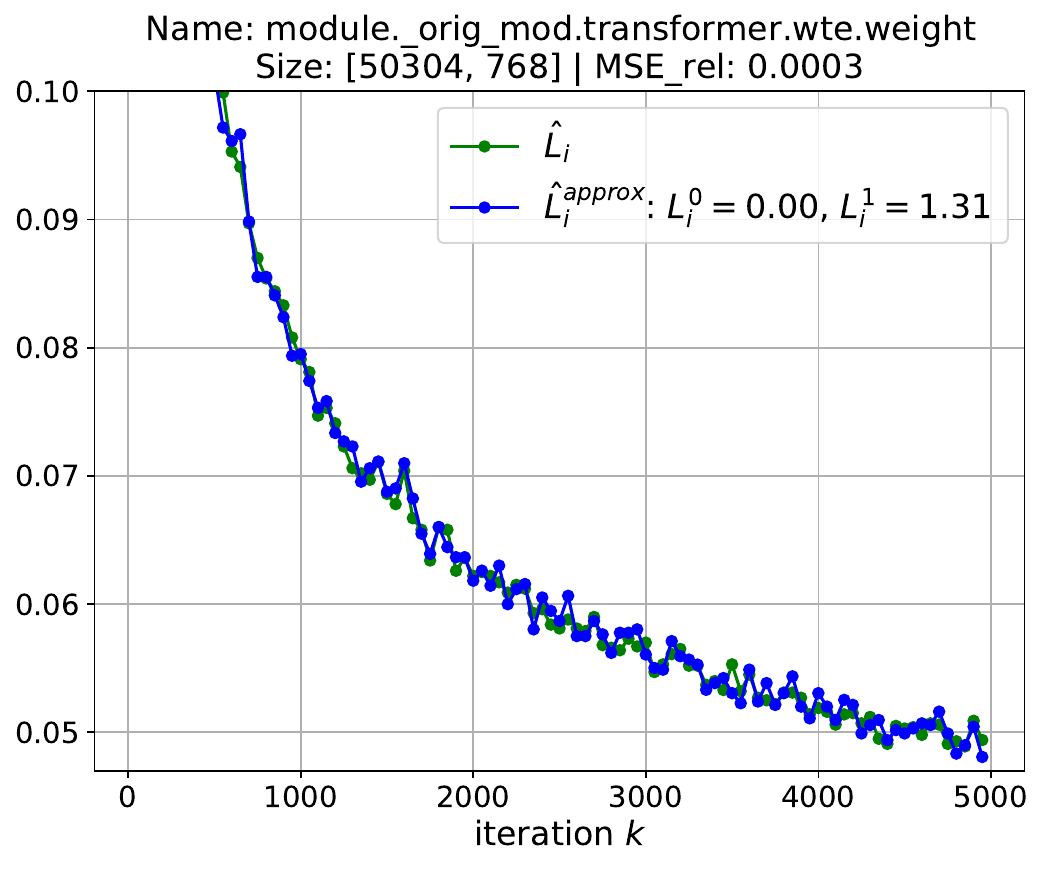}
        \caption{\small $\mathrm{MSE}_i^{\mathrm{rel}} = 0.0003$}
    \end{subfigure}
    
    \caption{\small Validation of layer-wise $(L^0, L^1)$-smoothness for different groups of parameters in \texttt{NanoGPT-124M} along training trajectories of \algname{unScion} using the specialized norm choices defined in~(\ref{eq:scion_llm}).}
    \label{fig:specific_norms_mse}
\end{figure}

\begin{figure}[H]
    \centering
    \begin{subfigure}{0.32\textwidth}
        \includegraphics[width=\textwidth]{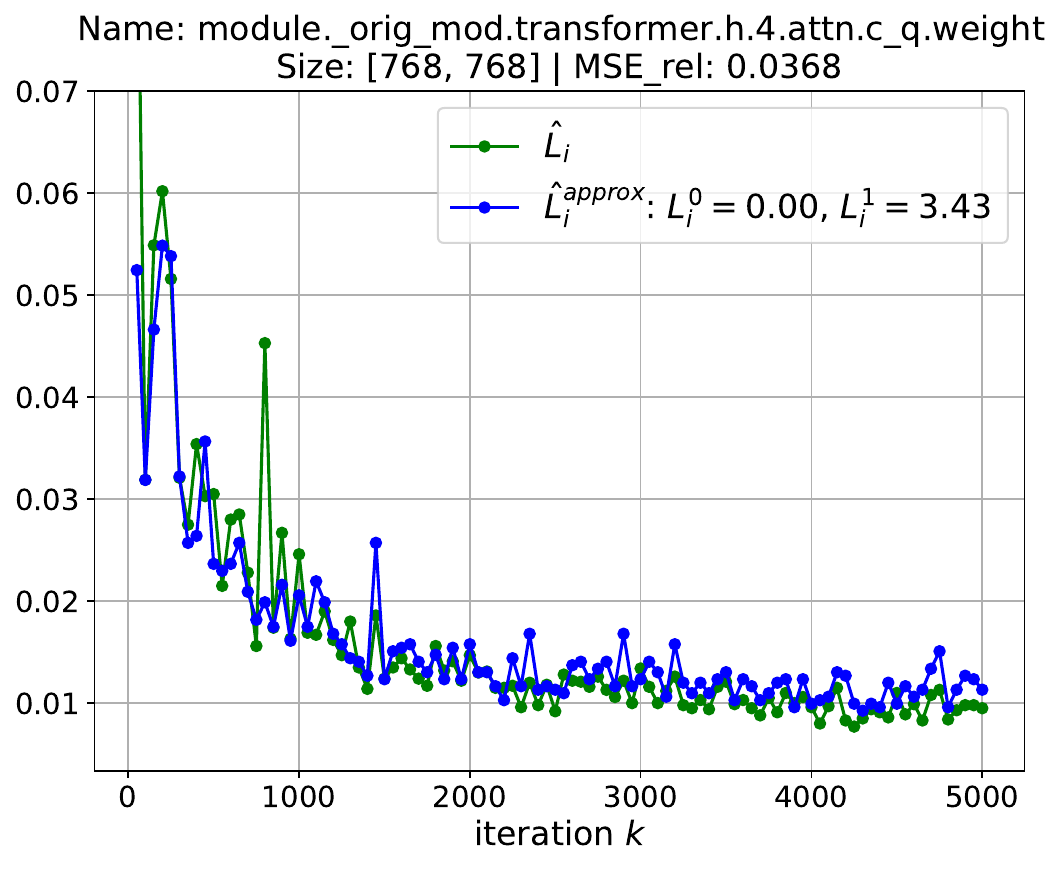}
        \caption{\small $\mathrm{MSE}_i^{\mathrm{rel}} = 0.0368$}
    \end{subfigure}
    \hfill
    \begin{subfigure}{0.32\textwidth}
        \includegraphics[width=\textwidth]{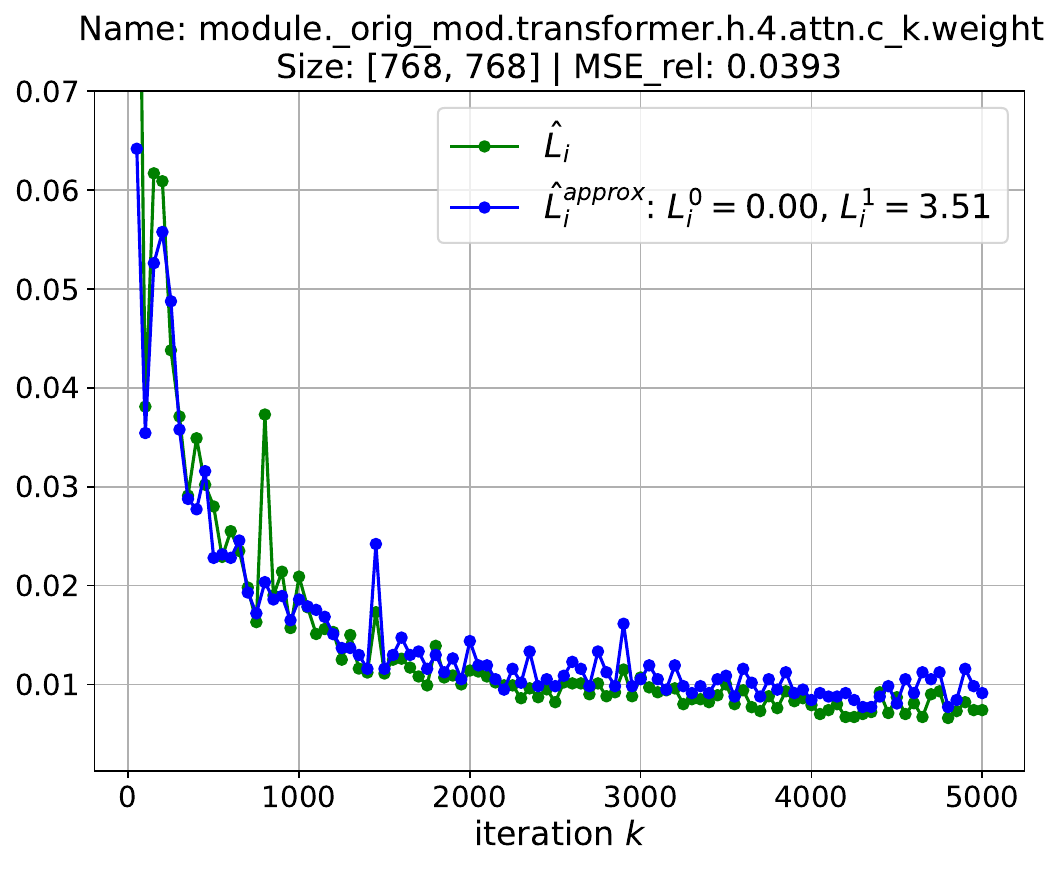}
        \caption{\small $\mathrm{MSE}_i^{\mathrm{rel}} = 0.0393$}
    \end{subfigure}
    \hfill
    \begin{subfigure}{0.32\textwidth}
        \includegraphics[width=\textwidth]{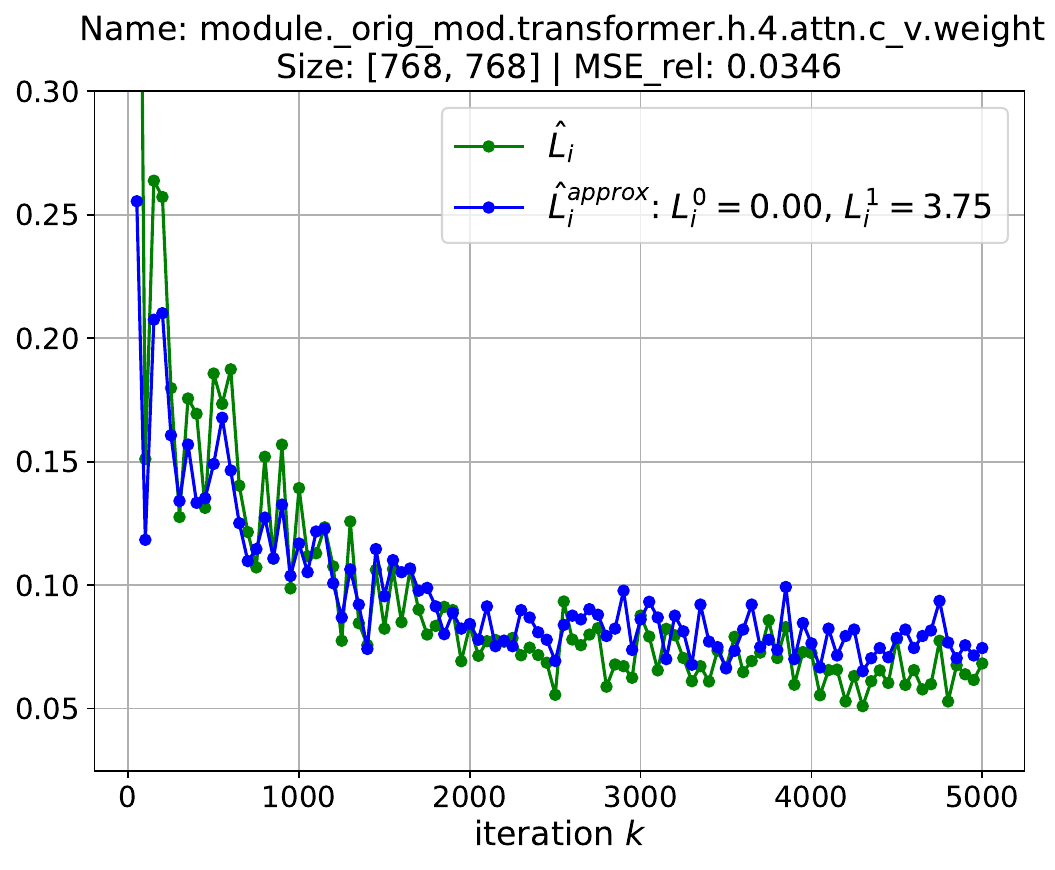}
        \caption{\small $\mathrm{MSE}_i^{\mathrm{rel}} = 0.0346$}
    \end{subfigure}
    
    
    \begin{subfigure}{0.32\textwidth}
        \includegraphics[width=\textwidth]{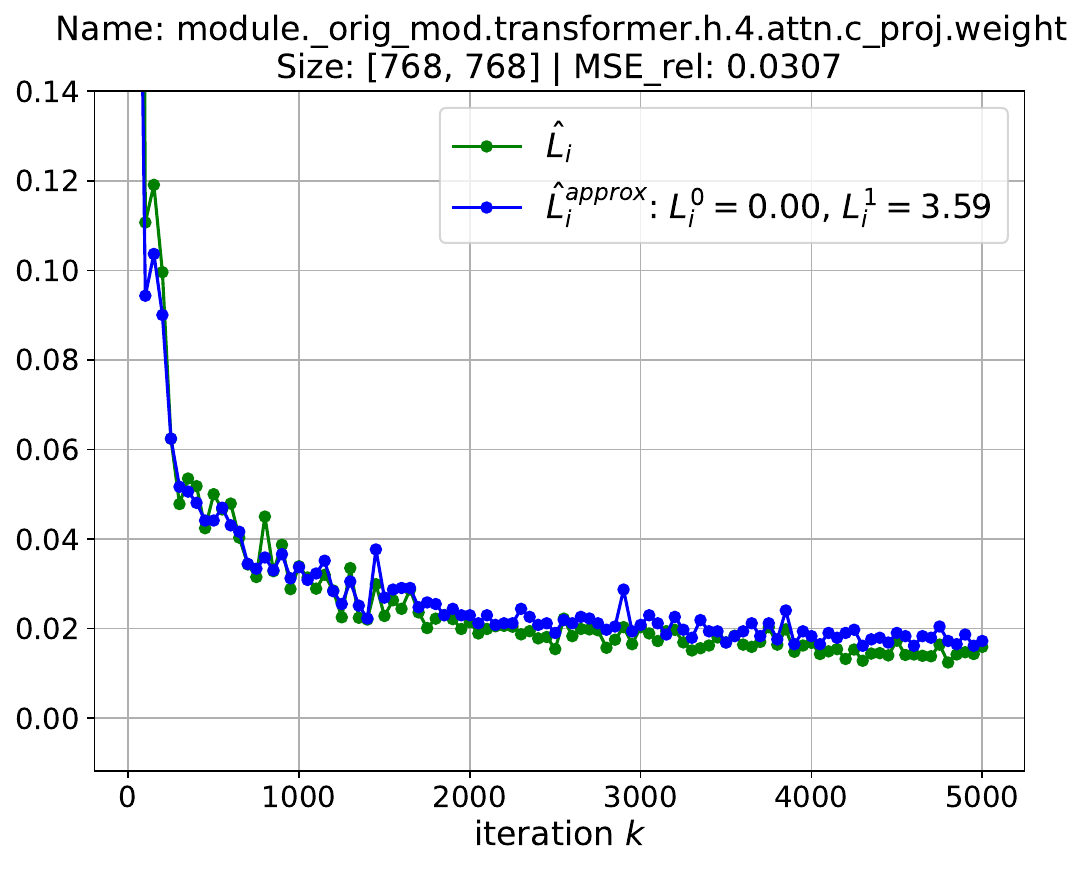}
        \caption{\small $\mathrm{MSE}_i^{\mathrm{rel}} = 0.0307$}
    \end{subfigure}
    \hfill
    \begin{subfigure}{0.32\textwidth}
        \includegraphics[width=\textwidth]{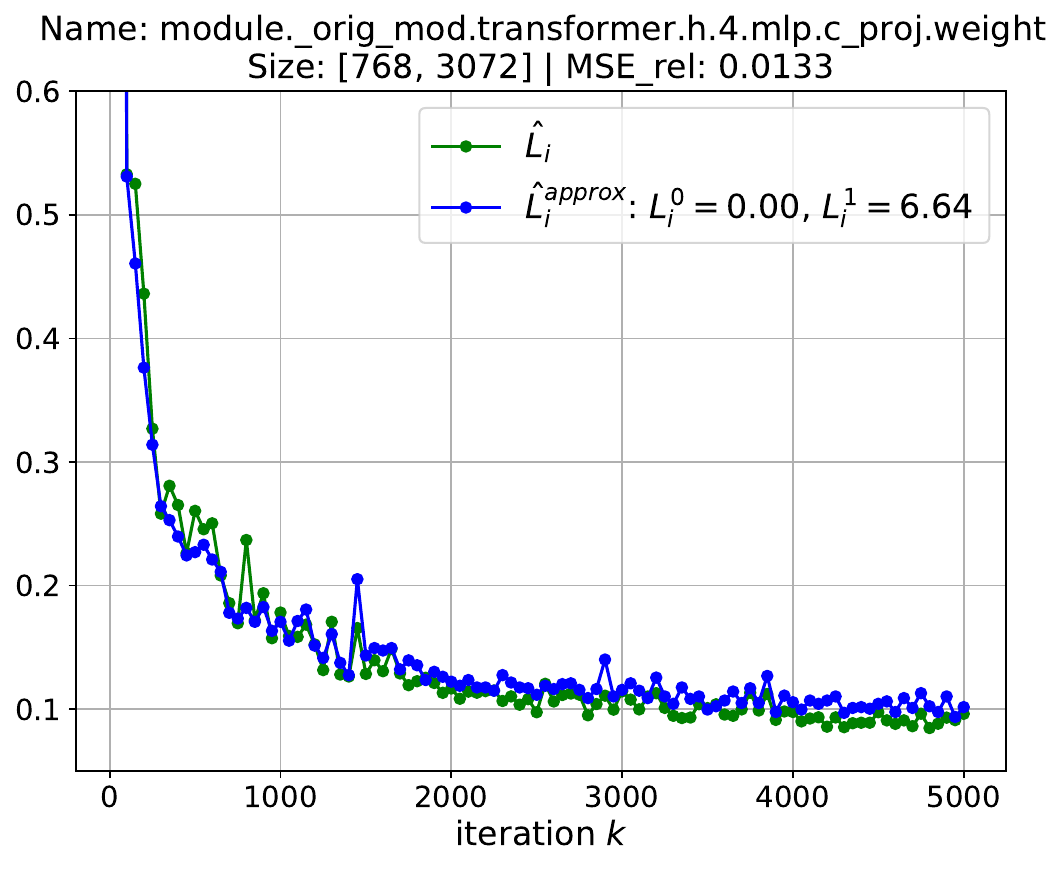}
        \caption{\small $\mathrm{MSE}_i^{\mathrm{rel}} = 0.0133$}
    \end{subfigure}
    \hfill
    \begin{subfigure}{0.32\textwidth}
        \includegraphics[width=\textwidth]{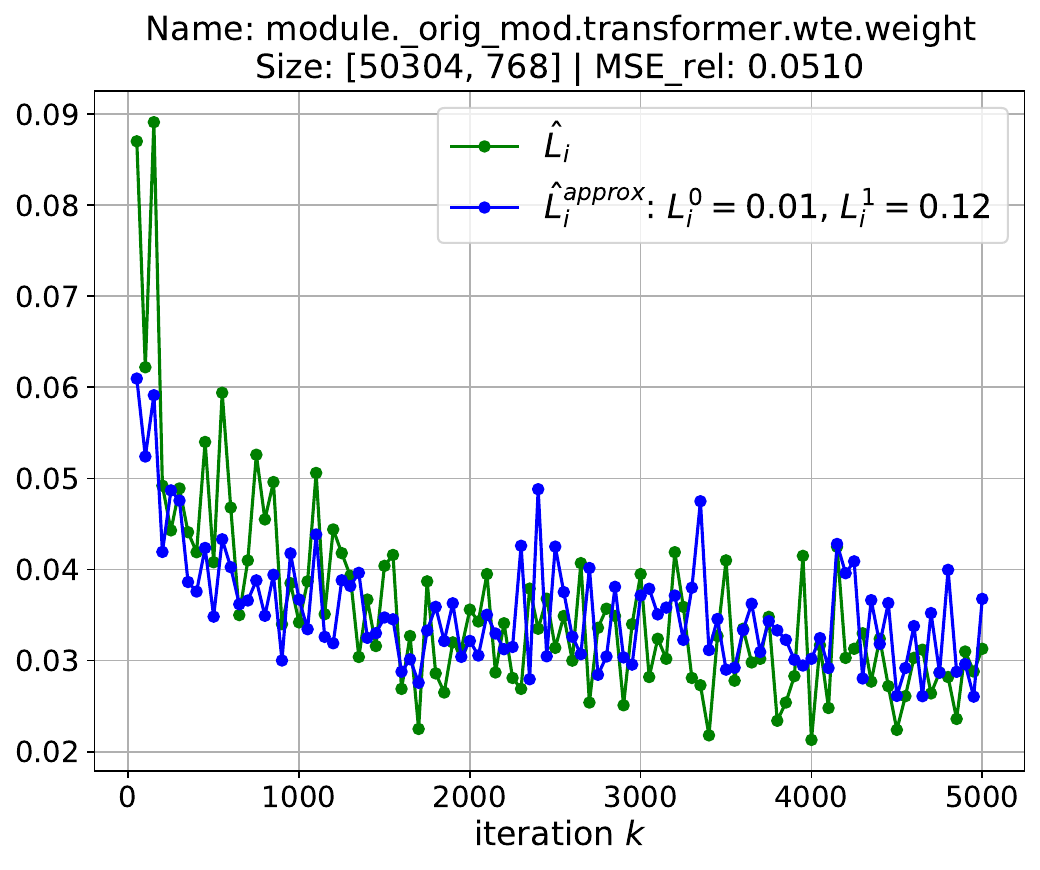}
        \caption{\small $\mathrm{MSE}_i^{\mathrm{rel}} = 0.051$}
    \end{subfigure}
    
    \caption{\small Validation of layer-wise $(L^0, L^1)$-smoothness for different groups of parameters in \texttt{NanoGPT-124M} along training trajectories of \algname{unScion} using the standard Euclidean norm \( \|\cdot\|_2 \).}
    \label{fig:euclideum_norms_mse}
\end{figure}

\subsubsection{Learning rate transfer from AdamW}
\label{sec:Adam_exp}

We now aim to verify layer-wise $(L^0, L^1)$-smoothness following the approach used in Section~\ref{sec:nanogpt-fineweb}, but employing the \algname{AdamW} optimizer. We use hyperparameters specified in \citet[Table 7]{pethick2025training}. In Figure~\ref{fig:Adam}, we present the results for the estimated trajectory smoothness $\hat{L}_i$ and its approximation $\hat{L}_i^{\text{approx}}$ across several parameter groups along the training trajectories. Notably, for the group of parameters from the embedding layer $X_p$ (the last plot in Figure~\ref{fig:Adam}), the fitted value of $L^1_p$ is approximately 20–30 times smaller than in other groups. Since in all plots we observe that $L^0_i \ll L^1_i \|\nabla_i f_{\xi^k}(X^k)\|_{(i) \star}$, \Cref{theorem:1} implies that $t^k_i \approx \nicefrac{1}{L^k_i}$. Thus, \( t^k_p \) should be 20–30 times larger than \( t^k_i \) for \( i = 1, \dots, p-1 \), which is consistent with the tuned parameters from \citet[Table 7]{pethick2025training}.

This insight provides an efficient and principled method for initializing learning rates in \algname{Scion}. Smoothness statistics collected during standard \algname{AdamW} training (which is commonly used for training LLMs) can serve as a strong prior, allowing practitioners to directly incorporate structure-aware choices, such as larger stepsizes for embedding layers, into their tuning process.
Importantly, computing these statistics is computationally inexpensive, introducing minimal additional cost.

\begin{figure}[h]
    \centering
    \begin{subfigure}{0.32\textwidth}
        \includegraphics[width=\textwidth]{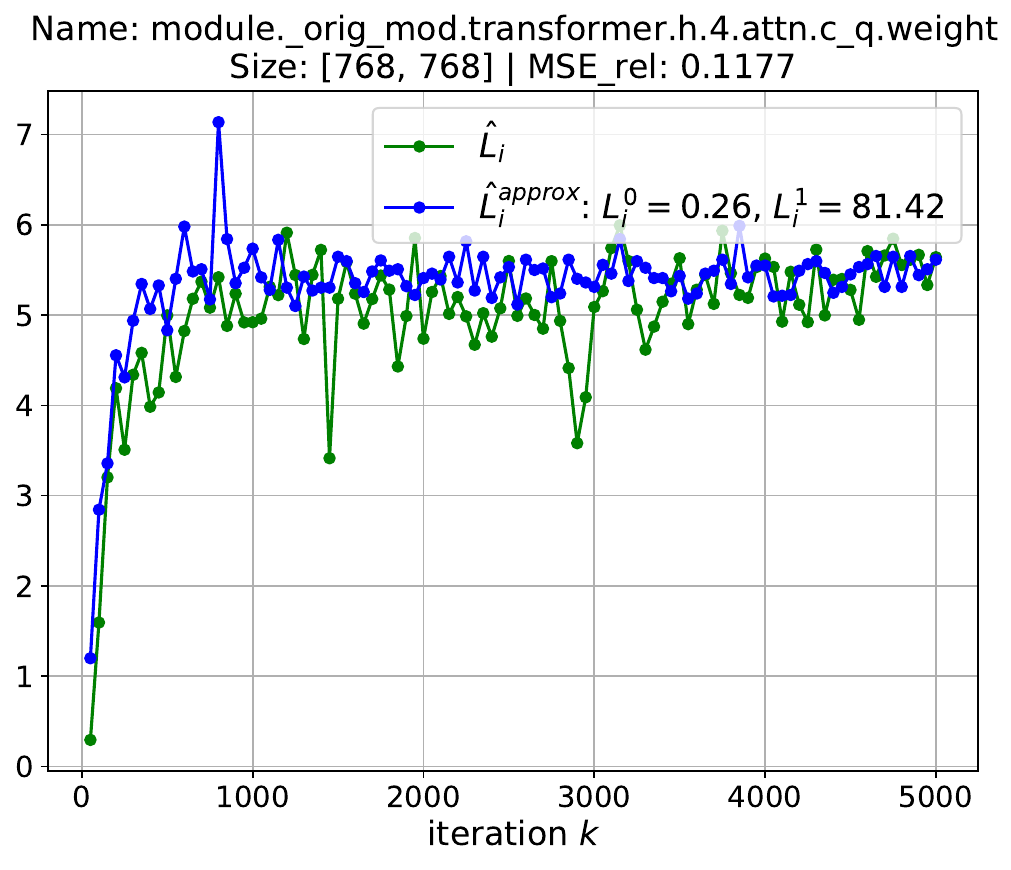}
    \end{subfigure}
    \hfill
    \begin{subfigure}{0.32\textwidth}
        \includegraphics[width=\textwidth]{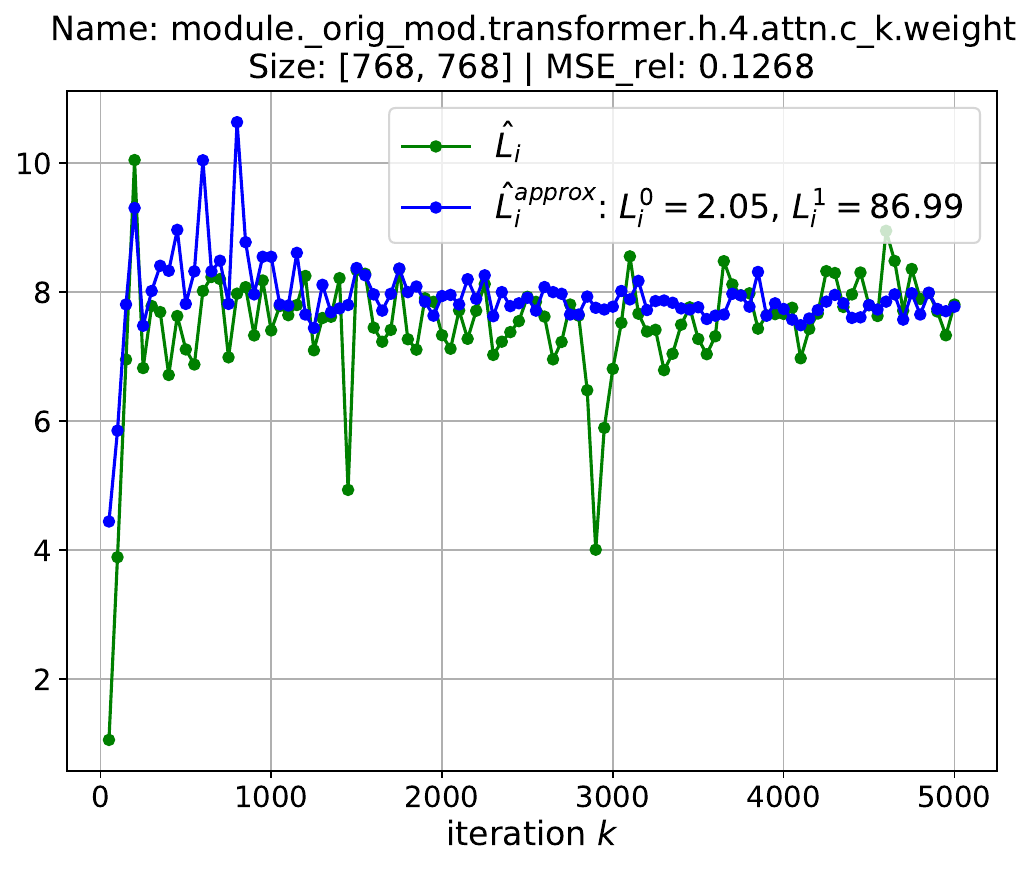}
    \end{subfigure}
    \hfill
    \begin{subfigure}{0.32\textwidth}
        \includegraphics[width=\textwidth]{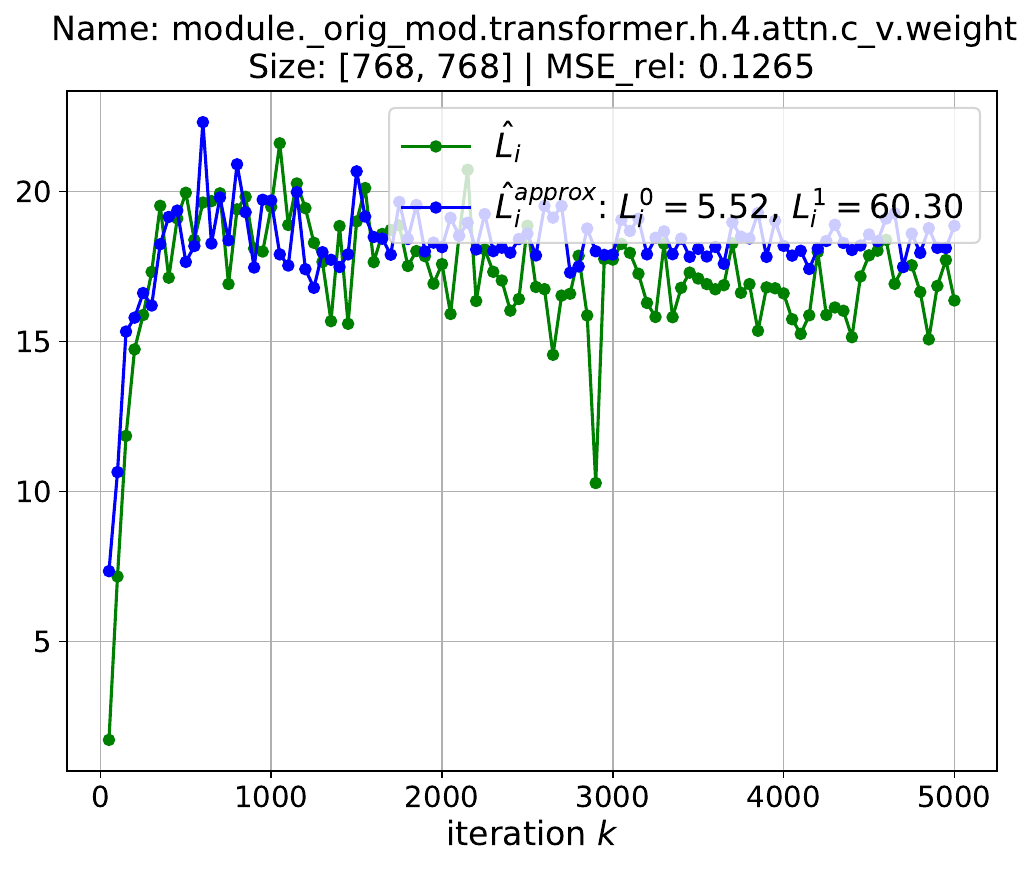}
    \end{subfigure}
    
    \vspace{0.5cm}
    
    \begin{subfigure}{0.32\textwidth}
        \includegraphics[width=\textwidth]{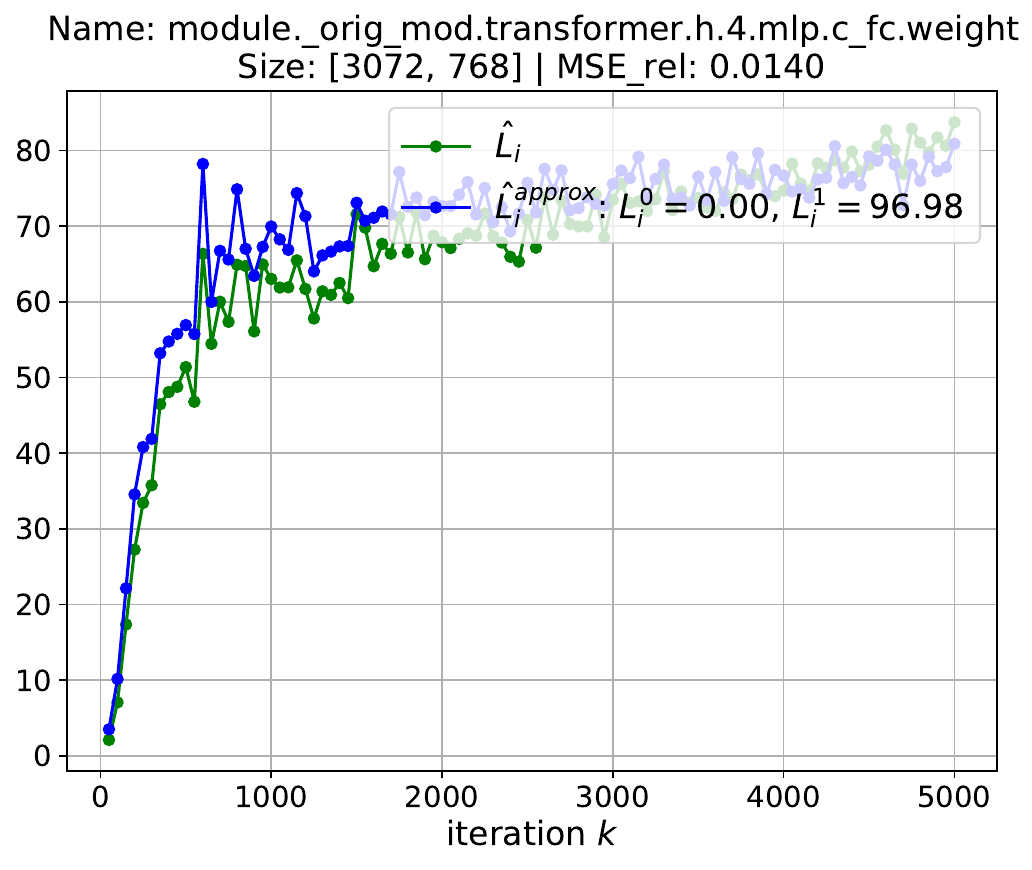}
    \end{subfigure}
    \hfill
    \begin{subfigure}{0.32\textwidth}
        \includegraphics[width=\textwidth]{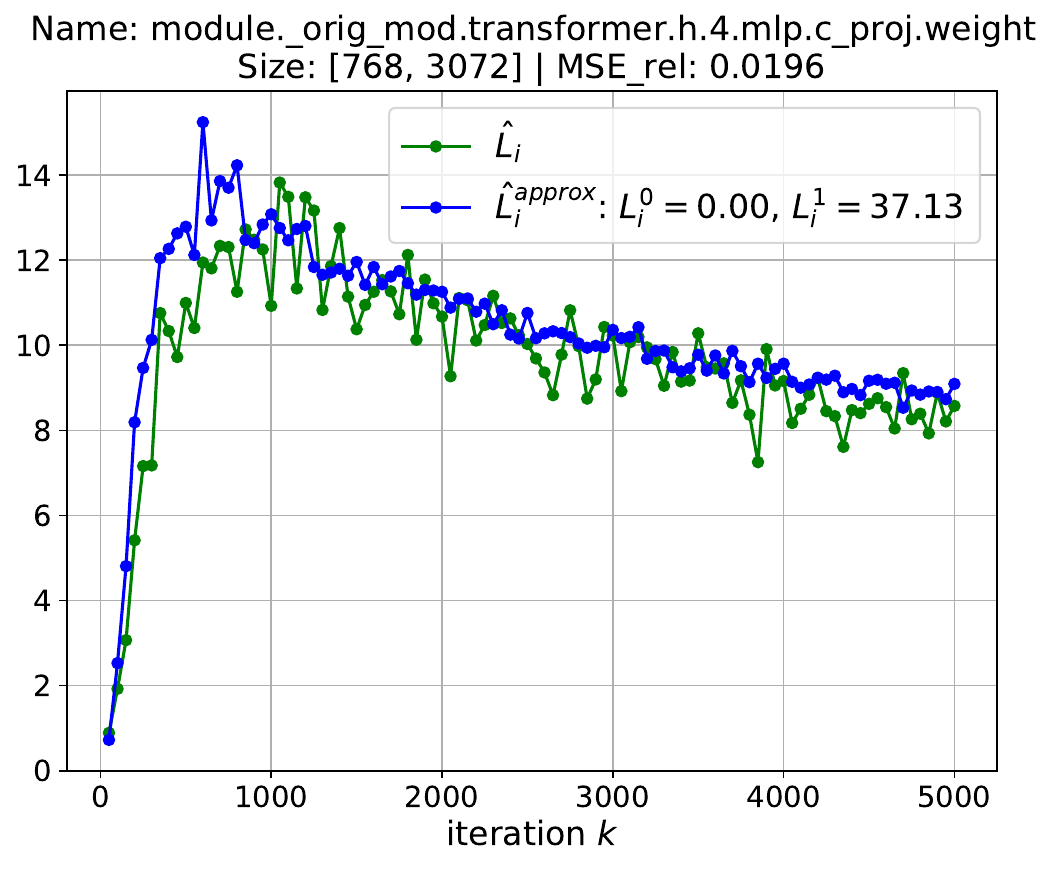}
    \end{subfigure}
    \hfill
    \begin{subfigure}{0.32\textwidth}
        \includegraphics[width=\textwidth]{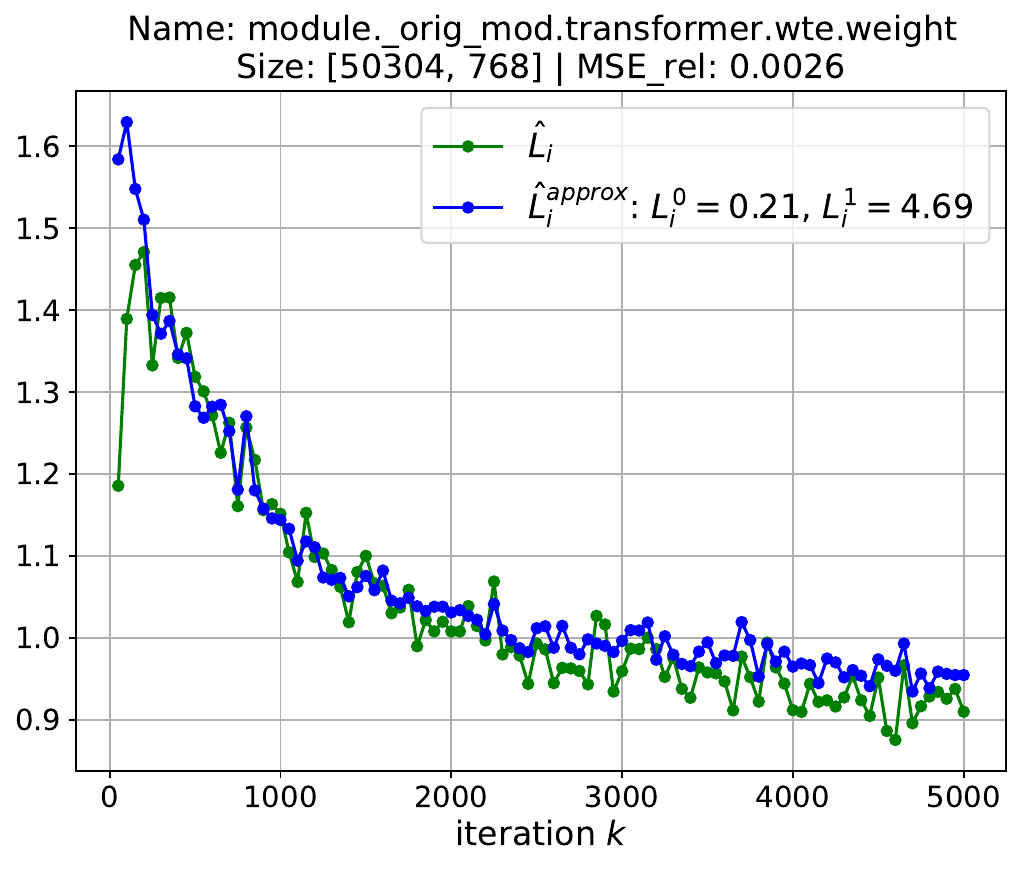}
    \end{subfigure}
    
    \caption{\small Validation of layer-wise $(L^0, L^1)$-smoothness for different groups of parameters in \texttt{NanoGPT-124M} along \algname{AdamW} training trajectories.}
    \label{fig:Adam}
\end{figure}

\subsection{Training CNN on CIFAR-10} \label{appendix:cifar}

In this section, we provide additional results for the experiments described in \Cref{sec:cnn}, where a \texttt{CNN} model is trained on the \texttt{CIFAR-10} dataset using the \algname{unScion} optimizer.

\paragraph{Full-batch (deterministic) gradients.}

\begin{figure}[t]
    \centering
    \begin{subfigure}{0.32\textwidth}
        \includegraphics[width=\textwidth]{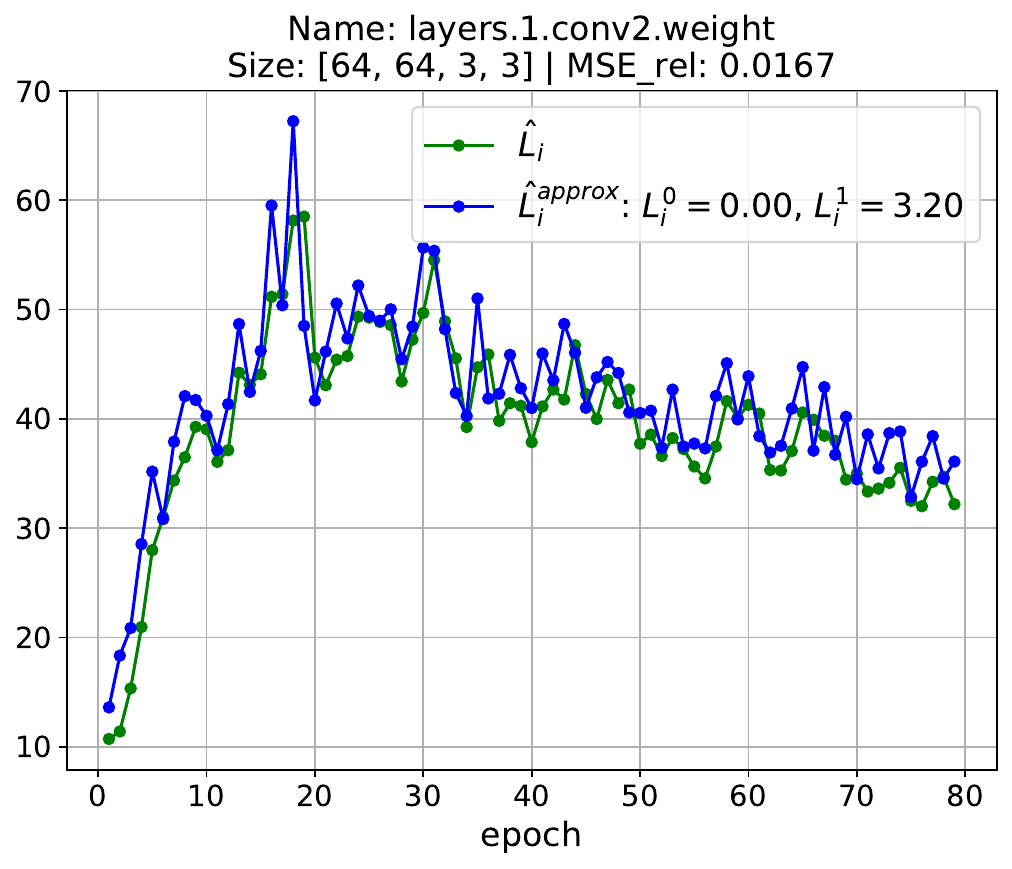}
    \end{subfigure}
    \hfill
    \begin{subfigure}{0.32\textwidth}
        \includegraphics[width=\textwidth]{plots/cnn/plot2_fb.pdf}
    \end{subfigure}
    \hfill
    \begin{subfigure}{0.32\textwidth}
        \includegraphics[width=\textwidth]{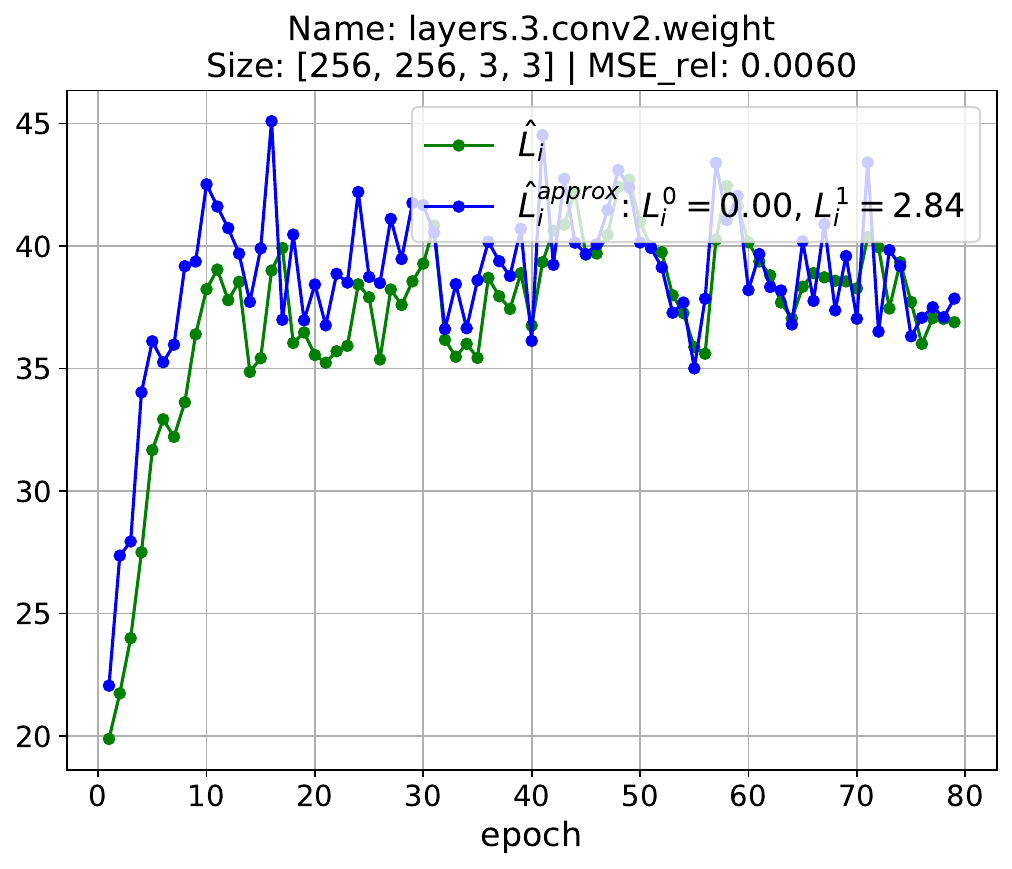}
    \end{subfigure}
    
    \vspace{0.5cm}
    
    \begin{subfigure}{0.32\textwidth}
        \includegraphics[width=\textwidth]{plots/cnn/plot4_fb.pdf}
    \end{subfigure}
    \hfill
    \begin{subfigure}{0.32\textwidth}
        \includegraphics[width=\textwidth]{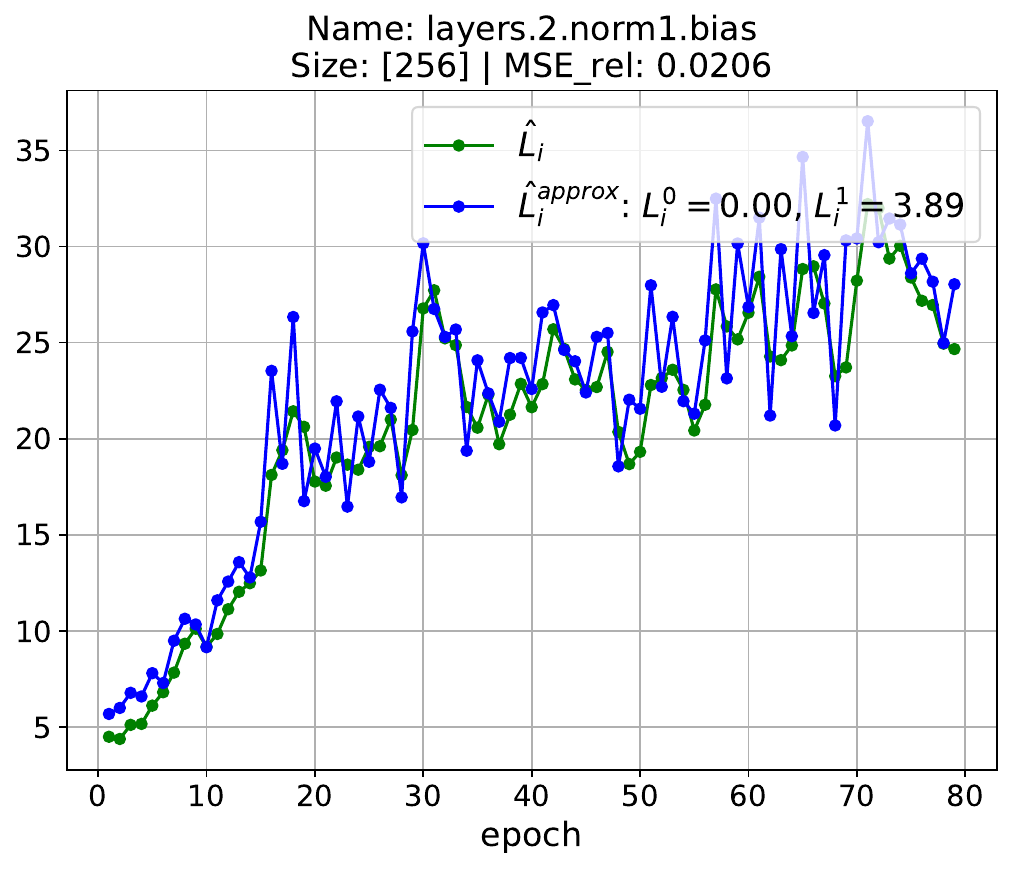}
    \end{subfigure}
    \hfill
    \begin{subfigure}{0.32\textwidth}
        \includegraphics[width=\textwidth]{plots/cnn/plot6_fb.pdf}
    \end{subfigure}
    
    \caption{\small Validation of layer-wise $(L^0, L^1)$-smoothness for different groups of parameters of a \texttt{CNN} model along the training trajectories of \algname{unScion} with \textbf{full-batch gradients}. The norms used for each group are as follows: \(\|\cdot\|_{(i)} = \sqrt{\nicefrac{1}{C_i^{out}}}\|\cdot\|_2\) for biases, \( \|\cdot\|_{(i)} = k^2 \sqrt{\nicefrac{C_i^{in}}{C_i^{out}}} \|\cdot\|_{2 \to 2} \) for conv, and \(\|\cdot\|_{(p)} = n_p \|\cdot\|_{1 \to \infty}\) for the last group \( X_p \), associated with classification head weights.}
    \label{fig:5_appendix}
\end{figure}

We begin with presenting additional results in the deterministic setting. Figure~\ref{fig:5_appendix} shows the estimated trajectory smoothness
\begin{align*}
    \hat{L}_i[k] \eqdef \frac{\|\nabla_i f (X^{k+1}) - \nabla_i f (X^{k}) \|_{(i)\star}}{\|X_i^{k+1} - X_i^{k}\|_{(i)}}
\end{align*}
and its approximation
\begin{align*}
    \hat{L}_i^{\text{approx}}[k] \eqdef L_i^1 \|\nabla_i f (X^{k+1})\|_{(i)\star}
\end{align*}
(where we set $L^0_i = 0$) for a broader selection of parameter groups than shown in the main text. The results further support the validity of Assumption~\ref{ass:generalized-smoothness} with $L^0_i = 0$.

\paragraph{Stochastic gradients.}

\begin{figure}[h]
    \centering
    \begin{subfigure}{0.32\textwidth}
        \includegraphics[width=\textwidth]{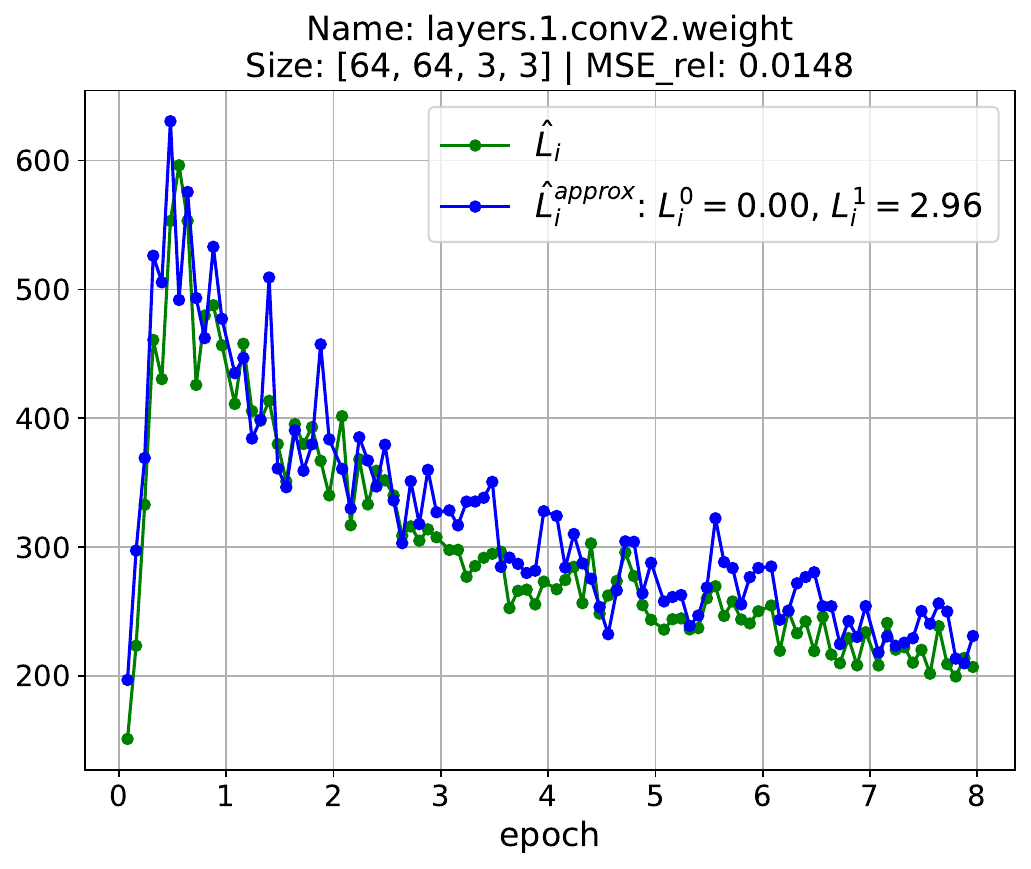}
    \end{subfigure}
    \hfill
    \begin{subfigure}{0.32\textwidth}
        \includegraphics[width=\textwidth]{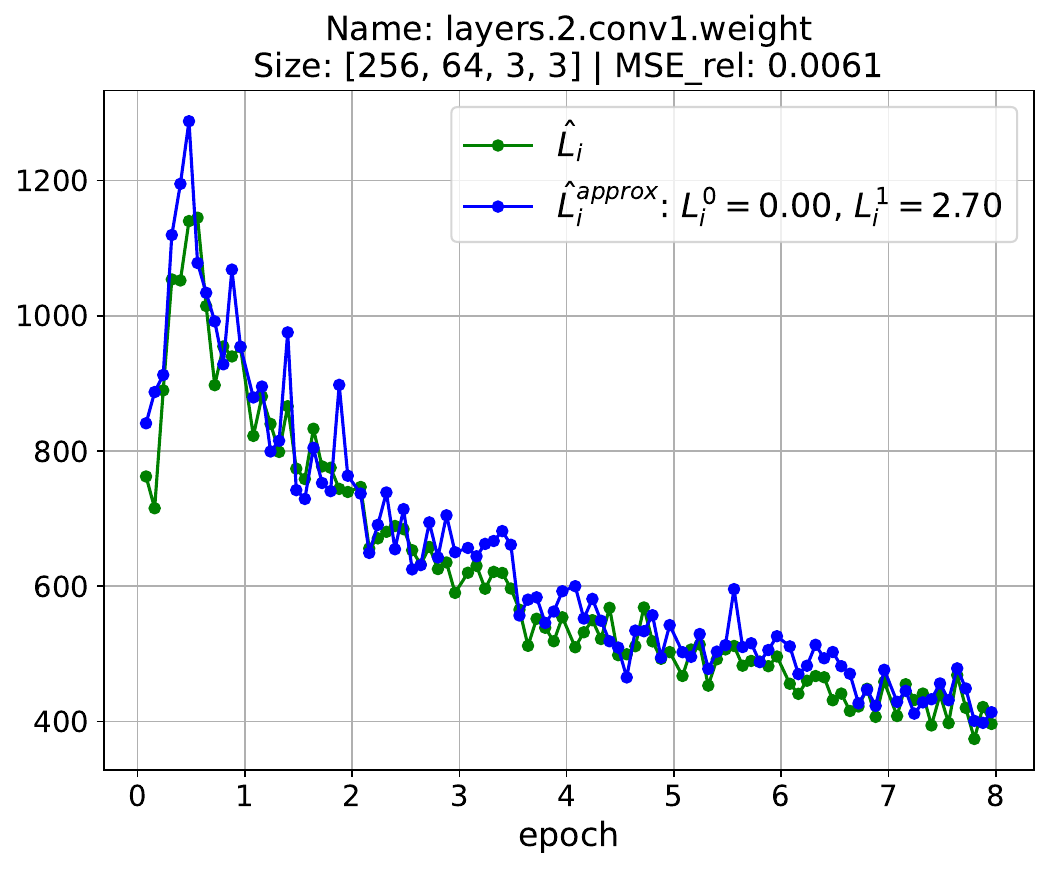}
    \end{subfigure}
    \hfill
    \begin{subfigure}{0.32\textwidth}
        \includegraphics[width=\textwidth]{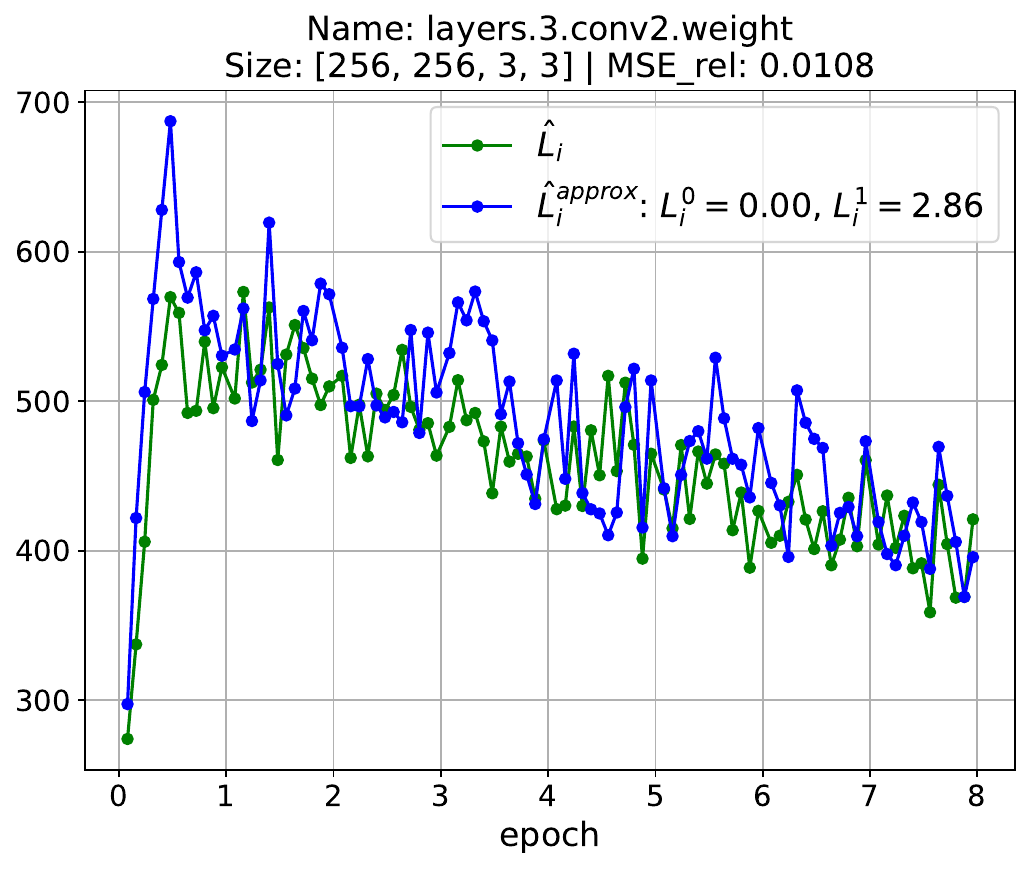}
    \end{subfigure}
    
    \vspace{0.5cm}
    
    \begin{subfigure}{0.32\textwidth}
        \includegraphics[width=\textwidth]{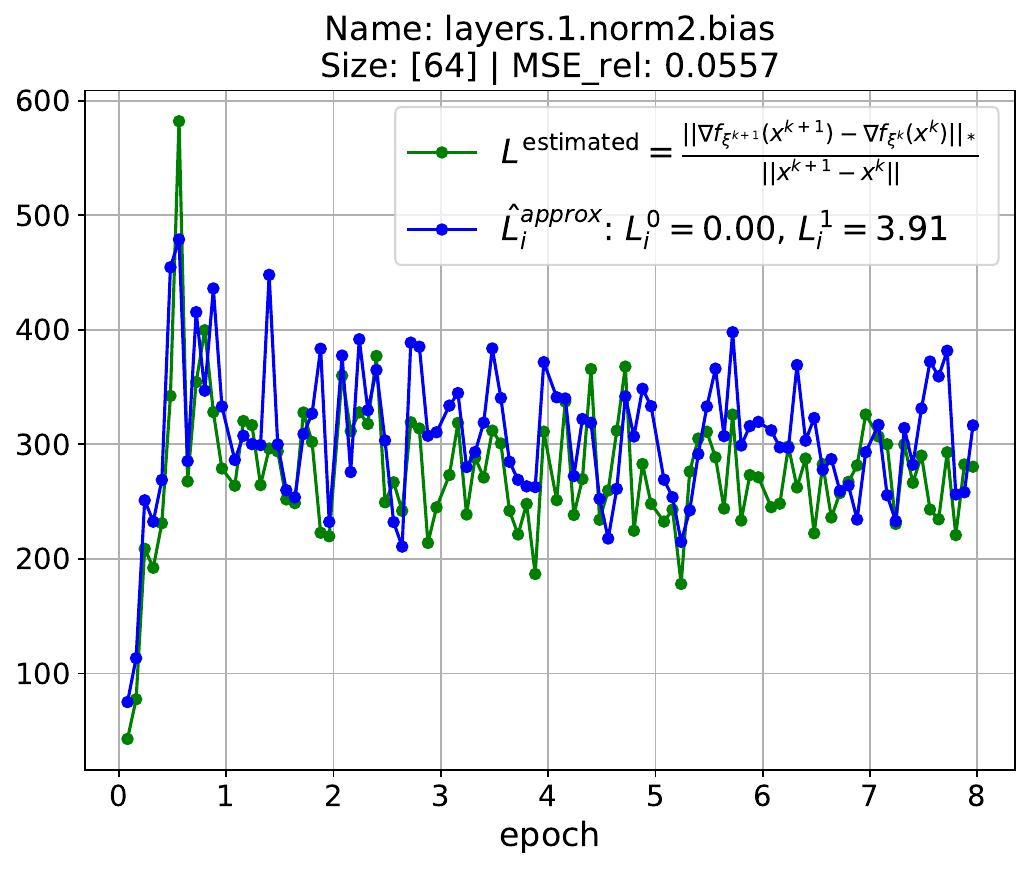}
    \end{subfigure}
    \hfill
    \begin{subfigure}{0.32\textwidth}
        \includegraphics[width=\textwidth]{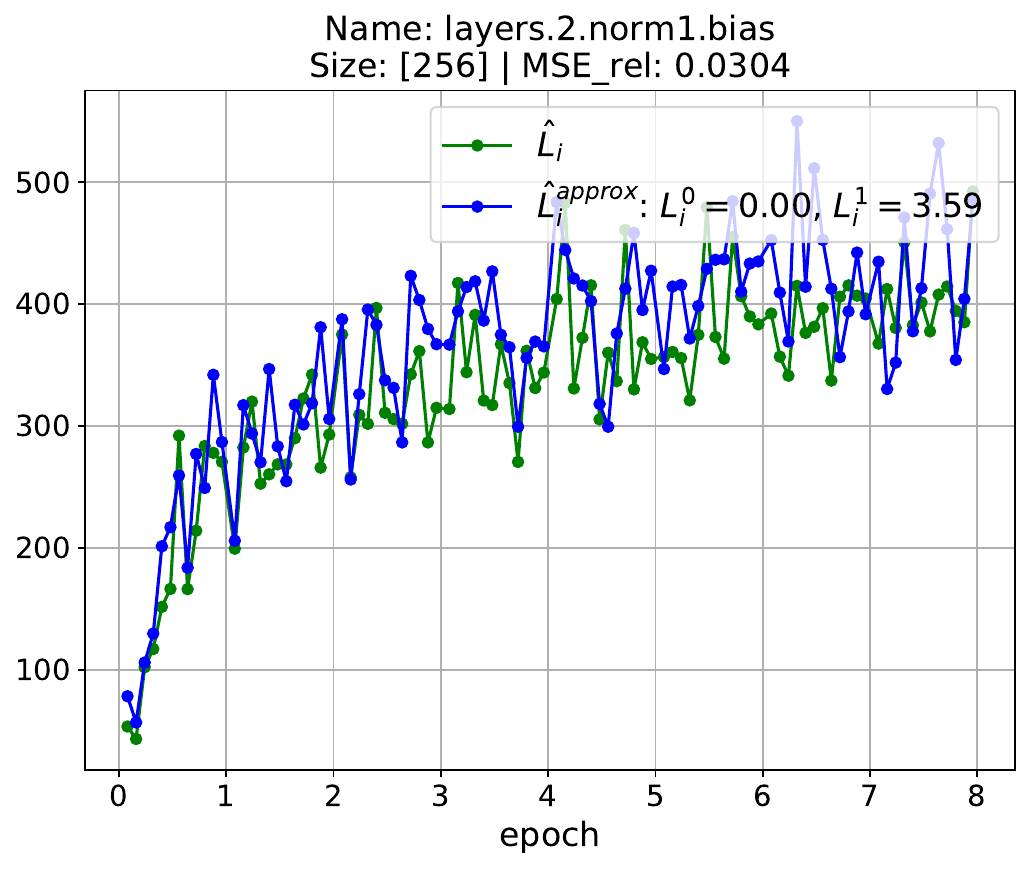}
    \end{subfigure}
    \hfill
    \begin{subfigure}{0.32\textwidth}
        \includegraphics[width=\textwidth]{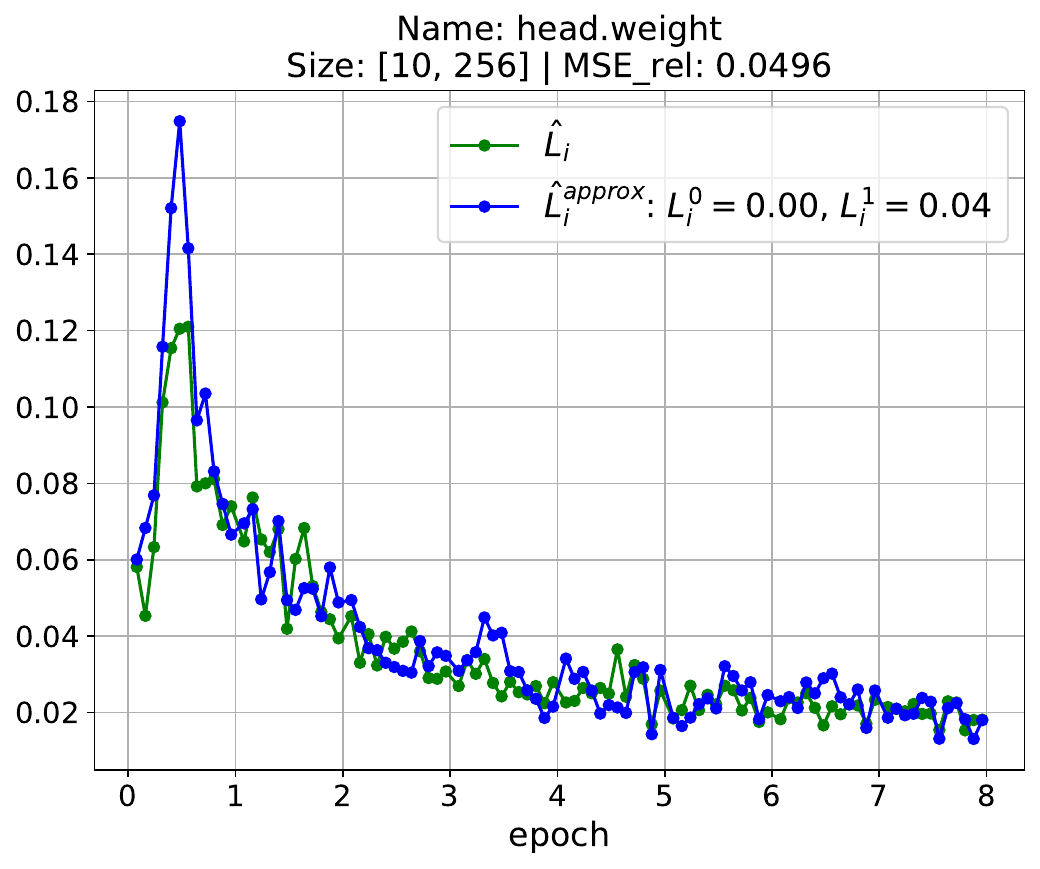}
    \end{subfigure}
    
    \caption{\small Validation of layer-wise $(L^0, L^1)$-smoothness for different groups of parameters of a \texttt{CNN} model along the training trajectories of \algname{unScion} with \textbf{stochastic gradients}. The norms used for each group are as follows: \(\|\cdot\|_{(i)} = \sqrt{\nicefrac{1}{C_i^{out}}}\|\cdot\|_2\) for biases, \( \|\cdot\|_{(i)} = k^2 \sqrt{\nicefrac{C_i^{in}}{C_i^{out}}} \|\cdot\|_{2 \to 2} \) for conv, and \(\|\cdot\|_{(p)} = n_p \|\cdot\|_{1 \to \infty}\) for the last group \( X_p \), associated with classification head weights.}
    \label{fig:4}
\end{figure}

Here, we report results for analogous experiments in the stochastic setting, using noisy gradients \( \nabla_i f_{\xi^k} \). We use momentum as in \citet[Table 10]{pethick2025training}, but do not apply a linear decay schedule.
In Figure~\ref{fig:4}, we plot
\[
\hat{L}_i[k] = \frac{\|\nabla_i f_{\xi^{k+1}}(X^{k+1}) - \nabla_i f_{\xi^k}(X^k)\|_{(i)\star}}{\|X_i^{k+1} - X_i^k\|_{(i)}}, \qquad \hat{L}_i^{\text{approx}}[k] = L^1_i \|\nabla_i f_{\xi^{k+1}}(X^{k+1})\|_{\star},
\]
again setting $L^0_i = 0$. Despite the added variance, we still observe that the stochastic trajectory roughly adheres to Assumption~\ref{ass:generalized-smoothness}.

\end{document}

%% file: macros.tex
\usepackage{multirow}
\usepackage{layout}

\usepackage{amsmath}
\usepackage{amsthm}
\usepackage{amssymb}
\usepackage{amsfonts}
\usepackage{mathtools}
\usepackage{siunitx}
\usepackage{lipsum}

\usepackage{enumitem}

\usepackage{bbm}

\usepackage{url}
\usepackage{color}
\usepackage{graphicx}
\usepackage{verbatim}
\usepackage{subcaption}

\usepackage{cleveref}

\usepackage{natbib}

\usepackage{apptools}
\usepackage{thmtools}
\usepackage{algorithm, algpseudocode}

\newcommand{\norm}[1]{\left\| #1 \right\|}

\newcommand{\inp}[2]{\left\langle#1,#2\right\rangle} 

\newcommand{\parens}[1]{\left( #1 \right)}
\newcommand{\brac}[1]{\left\{ #1 \right\}}

\newcommand{\cB}{\mathcal{B}}
\newcommand{\cC}{\mathcal{C}}
\newcommand{\cD}{\mathcal{D}}

\newcommand{\cS}{\mathcal{S}}

\newcommand{\del}[1]{}

\newcommand{\R}{\mathbb{R}} 

\newcommand{\eqdef}{:=} 

\newcommand{\Exp}[1]{{\mathbb{E}}\left[#1\right]}

\newcommand{\ExpCond}[2]{{\mathbb{E}}\left[\left.#1\right\vert#2\right]}
\newcommand{\ExpSub}[2]{{\mathbb{E}}_{#1}\left[#2\right]}

\DeclareMathOperator*{\argmin}{arg\,min}

\usepackage{tcolorbox}
\usepackage{pifont}
\definecolor{mydarkgreen}{RGB}{39,130,67}
\definecolor{mydarkred}{RGB}{192,47,25}
\definecolor{mydarkorange}{RGB}{39,130,67}
\newcommand{\green}{\color{mydarkgreen}}
\newcommand{\red}{\color{mydarkred}}
\newcommand{\cmark}{\green\ding{51}}%
\newcommand{\xmark}{\red\ding{55}}%

\newcommand{\algnametiny}[1]{{\tiny \red\sf #1}}

\newcommand{\algnamefootnote}[1]{{\footnotesize \red\sf #1}}
\newcommand{\algname}[1]{{\small \red \sf #1}}

\newtheorem{assumption}{Assumption}
\newtheorem{lemma}{Lemma}

\newtheorem{theorem}{Theorem}

\theoremstyle{plain}

\theoremstyle{definition}

\newtheorem{remark}[theorem]{Remark}

\newtheorem{definition}[theorem]{Definition}

\usepackage{listings}
\usepackage{xcolor}

\definecolor{codegreen}{rgb}{0,0.6,0}
\definecolor{codegray}{rgb}{0.5,0.5,0.5}
\definecolor{codepurple}{rgb}{0.58,0,0.82}
\definecolor{backcolour}{rgb}{0.95,0.95,0.92}

\lstdefinestyle{mystyle}{
    backgroundcolor=\color{backcolour},   
    commentstyle=\color{codegreen},
    keywordstyle=\color{magenta},
    numberstyle=\tiny\color{codegray},
    stringstyle=\color{codepurple},
    basicstyle=\ttfamily\footnotesize,
    breakatwhitespace=false,         
    breaklines=true,                 
    captionpos=b,                    
    keepspaces=true,                 
    numbers=left,                    
    numbersep=5pt,                  
    showspaces=false,                
    showstringspaces=false,
    showtabs=false,                  
    tabsize=2
}

\newcommand{\lmo}[2]{{\rm LMO}_{#1}(#2)}